\documentclass[letterpaper]{article}

\usepackage{comment}
\usepackage[margin=1in]{geometry}

\usepackage{amsmath} 
\usepackage{amssymb}  
\usepackage{amsthm}
\usepackage{algorithm}
\usepackage{algorithmic}

\usepackage{mathptmx}       

\usepackage{helvet}         
\usepackage{courier}        
\usepackage{type1cm}        
                            
\DeclareMathAlphabet{\mathcal}{OMS}{cmsy}{m}{n}

\usepackage{graphicx}
\usepackage{subfig}
\usepackage{booktabs} 
\usepackage{color}
\usepackage{ifthen}
\usepackage{psfrag}
\usepackage{amssymb}
\usepackage[ampersand]{easylist}
\usepackage{enumitem}

\definecolor{darkgreen}{rgb}{0.,.66,0.}
\usepackage{hyperref}
\hypersetup{bookmarks=true,colorlinks=true,citecolor=darkgreen,urlcolor=darkgreen,linkcolor=darkgreen,bookmarksnumbered=true,bookmarksopen=true,breaklinks}

\usepackage[numbers,sort&compress]{natbib}
\bibpunct[, ]{(}{)}{,}{a}{}{,}

\allowdisplaybreaks[1]

\newtheorem{definition}{Definition}
\newtheorem{theorem}{Theorem}
\newtheorem{lemma}{Lemma}

\newtheorem{corollary}{Corollary}
\newtheorem{assumption}{Assumption}

\DeclareMathOperator*{\im}{Im}

\usepackage{comment}
\newcommand{\defn}[1]{\begin{definition} #1 \end{definition}}
\newcommand{\set}[1]{\left\{ #1 \right\}}
\newcommand{\abs}[1]{\left| #1 \right|}
\newcommand{\absM}[1]{\abs{#1}_{\Mbar}}
\newcommand{\absMinv}[1]{\abs{#1}_{\Minv}}
\newcommand{\angleM}[2]{\langle #1, #2 \rangle_{\Mbar}}

\newcommand{\paren}[1]{\left( #1 \right)}
\newcommand{\brak}[1]{\left[ #1 \right]}

\newcommand{\vfof}[2]{\frac{\partial #1}{\partial #2}}
\newcommand{\into}{\rightarrow}
\newcommand{\goesto}{\rightarrow}
\newcommand{\dt}[1]{\frac{d\,}{d#1}}
\newcommand{\e}[1]{\ensuremath{\mathcal{#1}}}
\newcommand{\hds}{\e{H}=(\e{J},\Gamma,\e{D},F,\e{G},R)}
\newcommand{\sm}{\backslash}
\newcommand{\eqn}[1]{\begin{equation*}\begin{aligned}[b] #1 \end{aligned}\end{equation*}}
\newcommand{\eqnn}[1]{\begin{equation}\begin{aligned} #1 \end{aligned}\end{equation}}

\newcommand{\bd}{\partial}
\newcommand{\KN}{{\e{K}_n}}
\newcommand{\KT}{{\e{K}_t}}
\newcommand{\td}[1]{\widetilde{#1}}

\newcommand{\vphi}{\varphi}
\newcommand{\Lie}{\e{L}}
\newcommand{\obar}[1]{\overline{#1}}
\newcommand{\ubar}[1]{\underline{#1}}

\definecolor{grey50}{rgb}{0.5,0.5,0.5}
\definecolor{ltblue}{rgb}{0.25,0.25,1.0}
\newcommand{\changed}[1]{{#1}} 

\usepackage{tikz}

\begin{document}
\title{A Hybrid Systems Model for Simple Manipulation and Self-Manipulation Systems}

\author{Aaron~M.~Johnson\thanks{Corresponding author;
Robotics Institute,
Carnegie Mellon, Pittsburgh, PA, USA, e-mail: amj1@andrew.cmu.edu}
\and Samuel~A.~Burden
\and D.~E.~Koditschek}

\date{\vspace{-3ex}}
\maketitle

\begin{tikzpicture}[overlay, remember picture]
\path (current page.north west) ++(2.5,0) node[below right] {\fbox{\parbox{\dimexpr\textwidth-\fboxsep-\fboxrule\relax}{  \footnotesize  This paper has be published in the International Journal of Robotics Research, \textcopyright 2016 Sage Publishing. Cite as: \\ 
Aaron M. Johnson, Samuel E. Burden, and D. E. Koditschek. ``A Hybrid Systems Model for Simple Manipulation and Self-Manipulation Systems.'' \emph{International Journal of Robotics Research}, 35(11): 1354--1392. September 2016.
DOI: \href{https://doi.org/10.1177/0278364916639380}{https://doi.org/10.1177/0278364916639380} }}};
\end{tikzpicture}

\mathchardef\minus="002D

\newcommand{\mycoprod}{%
               \mathrel{\raisebox{.1em}{%
               \reflectbox{\rotatebox[origin=c]{180}{$\prod$}}}}}

\newcommand{\argmin}{\operatornamewithlimits{argmin}}
\newcommand{\argmax}{\operatornamewithlimits{argmax}}
\def\mathbi#1{\textbf{\em #1}}

\def \Mbar {\overline{\mathbf{M}}}
\def \Mbare {\overline{\mathbf{M}}_\epsilon}
\def \Md {\overline{\mathbf{M}}^{\dagger}}
\def \Cbar {\overline{\mathbf{C}}}
\def \Ctd {\td{\mathbf{C}}}
\def \Nbar {\overline{\mathbf{N}}}
\def \Minv {\Mbar^{-1}}
\def \bfq  {\mathbf{q}}
\def \biq  {\mathbf{T}\mathbf{q}} 
\def \Tq   {\biq} 
\def \biy  {\mathbf{T}\mathbf{y}} 
\def \Ty   {\biy} 
\def \bff  {\mathbf{f}}
\def \bfa  {\mathbf{a}}
\def \bfW  {\mathbf{W}}
\def \bfS  {\mathbf{S}}
\def \bfA  {\mathbf{A}}
\def \bfB  {\mathbf{B}}
\def \bfAd  {\mathbf{A}^{\dagger}}
\def \bfAdT  {\mathbf{A}^{\dagger T}}
\def \bfc  {\mathbf{c}}
\def \bfD  {\mathbf{D}}
\def \bfE  {\mathbf{E}}
\def \bfG  {\mathbf{G}}
\def \bfH  {\mathbf{H}}
\def \bfJ  {\mathbf{J}}
\def \bfK  {\mathbf{K}}
\def \bfM  {\mathbf{M}}
\def \bfN  {\mathbf{N}}
\def \bfP  {\mathbf{P}}
\def \bfU  {\mathbf{U}}
\def \bfu  {\mathbf{u}}
\def \bfx  {\mathbf{x}}
\def \bfy  {\mathbf{y}}
\def \bfY  {\mathbf{Y}}
\def \bfz  {\mathbf{z}}
\def \qimp {\mathbf{P}}
\def \limp {\widehat{\bfP}}
\def \dimp {\widetilde{\bfP}}
\def \Id  {\mathbf{Id}}
\def \calC {\mathcal{C}}
\def \calD {\mathcal{D}}
\def \calF {\mathcal{F}}
\def \calG {\mathcal{G}}
\def \calH {\mathcal{H}}
\def \calI {\mathcal{I}}
\def \calJ {\mathcal{J}}
\def \calK {\mathcal{K}}
\def \calN {\mathcal{N}}
\def \calQ {\mathcal{Q}}
\def \calR {\mathcal{R}}
\def \calT {\mathcal{T}}
\def \calU {\mathcal{U}}
\def \nGamma {\mathnormal{\Gamma}}
\def \bbR {\mathbb{R}}
\def \rmq {\mathrm{q}}
\def \itA  {\mathit{A}}
\def \itC  {\mathit{C}}
\def \itF  {\mathit{F}}
\def \itG  {\mathit{G}}
\def \itH  {\mathit{H}}
\def \itL  {\mathit{L}}
\def \itO  {\mathit{O}}
\def \itP  {\mathit{P}}
\def \itR  {\mathit{R}}
\def \itS  {\mathit{S}}
\def \itX  {\mathit{X}}
\def \itY  {\mathit{Y}}
\def \itSr  {\mathit{S_r}}
\def \itSl  {\mathit{S_l}}
\def \Real {\mathbb{R}}
\def \Nat {\mathbb{N}}
\def \Zed {\mathbb{Z}}
\def \TF {\mathbb{B}}
\def \ddx  {\delta_{\dot{x}}}
\def \ddz  {\delta_{\dot{z}}}
\def \dz   {\delta_z}
\def \dtp {\dot{\:\!\phi}}
\def \simp {B}
\def \Cl {\mathrm{Cl}\,}
\def \St {\mathrm{St}\,}
\def \nrhd {\not\!\rhd}
\def \CP {\mathrm{CP}}
\def \FA {\mathrm{FA}}
\def \FFA {\mathrm{FFA}}
\def \IV {\mathrm{IV}}
\def \PIV {\mathrm{PIV}}
\def \PRED {\mathrm{PRED}}
\def \TD {\mathrm{TD}}
\def \NTD {\mathrm{NTD}}
\def \LO {\mathrm{LO}}

\begin{abstract}

Rigid bodies, plastic impact, persistent contact, Coulomb friction, and massless limbs are ubiquitous simplifications 
introduced to reduce the complexity of mechanics models despite the obvious physical inaccuracies that each incurs individually. 
In concert, it is well known that the interaction of such idealized approximations can lead to conflicting and even paradoxical results. 
As robotics modeling moves from the consideration of isolated behaviors to the analysis of tasks requiring their composition, 
a mathematically tractable framework for building models that combine these simple approximations yet achieve reliable results is overdue.   
In this paper we present a formal  hybrid dynamical system model 
that introduces suitably restricted compositions of these familiar abstractions 
with the guarantee of consistency analogous to global existence 
and uniqueness in classical dynamical systems.
The hybrid system developed here provides a discontinuous but self-consistent approximation to the continuous (though possibly very stiff and fast) dynamics of a physical robot undergoing intermittent impacts.
The modeling choices sacrifice some quantitative numerical efficiencies while maintaining qualitatively correct and analytically
tractable results with consistency guarantees promoting their use in formal reasoning about mechanism,
feedback control, and behavior design in robots that make and break contact with their environment.

\end{abstract}

\section{Introduction}\label{sec:intro}

Simple models of complex robot--world interactions are key to understanding, implementing and generalizing 
behaviors as well as identifying and composing their reusable 
constituents to generate new behaviors \cite[]{paper:Full-JEB-1999}. There is strong appeal to using familiar physical simplifications such
as rigid bodies, plastic impacts, persistent contact, Coulomb friction, and massless limbs in building up simple robotics 
models. Their coarse approximation to the underlying physical processes of interest are widely understood to offer 
the right combination of analytical tractability and physical realism in isolation. However, it is also widely understood 
that such individually useful simplifications can introduce catastrophic side-effects when  
combined (e.g.\ in \cite{painleve1895leccons,keller1986impact,mason1988inconsistency,dupont1994jamming,trinkle1997dynamic,chatterjee1999realism}  
and others, as discussed in Section~\ref{sec:lit}).

In this paper we assemble a framework of reasonable physical assumptions and 
accompanying mechanics to develop a formalism for combining them
at will in the construction of a simple hybrid system model for contact robotics 
that yields a provably consistent\footnote{
Here, \emph{consistent} refers to a combination of properties detailed in Section~\ref{sec:consistency}
analogous to the guarantee of global existence and uniqueness of solutions for a 
classical dynamical system.}
and empirically useful approximation to many behavioral settings of interest. 
As an example of the value of such mathematical  models, new work \cite[]{Brill_De_Johnson_Koditschek_2015} 
uses the formal properties of our self-manipulation model to develop rigorous correctness (or, non-existence) 
proofs for desirable robot behaviors -- in that case, gap crossing and ledge mounting. However, while
the primary goal of this paper is not numerical analysis, simulation does 
provide a useful way to visualize key features of the model and the utility of some 
of the simplifying assumptions. Numerical results obtained through a custom
Mathematica
simulation \changed{(described in Section~\ref{sec:dis:numerical})} are used throughout the paper to illustrate key concepts, and to suggest the fidelity to physical
settings of interest.

For example, our model generates simulations\footnote{For this simulation the
middle and rear legs are used with a maximum current limit of $20$A, a \emph{pseudo-impulse} (defined in Section~\ref{sec:pseudoimpulse}) magnitude
of $\delta_t=0.03$ (hand selected to give the qualitatively best overall results), relative
leg timing of $t_2 = 0.01$ (i.e., the middle legs are started $0.01$s before the rear legs), 
and once a leg has lifted off the ground it is slowly rotated upwards out of the way.
Remaining model parameters are as listed in \cite[Sec.~III, Appendix~G]{johnson_selfmanip_2013}.} of the leaping behavior depicted in 
Figure~\ref{fig:leapsim} that recreate the empirical results of \cite{paper:johnson-icra-2013} 
qualitatively (i.e., predicts the same salient features though not necessarily the same metric results),
yet enjoys a combination of mathematical properties that we believe will provide a foundation for reasoning about 
and thereby generalizing the platform design and control strategies that gave rise to such behaviors. 
Of course, physical fidelity is not mathematically demonstrable 
and the relevance of the modeling choices we propose (i.e., the empirical sway of this formally self-consistent 
model) can only be established over the long run in practice by the breadth of physical phenomena they usefully 
approximate, regardless of the simplification and ease of analysis they~afford.

\begin{figure}
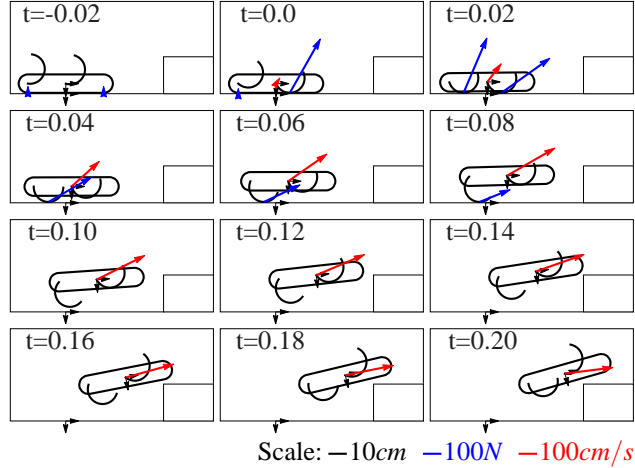

\centering
\def\svgwidth{8.3cm}
\include{leapframes}
  \vspace{-8pt}
\caption{Keyframes from RHex simulation leaping onto a 20cm ledge. Blue arrows show contact forces while
the red arrow shows body velocity. 
}
\label{fig:leapsim}
\end{figure}

The paper is structured as follows. This section finishes with a summary of contributions, 
followed by a discussion of their relation to prior work. Section~\ref{sec:imp} introduces the various simplifying physical
modeling assumptions and draws out some of the mathematical consequences bearing on their relationships 
to alternative formulations and to each other. Section~\ref{sec:hyb} assembles from these pieces a formal 
hybrid dynamical system model and proves its consistency. Section~\ref{sec:discussion} reviews the scope 
of physical settings admitted by our assumptions and discusses the most delicate aspects of their interplay 
with our formal results, providing additional examples that help give a broader context for the applicability 
of the theory. Section~\ref{sec:Conclusion} concludes with some final thoughts on the
implications of this work and future directions. An extensive Appendix works through the
details of selected proofs and provides additional background material. 

\subsection{Contributions of the Paper}
\label{sec:cont}

This paper extends a framework for manipulation \cite[]{book:mls-1994} and self-manipulation \cite[]{johnson_selfmanip_2013} 
modeling into a formal hybrid dynamical systems specification whose discrete 
modes are indexed by the active contact constraint set in a manner guaranteed to produce 
a unique execution from every initial condition under mild conditions on the motor feedback control laws. 
The foundation on which we rest this physically simple and mathematically tractable modeling framework 
arises from Assumptions~\ref{ass:rigid}--\ref{ass:friction}, introduced in Section~\ref{sec:imp} and discussed
further in Section~\ref{sec:dis:ass}, comprising various familiar phenomenological 
representations and physically natural hypotheses, including: rigid bodies (\ref{ass:rigid}), 
massless limbs (\ref{ass:contact}), plastic impact (\ref{ass:plastic}), and static friction (\ref{ass:friction}). 
It is known that in general these properties are not mutually consistent, however we
formally demonstrate that the particular set of assumptions included here provides a well defined, 
deterministic, and computationally well-behaved model. 
The physical fidelity may, in some important applications contexts that we point out, necessarily remain something of a 
leap of faith relative to the still incomplete state of \emph{the theory of} rigid body 
mechanics. To the best of our knowledge this is the first time any succinctly stated list of physical 
assumptions about rigid body mechanics has been shown to yield a consistent hybrid dynamical system with 
unique and globally defined executions.

Our central technical contribution is the derivation of a consistent  extension of Lagrangian dynamics, Newtonian
impact laws, and complementarity contact conditions to systems that have certain rank deficiencies in their inertia
tensor that agrees with (i.e., when rank is restored, maintains equivalence to) the nonsingular case 
(Lemmas~\ref{thm:mlequiv}, \ref{thm:dynamics}, \&~\ref{thm:impulse} and Theorems~\ref{thm:fac} \&~\ref{thm:ivc}).
The possibly massless dynamics motivate a reformulation of complementarity as a logical equivalence 
(Lemma~\ref{thm:comp}) so that its unique solvability (for both force--acceleration and impulse--velocity 
complementarity problems, Assumptions~\ref{ass:fac} \& \ref{ass:ivc}, respectively) 
\changed{is shown} to imply a unique partition of the \emph{guard set}  (i.e., those states which are to
undergo a mode transition) into disjoint components labeled deterministically 
by the destination mode of the transition (Theorem~\ref{thm:disjoint}).
These conditions are expressed in terms of a higher order scalar relation 
($\prec$, Definition~\ref{def:through}), and we exhibit certain properties of this relation that 
clarify its role in determining the guard set (Lemmas~\ref{lem:closure}--\ref{lem:trend}).

Even without the introduction of massless limbs there exist
many opportunities for repeated (and even \emph{Zeno}\footnote{An 
execution of a hybrid dynamical system exhibits the Zeno phenomena if it undergoes an infinite 
number of discrete or logical switches in finite time (Definition~\ref{def:zeno}).}) discrete transitions that seem unlikely to add much 
physical insight (and, speaking practically, generally degrade the numerical performance of simulations 
based upon this model). Hence, to resolve the qualitative problem of spurious transitions at arbitrarily
low velocities (Lemma~\ref{lem:pseudoimp}), we introduce a new \emph{pseudo-impulse}, 
which acts on the discrete transitional logic (rather than the continuous dynamics), imposing an implicit bound 
on contact velocity below which such contacts persist (Theorem~\ref{thm:pseudoimp}), precluding certain 
Zeno phenomena (Theorem~\ref{thm:zeno}). 

As a structure to combine these physical models and assumptions, this work presents the 
formal definition of the \emph{self-manipulation hybrid system} in Definition~\ref{def:smhs} (along with 
Definitions~\ref{def:hs}--\ref{def:ex}), and the formal demonstration of its 
consistency (including that it is deterministic and 
non-blocking, Theorems~\ref{thm:agree}--\ref{thm:uniqueex} and Lemmas~\ref{lem:lip} \&~\ref{lem:lip:ml}),
incorporating a well-behaved notion of \emph{completion} in case of a Zeno execution
(Definition~\ref{def:zeno}, Theorem~\ref{thm:zenolim}, Corollary~\ref{cor:zenolim}) by adapting to this more elementary setting the 
measure theoretic arguments of \cite{Ballard2000}. 

\changed{
\subsection{Reader's Guide}

We anticipate that different readers will approach this paper with diverse goals and backgrounds.
In this section we suggest sections that may have the greatest relevance with 
respect to specific interests. All readers are
encouraged to start with the setup and notation explained in Section~\ref{sec:not} and summarized
in Tables~\ref{tab:Symbols} and~\ref{tab:Symbols2}.
\begin{itemize}
\item For readers interested in instantiating a hybrid dynamical system model of a specific robot, 
the most relevant section is Section~\ref{sec:smsystem}, which is based
on the assumptions and derivations in Section~\ref{sec:imp} and 
definitions in Section~\ref{sec:hs}. The continuous dynamics in any particular contact
mode is based on \cite{book:mls-1994} for manipulation and \cite{johnson_selfmanip_2013} 
for self-manipulation systems (whose differences are summarized in Section~\ref{sec:manip}).
See also the discussion on numerical simulation of this system in Section~\ref{sec:dis:numerical}.
\item Some readers may be interested in the relationship of this model to the still 
evolving theory of rigid body mechanics. They will likely wish to focus on our treatment of:
massless bodies and singular inertia tensors (Section~\ref{sec:astar}, with implications
throughout Section~\ref{sec:imp} and in Section~\ref{sec:dis:ml}), Lagrangian dynamics (Section~\ref{sec:contdyn}), 
impulsive dynamics (Sections~\ref{sec:impactmap}, \ref{sec:pseudoimpulse}, and~\ref{sec:dis:pseudo}), 
complementarity systems (Section~\ref{sec:complementarity}), friction (Section~\ref{sec:friction}),
and accumulation points (Section~\ref{sec:zeno}).
See also the discussion in Section~\ref{sec:dis:ass} on these mechanics assumptions.
\item Those most interested in the properties of this model considered as a mathematical 
object will likely wish to focus on the hybrid dynamical system itself which is 
presented in Sections~\ref{sec:hs} and~\ref{sec:smsystem},
with various properties related to existence and uniqueness of executions proven in Sections~\ref{sec:correct} and~\ref{sec:consistency}. Zeno considerations
are presented in Section~\ref{sec:zeno}, based in part on the pseudo-impulse presented 
in Section~\ref{sec:pseudoimpulse}, and discussed further in Section~\ref{sec:zenodisc}.
\item Readers interested in the particular example of the RHex robot used in, e.g., Figure~\ref{fig:leapsim}
can find details on the model (lengths, masses, etc) in \cite{johnson_selfmanip_2013}
and details on the behaviors this paper aims to model in \cite{paper:johnson-icra-2013}.
Of particular interest for modeling similar robots may be the treatment of massless
legs (Sections~\ref{sec:astar} and~\ref{sec:dis:ml}), the pseudo-impulse (Section~\ref{sec:pseudoimpulse},
in particular the example in Section~\ref{sec:pseudoimpex}), and the full self-manipulation
hybrid system model (Section~\ref{sec:smsystem}).
\end{itemize}
}
\subsection{Relation to Prior Literature}
\label{sec:lit}

This paper aims to promote simplified physics based models of robotic systems for purposes of 
analysis. Doing so entails integrating results and ideas that have developed somewhat independently 
across several different longstanding technical fields. For surveys of some of these ideas (with a 
focus on numerical considerations), see e.g.\ \cite{brogliato2002numerical,gilardi2002literature}.

\subsubsection[Numerical Simulation Methods]{Numerical Simulation Methods}\label{sec:sim}

While this paper is focused on a model for analysis and not simulation, it is informative to consider
how other simulation strategies compare.
The model developed here generates trajectories from the flow of hybrid dynamical systems defined by 
differential-algebraic equations (DAEs) 
between discrete transitions and so, in the language of \cite[Sec.~6.3]{brogliato2002numerical}, 
simulations of these trajectories could be obtained%
\footnote{E.g.\ using the algorithm proposed in~\cite{burden2013metrization} for hybrid control systems.}
via an \emph{event-driven} scheme, as opposed to a 
\emph{penalized-constraint}/\emph{continuous-contact} scheme, or a \emph{time-stepping} scheme.

Event-driven schemes have a long history, e.g.\ \cite[]{wehage1982dynamic}, \cite[]{pfeiffer2008multibody},
\cite[Sec.~6.7]{brogliato2002numerical}, and include the hybrid dynamical systems formulations outlined 
in the next section. Typically they entail alternating between integration of smooth dynamics involving
(usually) finite forces from contacts and the discontinuous handling of constraint addition or deletion (the ``events''). 
Here, we extend these methods and codify the event-driven scheme
in terms of a formal hybrid dynamical system. 
In contrast, some event-driven schemes formulate the contact dynamics as always consisting of impulses, e.g.~\cite{mirtich1995impulse}.
These impulse-based simulations combine both smooth and discontinuous contact interactions
into impulses, with a continuous-time ballistic trajectory in between events.

Time-stepping schemes, which also account for contact interactions only using impulses by integrating
applied forces over small time steps, are numerically efficient especially for systems with large numbers
of constraints, see e.g.\ \cite[]{stewart1996implicit} \cite[]{anitescu1997formulating}, or \cite[Sec.~7]{brogliato2002numerical}. 
These models can be relaxed (by allowing contact forces and impulses to arise even before contact occurs)
to enable efficient numerical simulation and motion synthesis \cite[]{drumwright2011modeling,todorov2012mujoco}.
These methods allow contact constraints to be added or removed at any 
time step, but only once per time step. Furthermore, no distinction between continuous contact
forces and discontinuous impulses is made.
In this way these methods relax the requirements of the \emph{Principle of Constraints}, i.e.\ that, ``Constraints shall be maintained
by forces, so long as this is possible; otherwise, and only otherwise, by impulses'' \cite[p.~79]{kilmister1966} 
\citep[as noted e.g.\ in][Section 1]{stewart1996implicit}.
Their advantage in avoiding many of the well explored physical 
paradoxes of rigid body mechanics (including Zeno phenomena \cite[]{drumwright2010avoiding} as well as apparent contradictions between 
frictional forces and impulses discussed in Section~\ref{sec:intfric}) seems to come at the cost of persistence 
of contact. In contrast, here, persistence is one of the key simplifying modeling assumptions, expressing our intuitive 
experience of limbs interacting with the world, enabling some of the other assumptions, and affording our strong formal results.
Despite being targeted at a different numerical integration scheme, many of the results in this paper, such
as the consistent handling of massless limbs, are potentially applicable to time-stepping schemes. 

\subsubsection{Hybrid Dynamical Systems}

This paper models manipulation and self-manipulation systems using a hybrid systems paradigm
that assumes instantaneous transitions.
Though we develop our (so-called) \emph{self-manipulation hybrid dynamical system} for a similar class of mechanical systems as that considered in~\cite[Ex.~3.3]{van1998complementarity}, we specialize from the more general class of \emph{hybrid automata} considered in~\cite[Def.~II.1]{LygerosJohansson2003} to facilitate connections with the broader hybrid systems literature.
Our self-manipulation system is closely related to the \emph{$n$-dimensional hybrid system} of \cite[Def.~2.1]{simic2005towards},
the \emph{simple hybrid system} of \cite[Def.~1]{OrAmes2011}, 
and \emph{hybrid dynamical system} of \cite[Def.~1]{BurdenRevzen2013} as we require:
(i) multiple disjoint domains of varying dimension, disallowed by \cite{simic2005towards, OrAmes2011}; 
(ii) guards with arbitrary codimension, disallowed by \cite{BurdenRevzen2013}; and we desire
(iii) more analytical and geometric structure than is provided by the general framework in \cite{LygerosJohansson2003}, specifically domains that are differentiable manifolds and guards that are sub-analytic. 
Note that (i) is precluded in \cite{simic2005towards, OrAmes2011} only for notational expediency since any multitude of domains may be embedded as disjoint submanifolds of a high-dimensional Euclidean space.
The condition (ii) is excluded by \cite{BurdenRevzen2013} since it is generally incompatible with the results contained therein.

One property of hybrid systems that is crucial to establish for the present setting
is that the guards are disjoint, i.e.\ no state is a member of two distinct guards, so there is no ambiguity as to which
reset map to apply. 
This key property yields the proof that the model is 
\emph{deterministic} \cite[Def.~III.2]{LygerosJohansson2003}.
Furthermore the system is set up such that every point on the boundary of the domain where the flow points
outward is a member of a guard, thus guaranteeing that the system is \emph{non-blocking} \cite[Def.~III.1]{LygerosJohansson2003}, 
i.e.\ the execution continues for infinite time.

The self-manipulation hybrid system developed in this paper uses the active contact constraints to define the discrete 
state or status (that we call the \emph{mode}), as in e.g.\ \cite{hurmuzlu1994rigid,brogliato2002numerical}.
However even when starting with a simple Lagrangian hybrid system without modes for every contact condition it appears
to be useful to add such states to allow executions to be completed beyond a so-called ``Zeno 
equilibrium'' \cite[]{ames_acc_2006,OrAmes2011}. 
Furthermore, the pseudo-impulse we introduce avoids certain Zeno executions 
by allowing the system to remain in a constrained mode after finitely many transitions, in a manner analogous 
to but formally distinct from the \emph{truncation} proposed in \cite{ames_acc_2006,OrAmes2011}.

\subsubsection{Consistent Complementarity}
\label{sec:lit:comp}

Any formulation that allows for persistent contact through impact must determine which contacts
to make active and which to remove\footnote{The removal ends up being the harder question, 
as ``there is no problem in deciding when and which 
constraint to add to the active set since there is a constraint function to base the decision on. 
The problem of dropping constraints is more delicate...'' \cite[p.~283]{lotstedt1982mechanical}.}.
When there is no impulse (i.e., no constraint to add, but one or more constraints have violated the
unilateral constraint cone), the removal process is called \emph{force--acceleration complementarity}, 
as it is commonly modeled by a complementarity problem involving contact force and 
separating acceleration, e.g.\ \cite[Eqn.~12]{trinkle1997dynamic}, \cite[Eqn.~10]{brogliato2002numerical},
where in the simplest case of a single contact point with zero or negative contact force it is simply removed.
This complementarity problem framework can introduce paradoxical consequences in certain physical problem 
settings, for example in taking the rigid limit of a deformable body \cite[]{chatterjee1999realism}. 
It can also be computationally efficient to relax the hard constraints of the complementarity conditions,
resulting in a convex optimization problem \cite[]{todorov2011convex,drumwright2011modeling}.

An impulse induced from one or more contact constraints becoming active generally necessitates the 
removal of other constraints, specifically, those that require a negative impulse to remain. 
When invoked as a modeling principle, this \emph{impulse--velocity complementarity} precludes 
a simultaneous impulse and separation velocity at a particular contact,
e.g.\ \cite[Eqn. 2.10b]{lotstedt1982mechanical}, \cite[Eqn.~9]{brogliato2002numerical}.
Imposing this modeling discipline affords the well established benefit of yielding a unique post-collision 
state for collisions modelled as plastic frictionless 
impacts \cite[]{ingleton1966problem,cottle1968problem,van1998complementarity,heemels2000linear}. 
Unfortunately further generalizations can lead to inconsistencies and 
ambiguities \cite[]{chatterjee1999realism,hurmuzlu1994rigid,ivanov1995multiple,seghete2010variational}.
The existence and uniqueness of a solution must therefore be separately 
established in each physical circumstance that includes friction -- or merely be~assumed.

Massless legs introduce new problems into the complementarity problem.
The massless leg condition in general, as introduced in \cite[Assumption~C.6]{johnson_selfmanip_2013} and also
used in countless prior works, e.g.\ \cite[]{blickhan1989spring,kajita1992dynamic,Holmes_Full_Koditschek_Guckenheimer_2006}, allows 
for the neglect of certain states deemed inconsequential to the dynamics of interest when 
unconstrained (of course, the appropriateness of this neglect is task dependent 
rather than in any way intrinsic to the underlying physics, c.f.\ \cite[Sec.~IV.C.5]{johnson_selfmanip_2013} 
or \citep{Balasubramanian_2008_6145}).
Indeed a massless leg that is not touching the ground is unconstrained and its position 
can be taken as arbitrary (or regarded as evolving according to dynamics sufficiently 
decoupled as to be considered independent), as used
in the behavior analysis in \cite[Sec.~IV.C.3]{johnson_selfmanip_2013}. 
However the complementarity condition as used in
e.g.\ \cite[Eqn. 2.10b]{lotstedt1982mechanical}, \cite[Eqn.~9]{brogliato2002numerical}, and listed in~\eqref{eq:UjP2}--\eqref{eq:UkP2} 
is ill-posed in the absence of mass since there is no well-defined separation velocity or acceleration, nor anything
precluding all massless contact points from always separating  
(at least for the dynamic model of interest here, as opposed to a quasistatic model, \cite[]{trinkle1995quasistatic}).
Instead here we reformulate the complementarity condition as~\eqref{eq:UjP}--\eqref{eq:UkP} to not
depend on the separation velocity.

\subsubsection{Impact Mechanics}
\label{sec:lit:impact}

The usual Newtonian impact law (as in e.g.\ \cite[Eqn.~3]{chatterjee1998new}, \cite[Eqn.~11.65]{featherstone2008rigid}
and many others) can be thought of as a mass-orthogonal projection onto the constraint
manifold as used in e.g.\ augmented Lagrangian techniques \cite[Eqn.~25]{bayo1996augmented}.
More generally, \cite{moreau1985} showed that impact problems can be modeled using measure differential inclusions.
The algebraic plastic impact law involves inversion of the inertia tensor, which precludes the 
possibility of massless limbs and necessitates the reformulation given in this paper. 
Even if there are no truly massless links, a nearly massless body segment 
yields a poorly-conditioned inertia tensor \cite[Sec.~5.1.1]{thesis:johnson-2014}, leading to similar formulations 
as the one presented here in \cite[Eqn.~9]{Westervelt_Grizzle_Koditschek_2003}
or, for continuous time dynamics, 
\cite[Sec.~4.3]{Holmes_Full_Koditschek_Guckenheimer_2006} \cite[Eqn.~3.17]{featherstone2008rigid}. 

In this paper we restrict our attention to systems modeled as exhibiting only perfectly plastic impact 
(perfectly inelastic impact). In the elastic impact case, it is necessary to consider the relative 
stiffness of contact points; depending on the restitution law invoked, multiple outcomes are consistent 
with the constitutive assumptions \cite[]{hurmuzlu1994rigid, chatterjee1998new}. Though it is possible to bypass 
this technical obstacle by introducing an additional constitutive hypothesis, e.g.\ \cite[H3~in Sec.~3.3]{Ballard2000}, 
it remains to be validated (either theoretically or experimentally) that such assumptions accurately 
represent the physical system's behavior. Plastic impact avoids these
inconsistencies, but more importantly we claim plastic impact provides a more
useful model of the robotic systems of interest. Elastic impact is clearly needed in some
robotics applications such as juggling \cite[]{BuehlerKoditschek1994,schaal1993open}, tapping \cite[]{huang1998experiments} 
or ping-pong \cite[]{Andersson_1989}, but plastic impact, where there is no restitution and therefore no separation velocity after impact,
is a more desirable model for most forms of locomotion (when it is important to keep feet on the 
ground) \cite[]{Westervelt_Grizzle_Koditschek_2003,chatterjee2002persistent} and manipulation
(when it is important to keep fingers on the object) \cite[]{chatterjee2002persistent,wang1987modeling}. 

The new pseudo-impulse presented here, in addition to the Zeno results mentioned above, eliminates other evidently unwanted
transitions by allowing the continuous-time forces to 
play a role in the impact process which is primarily ``logical'' (as opposed to energetic). This role
may be best summarized by comparison to the most common alternatives. For example, instead of introducing a variable
coefficient of restitution \cite[]{quinn2005finite} (which our plastic impacts of interest already eliminate),
the pseudo-impulse is not applied to the continuous (energetic) system directly but instead used to 
regularize the complementarity driven hybrid switching logic. Or as a second point of comparison, rather than introducing a 
fixed dead zone in impact energy \cite[]{pagilla2001stable} or velocity \cite[Sec.~6.4]{brogliato2002numerical}, the magnitude of the effect on our model's 
hybrid logic is not fixed but rather scales with the continuous time forces.
An effect similar to this pseudo-impulse condition
is also introduced by time-stepping
simulations \cite[]{stewart1996implicit,anitescu1997formulating}, 
which, true to their name, always consider forces over small but finite time-steps. Under such schemes 
the magnitude of this effect is not a fixed, independent, user-imposed parameter since
it must remain proportional to the duration of each
time-step. Our preference for the independent, fixed choice reflects both mathematical convenience 
(the clearly defined hybrid dynamical system with its formal properties) as well as our taste in 
preferring to work with robotics models targeted for specific physical environments and settings.

\subsubsection{The Effect of Friction Models}
\label{sec:intfric}

While this paper focuses on the impact problem, which friction greatly 
complicates \cite[]{keller1986impact,wang1987modeling,McGeer_Wobbling_1989,wang1994simulation,trinkle1997dynamic}, 
even simulating continuous-time dynamics of rigid bodies with friction can be difficult (formally $NP$-hard \cite[]{baraff1991coping})
due to the possibility of ``jamming'' events \cite[]{mason1988inconsistency,dupont1994jamming},
first attributed to \cite{painleve1895leccons}. In this paper, following the model from \cite{johnson_selfmanip_2013},
strong assumptions about frictional contact avoid these issues and enable integration of the dynamics as a differential-algebraic equation (DAE).
As noted above, an alternative method to numerically solving these problems is
the time-stepping approaches pursued in, e.g., \cite{stewart1996implicit,anitescu1997formulating},
which resolve these issues by allowing for impulses at any time step.
To resolve these issues in more general extensions of the system presented here
(in particular those that are not well modeled by the frictional assumption, \ref{ass:friction}),
the hybrid dynamical system could similarly be extended by allowing impulses at times without collisions,
with such jamming events considered with additional guards and reset maps.
We refer the interested reader to ``Is Painlev\'{e} a real obstacle?'' \cite[Sec.~8.1]{brogliato2002numerical} for further discussion of these issues.

\begin{table}[t]
  \begin{center}
    \begin{tabular}{ l l}
      \toprule
      \midrule
      $\bfa:\mathcal{Q} \rightarrow \mathcal{C}$ & Base constraint function (\ref{sec:not})\\
      $\bfA:T\mathcal{Q} \rightarrow T\mathcal{C}$ & Velocity constraint function (\ref{sec:not})\\
      $\bfAd:T^*\mathcal{Q} \rightarrow T^*\mathcal{C}$ & Force constraint function (\ref{eq:astardef})\\
      $C^r\!, r \in \Nat\cup\set{\infty,\omega}$ &  $C^r$ differentiable function (\ref{sec:not})\\
      $\CP_\PRED:T\calQ  \rightarrow 2^\calK$ & Solution to the $\PRED$ predicate (\ref{eq:CP})\\
      $\Cbar: T\calQ^2 \rightarrow T^*\calQ$ & Coriolis forces (\ref{eq:dyn})\\
      $\FA:2^\calK \times T\calQ \rightarrow \TF$ &  Force--acceleration predicate (\ref{eq:CPFA})\\
      $i,j,k \in \mathcal{K}$ & Contact constraints (\ref{sec:not})\\     
      $I,J,K \subseteq \mathcal{K}$ & Set of active contact constraints (\ref{sec:not})\\
      $\calI \subseteq \mathcal{K}$ & Complementarity scope (\ref{eq:IScope})\\
      $\Id, \Id_\bfq$ & Identity matrix, of dimension $|\calQ|$ (\ref{sec:not})\\
      $\IV:2^\calK \times T\calQ \rightarrow \TF$ &  Impulse--velocity predicate (\ref{eq:CPIV})\\
      $\mathcal{K}:=\KN \cup \KT \subset \mathbb{N}$& Set of all contact constraints (\ref{sec:not})\\
      $\Mbar : T^2 \mathcal{Q} \rightarrow T^*\mathcal{Q}$ & Inertia tensor (\ref{sec:astar})\\
      $\Md : T^*\mathcal{Q} \rightarrow T^2\mathcal{Q}$ & Constrained inverse inertia tensor (\ref{eq:astardef})\\
      $\Nbar: \calQ \rightarrow T^*\calQ$ & Potential forces (e.g.\ gravity) (\ref{eq:dyn})\\
      $\NTD:T\calQ \rightarrow \TF $ & New touchdown predicate (\ref{eq:NTD})\\
      $\qimp \in T^*\mathcal{Q}$ & Impulse in state space (\ref{sec:impactmap})\\
      $\limp, \dimp \in T^*\mathcal{C}$ & Impulses in constraint space (\ref{eq:impulse}), (\ref{eq:pseudoimp})\\
      $\PIV:2^\calK \times T\calQ \rightarrow \TF$ &  Pseudo-impulse $\IV$ predicate (\ref{eq:PIV})\\
      $\bfq \in \mathcal{Q}:=\Theta \times SE(\mathrm{d})$ & Continuous state (\ref{sec:not})\\
      $\biq:=(\bfq,\dot{\bfq}) \in T\mathcal{Q} $ & Continuous state and velocity (\ref{sec:not})\\
      $\TD: \KN \times T\calQ \rightarrow \TF$ & Touchdown predicate (\ref{eq:TDsimp})\\
      $\bfU : T^*\mathcal{C} \rightarrow \Real^{|\calC|}$ & Unilateral constraint cone (\ref{sec:not})\\
      $\alpha:\calK \into \KN$ & Corresponding normal constraint \eqref{eq:alpha}\\
      $\delta_t \in \Real^+$ & Small time duration of impact (\ref{eq:pseudoimp})\\
      $\Delta\dot{\bfq} \in T\mathcal{Q}$ & Instantaneous change in velocity (\ref{sec:impactmap})\\
      $\lambda \in T^*\mathcal{C}$ & Lagrange multipliers \eqref{eq:ldyn}\\
      $\Lambda:T^2\mathcal{C} \rightarrow T^*\mathcal{C}$ & Constrained contact inertia tensor (\ref{eq:astardef}) \\
      $\Upsilon \in T^*\calQ$ & External forces and torques (\ref{eq:dyn})\\
      $\prec, \succ, \preceq, \succeq, \equiv$ & Trending negative/positive (Def.~\ref{def:through})\\
      \bottomrule
    \end{tabular}
     \caption{Key symbols used throughout this paper, with section or equation number of introduction marked. See also Table~\ref{tab:Symbols2} for symbols introduced in Section~\ref{sec:hyb}.
     }
     \label{tab:Symbols}
  \end{center}
\end{table}

\section{Modeling Assumptions}
\label{sec:imp}

The continuous Lagrangian dynamics of self-manipulation is specified in \cite{johnson_selfmanip_2013} 
using the notation and terminology of \cite{book:mls-1994} and summarized in Section~\ref{sec:not}.
We continue to work within that framework here and briefly list the subtle differences between these
two classes of systems in Section~\ref{sec:manip}. 
However the impulsive dynamics (instantaneous changes in velocity when a new contact is added)
were not specified in either, and so \changed{we introduce} a plastic impact model in Section~\ref{sec:impactmap} 
and explore the induced complementarity conditions in Section~\ref{sec:complementarity}. 
In addition, \changed{we} make explicit how the massless leg (Section~\ref{sec:astar}) and frictional assumptions (Section~\ref{sec:friction})
made in \cite{johnson_selfmanip_2013} affect both the continuous time (Section~\ref{sec:contdyn}) and impulsive dynamics,
leading to a new formulation for the dynamics that is equivalent to the usual formulation when there are no massless links.
Finally, Section~\ref{sec:pseudoimpulse} introduces a new pseudo-impulse that eliminates certain Zeno executions and related chattering behavior.

\subsection{Setup and Notation}
\label{sec:not}

The notation used in this paper is chosen to be consistent with \cite[Table~I]{johnson_selfmanip_2013} 
(and agreeing where possible with \cite{book:mls-1994})
or is defined as it is used and summarized in Table~\ref{tab:Symbols}. 
The base 
component of the state is denoted, $\bfq \in \calQ$,
while the full \emph{state} is, $\biq := ( \bfq,\dot{\bfq} )$, and
this state completely describes the motion of interest, as, 
\begin{assumption}[Rigid Bodies]
\label{ass:rigid}
The robot is made up of a finite number of rigid bodies whose configuration
lies in a connected complete $C^\omega$ Riemannian manifold $\e{Q}$.
\end{assumption}

Since the configuration spaces of many extant robots are not linear (e.g., due to rotary joints, rigid body rotations, or constrained mechanisms), it is most natural to invoke the general framework of differentiable manifolds to model the state space.
For concreteness we consider the case where $\calQ:= \Theta \times SE(\mathrm{d})$
consists of joint angles and the special Euclidean group of dimension $\mathrm{d}$, 
but our formal results are stated for an arbitrary connected complete $C^\omega$ Riemannian manifold $\calQ$. 
We recognize that this generality necessitates mathematical formalisms and notation that are not uniformly adopted in the robotics community (exceptions such as~\cite{book:mls-1994} notwithstanding);
we aim whenever possible to translate unfamiliar objects into standard terminology and provide a terse overview of the background material needed to parse the more general case in Appendix~\ref{app:dg}.

We are concerned with sets of contact constraints (e.g., $I,J,K\subset\e{K}$) that we shall call
\emph{modes} or \emph{contact modes} hereinafter, subsets of indices whose particular elements
(e.g., $i, j, k\in \e{K}$) index the contact constraints that prevail at some 
instant \cite[Sec.~II.C]{johnson_selfmanip_2013} \cite[Sec.~5.2.1]{book:mls-1994}.
In addition to contact with the robot's environment, contact constraints may include 
cases of self-contact as well as joint limits.
The universe of all possible constraint indices from which these subsets are taken is denoted $\calK=\KN\cup\KT$,
partitioned by those that are in the \emph{normal} (non-penetrating) direction and
those that are in \emph{tangential} (non-sliding) direction. Similarly, for any set 
of constraints specified by mode $I$, define the subsets $I_n := I \cap \KN$ and $I_t := I \cap \KT$, where
clearly $I = I_n \cup I_t$ and $I_n \cap I_t = \varnothing$.

Contact constraints in the normal direction\footnote{Note that normal direction constraints
for non-adhesive contact is unilateral, although within a contact mode they can be considered bilateral until
the constraint force is violated \citep[e.g.][Sec.~4]{lotstedt1982mechanical}.}, $i\in\KN$, specify a 
holonomic constraint of the form $\{(\bfq,\dot{\bfq}) \in T\e{Q} :$ $\bfa_i(\bfq) = 0\}$ 
for $\bfa_i:\calQ\rightarrow \Real$ 
(and whose corresponding velocity constraint $\bfA_i:T\calQ\rightarrow \Real$ is
equivalent to the Jacobian $D\bfa_i$, \cite[Eqn.~11]{johnson_selfmanip_2013}),
while those in the tangential direction, $i\in\KT$, specify a 
nonholonomic constraint of the form $\set{(\bfq,\dot{\bfq}) \in T\e{Q} : \bfA_i(\bfq)\dot{\bfq} = 0}$ where again $\bfA_i:T\calQ\rightarrow \Real$.
For a given contact mode $I$, the space of constrained positions is a manifold $\calC_I$ of dimension $|I|$ 
(i.e., the number of constraints in $I$).
 
In the interest of notational clarity, we generally express functional dependence on contact modes via subscript, 
e.g., $X_I(\bfq,...) := X(I,\bfq,...)$, and when it is clear from context, we further
suppress the subscript, e.g.\ $X(\bfq,...)$. 
For example, and used extensively throughout this paper,
fixing an ordering on $\e{K}$ we obtain the velocity constraints active in mode $I$, $\bfA_I:T\calQ \rightarrow T\calC_I$, as a selection of rows from the set of all velocity constraints $\bfA_\calK$,~i.e.,\footnote{However,
note that most functions of the mode are not a simple
projections, and so e.g.\ $\bfAd_I$, defined in~\eqref{eq:astardef}, $\bfAd_I \neq \pi_I \bfAd_\calK$,
but rather $\bfAd_I$ is as defined in~\eqref{eq:astardef}, i.e.\ constructed with the corresponding~$\bfA_I$.}
\begin{align}
\bfA_I(\bfq) := \bfA(I,\bfq) = \pi_I \bfA_\calK(\bfq),
\end{align}
where $\pi_I$ is the Boolean projection matrix formed by the rows of canonical unit vectors associated with the elements in the index set $I$. 
Similarly for a single constraint $i$, 
$\bfA_i := \pi_i \bfA_\calK = \bfA_{\{i\}}$.

We make the following assumption on the combined maps,
\begin{assumption}[Simple Constraints]
\label{ass:rank}
All constraints are independent, that is for all contact modes $I$, the maps $\bfa_{I_n}:\calQ\rightarrow\calC_{I_n}$ and 
$\bfA_{I}:T\calQ \rightarrow T\calC_{I}$ are constant rank.
\end{assumption}
We refer the reader to Appendix~\ref{app:dg} for the definition of rank of a $C^r$ map; in coordinates, this condition states that the gradient vectors of each coordinate of the respective maps are linearly independent at every point.
If this condition failed to hold, the configuration space could possess singularities that could preclude existence and/or uniqueness of trajectories for the mechanical system.
Note that this precludes the possibility of redundant constraints, though there are methods
of resolving such redundancies, e.g.\ in \cite{greenfield2005solving}.
In particular, this requirement is met if $\bfa_{\KN}\in C^\omega(\e{Q},\Real^{\abs{\KN}})$ and $\bfA_{\e{K}}\in C^\omega(T\e{Q},\Real^{\abs{\e{K}}})$ 
are constant rank\footnote{This stronger assumption would not be true if there were two parallel constraints that, due to geometry, 
could not simultaneously be active, in which case the original requirement must be checked for all $I$.}.

We note that there is an assignment,
\begin{align}
\alpha:\calK\into\KN, \qquad \alpha|_{\KN} = \Id,
\label{eq:alpha}
\end{align}
of contacts to normal contacts such that $\alpha|_{\KT}$ maps tangential contacts to the corresponding normal contact
(where $\Id$ is the appropriate identity matrix).
Note that
for each $k\in\KT$ and $j=\alpha(k)$, $\bfA_k$ is orthogonal to $\bfA_{j}$.

It is well established that 
the motion of mutually constrained rigid bodies can be effectively modeled using polynomial 
maps \cite[]{Wampler_Sommese_2013}, hence imposing contact constraints arising from their 
interaction with the piecewise polynomial representations of the environment (commonly 
adopted by the sensory community \cite[]{Lalonde_Vandapel_Hebert_2007}) leads to,
\begin{assumption}[Analytic Constraints]
\label{ass:analytic}
All constraints are analytic functions, that is for all 
contact modes $I$, the maps $\bfa_{I_n}:\calQ\rightarrow\calC_{I_n}$ and 
$\bfA_{I}:T\calQ \rightarrow T\calC_{I}$ are $C^\omega$.
\end{assumption}

Given an analytic vector field subject to an analytic constraint, as shown in Lemma~\ref{lem:trend} it is possible to determine whether the constraint remains active over a nonzero time horizon by evaluating Lie derivatives at a single instant in time.
If either the vector field or constraint were merely smooth, the differential equation determined by the vector field would, in general, need to be solved over a nonzero time horizon to determine whether the constraint remained active.

\begin{assumption}[Persistent Contact]
\label{ass:contact}
Contact with the world occurs through a finite number of active constraints indexed by $I \subset \calK$ 
that apply continuous time forces. 
Furthermore, contact persists until the next event (e.g.\ touchdown or liftoff).
\end{assumption}

This assumption is related to the \emph{Principle of Constraints}, as discussed in Section~\ref{sec:sim}.
Its adoption partitions trajectories so that at all times between instantaneous touchdown or liftoff 
events there persists a well-defined set of active constraints (enabling the systematic a priori enumeration and
analysis of these constraint sets and their sequences, e.g.\ \cite{paper:johnson-icra-2013}).
This contrasts with simulations generated by time-stepping algorithms, wherein contact~\cite{stewart1996implicit} or interpenetration~\cite{anitescu1997formulating} are resolved only at multiples of the timestep, and no distinction between
forces and impulse are made (indeed this relaxation is what enables the efficient and consistent simulation in such formulations).

The impact problem can be summarized as determining which constraints to add or remove from
the active set. The active set continues to constrain the system so long as the unilateral
constraint cone \cite[Eqn.~7]{johnson_selfmanip_2013} is positive, $\bfU(\lambda)\geq0$, 
where $\lambda \in T^*\mathcal{C}$ is the vector of Lagrange
multipliers (constraint forces) \cite[Eqn.~33]{johnson_selfmanip_2013}. 
Included in $\bfU$ is both the non-attachment condition 
that normal direction forces are positive as well as the friction cone that relates the magnitude
of the normal and tangential components.

In the complementarity problems, the following definition simplifies
statements involving higher-order derivatives of the state
that seem to arise unavoidably 
(as stated in \cite[Sec.~3]{van1998complementarity}, \cite[Sec.~1]{heemels2000linear}, 
formalizing the concepts represented in e.g.\ \cite[Fig.~11.4]{featherstone2008rigid}, \cite[Sec.~27.2]{handbook}),

\begin{definition}
\label{def:through}
Given a smooth function $h:M\into\Real$ \emph{defined over a smooth manifold $M$}, a point $x \in M$, and a smooth vector field $F:M\into TM$, 
we say that $h$ is \emph{trending negative} with respect to the vector field $F$ at $x$, denoted $h(x) \prec_F 0$,
(or $h(x) \prec 0$ if the context specifies $F$),~if,
\begin{align}\label{eq:thru}
  \exists\; m \geq 0 : (\Lie^m_F h)(x) < 0 \wedge \forall\; \ell < m : (\Lie^\ell_F h)(x) = 0,
\end{align}
where $\Lie^m_F h:M\into\Real$ is the $m^{th}$ Lie derivative\footnote{See e.g.\ \cite[Ch.~9]{Lee2012}, 
and note our convention that, $\Lie_F^0 h = h$.} of $h$ with respect to the vector field $F$.
Similarly, we say that $h$ is \emph{trending positive} at $x$, denoted $h(x)\succ 0$, when $-h(x)\prec 0$.
We say that $h$ is \emph{identically zero} at $x$, denoted $h(x) \equiv 0$, when $\forall\; \ell \in \mathbb{N} : (\Lie^\ell_F h)(x) = 0$.
Finally, we say that $h$ is \emph{trending non-negative} at $x$, denoted $h(x)\succeq 0$, when $h(x)\succ 0$ or $h(x) \equiv 0$, 
and that $h$ is \emph{trending non-positive} at $x$, denoted $h(x)\preceq 0$, when $h(x)\prec 0$ or $h(x) \equiv 0$.
\end{definition}

\noindent
We refer the reader to Appendix~\ref{app:dg} for the definition of a vector field $F:M\into TM$; in the case where $M = \Real^n$, the tangent bundle $TM$ can be canonically identified with $\Real^n$ to obtain a more familiar function $\td{F}:\Real^n\into\Real^n$ that determines an ordinary differential equation $\dot{\td{x}} = \td{F}(\td{x})$.

That is, $h(x) \prec 0$ if and only if the following vector,
\begin{align}
\left[h(x), \;  (\Lie_F h)(x), \;  (\Lie^2_F h)(x), \; ... \right], \label{eq:trendvec}
\end{align}
is lexicographically smaller than zero \cite[Def.~3.5]{Bertsimas_1997}. 
As an example of when these properties are important, consider the examples in Figure~\ref{fig:ptex}.
In each case, the initial configuration (taken as the bottom point of the circle)
is $\bfq = [x\;y]^T=[0\;0]^T$, and the initial velocity is $\dot{\bfq} = [v\;0]^T,\ v>0$. Assume that the particle
has unit mass ($\Mbar = \Id_\mathrm{2}$), and that there are no non-contact forces ($\Nbar=\mathbf{0}$, 
$\Cbar=\mathbf{0}$, $\Upsilon =\mathbf{0}$, as defined in Section~\ref{sec:contdyn}).
In all cases the particle is touching the constraint ($\bfa(\bfq)=0$) 
but has no impacting or separating velocity ($\bfA\dot{\bfq}=0$), so
there is no impulse (as defined in Section~\ref{sec:impactmap}). 
Furthermore in c) and d) there is no impacting or separating acceleration ($\bfA\ddot{\bfq}+\dot{\bfA}\dot{\bfq}=0$).
However in a) and c) the constraint function is trending positive, $\bfa(\bfq)\succ0$, while
in b) and d) the constraint function is trending negative $\bfa(\bfq)\prec0$. 

\begin{figure}
\centering
\def\svgwidth{16cm}
\include{ptex}
  \vspace{-8pt}
\caption{
Four examples of a planar point particle ($\e{Q}=\Real^2$) with a single constraint ($\e{K} = \set{1}$), defined as 
(a) $\bfa=x^2 + 4 y$, 
(b) $\bfa=-x^2+4y$, 
(c) $\bfa=x^3+8y$, 
and 
(d) $\bfa=-x^3+8y$. 
Note that if the particle velocity is directed to the right ($\dot{\bfq} = [v\;0]^T,\ v > 0$), as illustrated, then:
the constraint function is trending positive ($\bfa(\bfq)\succ0$) in (a) and (c); 
the constraint function is trending negative ($\bfa(\bfq)\prec0$) in (b) and (d). 
}
\label{fig:ptex}
\end{figure}

Furthermore, we make use of the following properties of this trending relation,
\begin{lemma}
\label{lem:closure}
The closure of $\{x : h(x) \prec 0\}$ or $\{x : h(x) \preceq 0\}$ is $\{x:h(x)\leq 0\}$, while the closure of
$\{x : h(x) \succ 0\}$ or $\{x : h(x) \succeq 0\}$ is $\{x:h(x)\geq 0\}$.
\end{lemma}
This is easy to see as $\{x:h(x)<0\} \subset \{x:h(x)\preceq0\} \subset \{x:h(x) \leq 0\}$ for any vector field.

\begin{lemma}
Given a smooth vector field, $F:M \into TM$, a point in a smooth boundaryless manifold, $x \in M$, and a smooth positive function, $g:M\rightarrow \Real^+$, any other smooth function, $h:M \into \Real$, is trending negative if and only if its product with $g$ has the same property, i.e.,
\begin{align}
h(x) \prec_F 0 \Leftrightarrow g(x) \cdot h(x) \prec_F 0.
\end{align}
\label{lem:twofun}
\end{lemma}

\noindent
See Appendix~\ref{app:twofun} for a proof.

\begin{lemma}\label{lem:trend}
Let $h:M \into \Real$ be a $C^\omega$ function and $F:M \into TM$ be a $C^\omega$ vector field over a $C^\omega$ boundaryless manifold $M$, 
and let $\chi:(-\varepsilon,+\varepsilon)\into M$ denote an integral curve for $F$ through $x := \chi(0)$.
Then $h$ is trending positive at $x$ with respect to $F$, 
$h(x) \succ_{F} 0$,
if and only if there exists $\delta \in (0,\varepsilon)$ such that,
\eqnn{
\forall\; s\in(0,+\delta) : h\circ\chi(s) > 0.
}
\end{lemma}
\noindent
The requirement that the manifold be boundaryless is introduced to simplify the statement of this Lemma;
the Lemma clearly applies to the interior of a manifold with corners (which is, after all, simply 
a manifold without boundary) \cite[Def.~2.1]{Joyce2012}.

To see that the lemma is true, note that if $\chi$ is an integral curve for $F$ such that $h\circ\chi(s)$ 
is positive for $s > 0$ sufficiently small, then since $h$ is analytic we conclude~\eqref{eq:thru} is satisfied.
The other direction follows easily by contradiction using the mean value Theorem.
We note that this is not true if $h$ or $F$ are merely $C^\infty$. Also note that the conditions of the lemma
do not imply that  $\forall\; s\in(-\delta,0) : h\circ\chi(s) < 0$ for two reasons: 1) it is possible that
$h(x)\neq 0$, and 2) even for $h(x)=0$, grazing contact would handled incorrectly, 
\changed{as in Figure~\ref{fig:ptex}, example (a)}.

Lemma~\ref{lem:trend} implies a 
computationally efficient way to test these trending conditions is to simply integrate a flow until
it reaches a zero crossing\changed{, as discussed further in Section~\ref{sec:dis:numerical}}.

\subsection{Manipulation and Self-Manipulation}
\label{sec:manip}

This section briefly summarizes the self-manipulation formalism introduced in \cite{johnson_selfmanip_2013},
as it relates to manipulation, e.g.\ as presented in \cite{book:mls-1994}. Each defines a number of frames on
the robot and its environment -- the palm frame, the object frame, the contact frame, etc. In an effort
to keep the problems as similar as possible, the following conventions were adopted in \cite{johnson_selfmanip_2013},
\begin{itemize}
\item In self-manipulation, the robot is the object being manipulated and so to properly consider 
the forces and torques on the object the robot's palm frame, $\itP$, and the object frame, $\itO$, 
are chosen to be coincident, \cite[Sec.~II-B]{johnson_selfmanip_2013}.
\item Thus motions that, in a manipulation problem, move an object to the right really move 
the robot to the left, and so the self-manipulation grasp map (a component of $\bfA$) is a reflection of 
the manipulation grasp map, $\bfG_s:=-\bfG$, \cite[Eqn.~15]{johnson_selfmanip_2013}.
\item By convention the contact frame is defined at any point of contact with the $z$-axis pointing into the 
object (away from the finger tip), \cite[Sec.~5.2.1]{book:mls-1994}. In self-manipulation the convention 
of \cite[Sec.~II-C]{johnson_selfmanip_2013} is to keep the contact frame consistent with respect to the legs,
and so the $z$-axis points away from the robot and into the ground. This results in a unilateral
constraint cone,~$\bfU$, that is negative, \cite[Eqn.~76,~78]{johnson_selfmanip_2013}.
\item Since the palm reference frame is accelerating with respect to the world, the inertia tensor,~$\Mbar$,
\cite[Eqn.~26]{johnson_selfmanip_2013}, and by extension the Coriolis terms,~$\Cbar$, \cite[Eqn.~30]{johnson_selfmanip_2013}, 
are more coupled and lack the block diagonal structure present in manipulation problems, \cite[Eqn.~6.24]{book:mls-1994}.
\end{itemize}
It should be no surprise that the problem formulations are structurally equivalent since the underlying
kinematics and dynamics are indifferent to the problem category. 
However
owing to the notational differences summarized above, through the remainder of this paper we choose to
write out the problems in terms of a self-manipulation system, with the understanding that the results
contained herein apply equally well to manipulation systems once these transformations are incorporated.

\subsection{Massless Considerations}
\label{sec:astar}

To properly define the dynamics of a partially massless system, consider a parametrized family 
of singular semi-Riemannian metrics,
\begin{align}
\Mbare(\bfq) : \calQ \times [0,\bar{\epsilon}] \rightarrow \bbR^{\rmq \times \rmq}, \label{eq:mbare}
\end{align}
such that $\Mbar_0(\bfq) := \Mbar$ is the (possibly) degenerate inertia tensor for the 
system \cite[Eqn.~26]{johnson_selfmanip_2013} and may be
singular, while $\epsilon$ assigns a small mass and inertia to any putatively massless 
links such that $\Mbare(\bfq)$ is full-rank for all $\epsilon>0$ (for our present purposes, it is sufficient to 
use a limiting model such as $\Mbare := \Mbar_0 +\epsilon \mathbf{Id}_\rmq$ rather than some more specific physically motivated one). 
We invoke the general definition of Riemmanian metric here since it provides the coordinate-invariant formulation of the familiar mass or inertia matrix associated with a collection of rigid bodies, and refer the reader to~\cite{Lee1997} for a formal definition and Section~\ref{sec:contdyn} for additional details.
The dynamics of the system in contact mode $I$ can be expressed (as shown below) using the inverse of the following block
matrix containing $\Mbare$, defining\footnote{ Note 
that \cite[Eqn.~40]{johnson_selfmanip_2013} used the notation $\bfA^*$ while in this paper we use $\bfAd$ to
signify the slight difference in definition used here, and to avoid confusion with the pullback of $\bfA$, usually noted as $\bfA^*$, but which happens to be $\bfAdT$.} $\bfAd:T^*\mathcal{Q}\rightarrow T^*\mathcal{C}$, $\;\Md:T^*\calQ \rightarrow T^2 \calQ$, and $\Lambda:T^2\calC \rightarrow T^*\calC$ as,
\begin{align}
&\left[ \begin{array}{ll}
\Md_I & \bfAdT_I \\
\bfAd_I & \Lambda_I
\end{array} \right]  := \lim_{\epsilon->0}\left( \left[ \begin{array}{cc}
\Mbar_\epsilon& \bfA^T_I\\
\bfA_I & 0 
\end{array} \right]\right)^{-1}  \label{eq:astardef} \\
&\qquad=\left(\lim_{\epsilon->0} \left[ \begin{array}{cc}
\Mbar_\epsilon& \bfA^T_I\\
\bfA_I & 0 
\end{array} \right]\right)^{-1}=  \left[ \begin{array}{cc}
\Mbar_0& \bfA^T_I\\
\bfA_I & 0 
\end{array} \right]^{-1}.
 \label{eq:astardefe}
\end{align}
From this definition, note that the following properties hold,
\begin{align}
\bfAd \bfA^T=\bfA \bfAdT&=\Id, \qquad &
\Md \bfA^T = \bfA \Md = 0, \label{eq:dagprop}\\ 
\Md\Mbar + \bfAdT \bfA &= \Id, \qquad &
\bfAd\Mbar+\Lambda\bfA = 0. \label{eq:dagprop2}
\end{align}

To ensure that the inverse of the matrix in \eqref{eq:astardefe} (sometimes called the ``Lagrangian matrix 
of coefficients'', e.g.\ \cite[Eqn.~7.79]{papalambros2000principles}, and
sometimes used in robotics for numerical reasons, e.g.\ \cite[Sec.~4.3]{Holmes_Full_Koditschek_Guckenheimer_2006})
is well-defined, we require some modeling assumptions on the nature of
the massless appendages. Thus if the inverse exists, 
this $\epsilon$-parametrized curve takes its image in $GL(n)$ (the group of invertible matrices over $\Real^n$) 
within which matrix inversion is a continuous operation, hence the limit commutes with the inverse operation, 
and $\Mbare^\dagger$ is a well defined smooth curve defined over all $\epsilon\in[0,\bar{\epsilon}]$.

To meet this requirement, massless appendages are allowed here only in a limited form,
\begin{assumption}[Constrained Massless Limbs]
\label{ass:cml}
For all limbs in contact with the world, any rank deficiencies of the inertia 
tensor~$\Mbar$ \cite[Eqn.~26]{johnson_selfmanip_2013} are ``corrected''
by velocity constraints~$\bfA$ sufficient to guarantee that any remaining allowed physical movement excites some 
associated kinetic energy, that is, the block matrix in~\eqref{eq:astardefe} is invertible. 
\end{assumption}

If the ``rank correction'' condition in this assumption were violated, then it would not be possible in general to determine the system's instantaneous acceleration solely from the internal, applied, and Coriolis forces; it could happen that either no accelerations are consistent with the net forces, or an infinite set of accelerations are.
This condition admits its most physically straightforward
expression via the requirement that the inertia tensor is nonsingular when written with
respect to generalized or reduced coordinates, $\td{\mathbf{M}}$ (i.e., any local chart arising from an implicit
function solution to the constraint equation \cite[Eqn.~10]{johnson_selfmanip_2013}). 
However, for purposes of this paper, we find it more useful to work with the Lagrange-d'Alembert 
formulation of the constrained dynamics, \cite[Eqn.~33]{johnson_selfmanip_2013},  
hence, we translate that natural assumption into more formal algebraic terms governing 
the relationship between the lifted (velocity) constraints, $\bfA$ \cite[Eqn.~11]{johnson_selfmanip_2013}, 
and the overall inertia tensor~$\Mbar$ as follows,
\begin{lemma}
\label{thm:mlequiv}
The matrix~$\bigl[\begin{smallmatrix}\Mbar & \bfA^T \\ \bfA& 0\end{smallmatrix} \bigr]$,~\eqref{eq:astardefe},
is invertible if and only if the inertia tensor expressed in generalized 
or reduced coordinates, $\td{\mathbf{M}}$ \cite[Eqn.~36]{johnson_selfmanip_2013}, is invertible
\cite[Sec.~II.K, Assumption~A.4]{johnson_selfmanip_2013}.
\end{lemma}
\noindent
as shown in Appendix~\ref{app:pfmle}. See Section~\ref{sec:dis:ml} for a discussion of physical scenarios that meet 
this requirement.

When not constrained on the ground, any such massless links or limbs must then be removed from consideration as 
mechanical degrees-of-freedom: since they are massless, when unconstrained, the associated 
joints can be considered to have arbitrary configuration. Their evolution is instead treated according to the principle,
\begin{assumption}[Unconstrained Massless Limbs]
\label{ass:uml}
For all limbs not in contact with the world, any components of the state that do not excite some kinetic energy 
must be removed from the usual dynamics and instead considered to evolve in isolation according to 
some independent, decoupled dynamics.\footnote{ 
That is, in contact mode $I$, the configuration manifold $\e{Q}$ 
decomposes as a product of manifolds $\e{Q} = \e{Q}_I \times \td{\e{Q}}_I$, 
where $\e{Q}_I$ corresponds to a subset of the system coordinates such that the matrix in~\eqref{eq:astardefe} is nonsingular,
and $\td{\e{Q}}_I$ corresponds to the remaining coordinates.
The dynamics for the coordinates of $\td{\e{Q}}_I$ is given by some vector field $\td{F}_I$.
Here we have written
the dynamics as a second order vector field so that the dynamics of the full system may be written in a notationally
consistent manner.
This is not required; regardless of how the dynamics are defined for these coordinates, there is no coupling of energy with the rest of the system through the inertia tensor. \label{fn:uml}}
\end{assumption}

In the same vein as the remark following Assumption~\ref{ass:cml} (Constrained Massless Limbs), 
we observe that it is not possible to uniquely determine accelerations of unconstrained massless limbs due to corresponding degeneracy in the inertia tensor.
Excluding such limb states from the coupled Lagrangian mechanics governing the remaining body and limb segments enables us in the sequel to specify a differential-algebraic equation that admits unique solutions.
As the dynamics of the excluded states do not affect those of the remaining states, for the rest of this section
we abuse notation and suppress the subscript $I$ from the state space $\calQ$, so that 
unless stated otherwise we are concerned with only the ``active'' component $\calQ_I$ of the 
decomposed state space for the mode of interest.
See also Section~\ref{sec:zenodisc} for a discussion of Zeno (Def.~\ref{def:zeno}) considerations with massless legs.

\subsection{Continuous Dynamics}
\label{sec:contdyn}

With this notation, the continuous-time dynamics of \cite[Eqn.~33]{johnson_selfmanip_2013} 
in contact mode $I$ can be expressed as,
\begin{align}
\ddot{\bfq}_I &:= \Md_I \left(\Upsilon_I  - \Cbar_I\dot{\bfq} - \Nbar_I\right) - \bfAdT_I \dot{\bfA}_I \dot{\bfq}, \label{eq:dyn}\\
\lambda_I &:= \,\,\bfAd_I \left(\Upsilon_I - \Cbar_I\dot{\bfq} - \Nbar_I\right) - \Lambda_I \dot{\bfA}_I \dot{\bfq},\label{eq:ldyn}
\end{align}
where $\Upsilon_I$ is the applied forces, $\Cbar_I$ is the centripetal and Coriolis forces, and $\Nbar_I$
is the nonlinear and gravitational forces \cite[Eqn.~30, 31]{johnson_selfmanip_2013}.

When $\Mbar_\epsilon$,~\eqref{eq:mbare}, is invertible (including, possibly, even for $\epsilon = 0$), it 
is easy to verify the equivalences (and dropping for now the subscripted contact mode $I$),
\begin{align}
\Md &=\Mbar^{-1} - \Mbar^{-1} \bfA^T (\bfA \Mbar^{-1} \bfA^T)^{-1} \bfA \Mbar^{-1},\label{eq:Md}\\
\bfAdT &= \Mbar^{-1}\bfA^T(\bfA \Mbar^{-1} \bfA^T)^{-1},\label{eq:bfAdT}\\
\Lambda &= -(\bfA \Mbar^{-1} \bfA^T)^{-1},\label{eq:Lambda}
\end{align}
as shown in Appendix~\ref{app:la},~\eqref{eq:mdexp}. 
Note that constructions such as these are commonly used in robotics when $\Mbar$ is invertible, 
e.g.\ \cite[Eqns.~45--46]{Khatib_83} and many others (where their $\Lambda_r$ has the opposite sign of our $\Lambda$ and
their $\bar{J}$ corresponds to $\bfAdT$, although note that the definition~\eqref{eq:bfAdT} is exact and not defined
as a minimal-energy pseudo-inverse).

\begin{lemma}
\label{thm:dynamics}
When $\Mbar_0=\Mbar$ is invertible, the dynamics~\eqref{eq:dyn} and~\eqref{eq:ldyn} are 
equivalent to the more common expression (as stated e.g.\ in the last equations 
of \cite[Appendix~D]{johnson_selfmanip_2013}, or \cite[Eqn.~6.5,~6.6]{book:mls-1994}),
\begin{align}
\ddot{\bfq} &= \Mbar^{-1} \left( \Upsilon  - \Cbar\dot{\bfq} - \Nbar -\bfA^T \lambda \right), \label{eq:ddqmass}\\
\lambda &= (\bfA \Mbar^{-1} \bfA^T)^{-1}\left(\bfA\Mbar^{-1}\left(\Upsilon  - \Cbar\dot{\bfq} - \Nbar\right) + \dot{\bfA} \dot{\bfq} \right). \label{eq:lammass}
\end{align}
\end{lemma}
The claim follows directly from substituting~\eqref{eq:Md}--\eqref{eq:Lambda}, the
explicit solution to~\eqref{eq:astardef} when $\Mbar$ is invertible, into~\eqref{eq:dyn}--\eqref{eq:ldyn}, 
as worked out in Appendix~\ref{app:pfdyn}.

Whether $\Mbar$ is invertible or not, we require,
\begin{assumption}[Lagrangian Dynamics]
In each contact mode $I$, the time evolution of the 
active coordinates of the system are governed by Lagrangian dynamics, and
the applied forces are such that the vector field defined by~\eqref{eq:dyn} for coordinates in $\e{Q}_I$ 
and\textsuperscript{\ref{fn:uml}} $\td{F}_I$ for coordinates in $\td{\e{Q}}_I$ is forward complete, i.e.\ the maximal integral curve through 
any initial condition is defined for all positive time.
\label{ass:complete}
\end{assumption}

Recall from the rigid body and unconstrained massless assumptions (\ref{ass:rigid} \&~\ref{ass:uml}) that the configuration space, $\e{Q}$,
is a manifold without boundary.
Thus the major obstacle to verifying Assumption~\ref{ass:complete} lies in preventing 
finite-time ``escape'' from the state space $T\e{Q}$, e.g.\ because the velocity grows without bound or there are ``open edges'' in the configuration manifold (i.e., the manifold is not compact).
If the configuration manifold were compact, then it would suffice to impose a 
global bound on the magnitude of the vector field in~\eqref{eq:dyn}.
If the configuration space were instead Euclidean, $\e{Q} = \Real^n$, then it would suffice to impose a global Lipschitz continuity condition on the vector field in~\eqref{eq:dyn}.
We note that configuration obstacles such as joint limits or self-intersections are treated as constraints in Section~\ref{sec:hyb}, and hence pose no obstacle to satisfying the above boundarylessness and completeness conditions on the configuration space.

However, since in examples of interest the configuration space is neither compact nor a vector space \changed{(as noted after Assumption~\ref{ass:rigid})}, we often require a more general condition.
One such condition is obtained from \cite[Thm.~10]{Ballard2000}; since we rely on this sufficient condition elsewhere in the paper, we transcribe it explicitly into our notation as follows.
When $\e{Q}$ is a complete connected configuration manifold and $\Mbar$ is a nondegenerate inertia tensor (i.e., at every $\bfq\in\e{Q}$ the coordinate representation of $\Mbar(\bfq)$ is invertible, thus here precluding the possibility of massless limbs, Assumption~\ref{ass:cml}), we let $d_{\Mbar}:\e{Q}\times\e{Q}\into\Real$ denote the distance metric induced by the Riemannian metric $\angleM{\cdot}{\cdot}$ associated with $\Mbar$ \cite[Ch.~6]{Lee1997}.
For any vector $\dot{\bfq} \in T_\bfq\e{Q}$ we define $\absM{\dot{\bfq}} := \angleM{\dot{\bfq}}{\dot{\bfq}}^{1/2}$.
For any covector $\bff \in T^*_\bfq\e{Q}$ we define $\absMinv{\bff} := \absM{\bff^\#}$, where $\bff^\#\in T_\bfq\e{Q}$ is the vector obtained by ``raising an index'' (in coordinates, $\bff^\# = \Minv\bff^T$) \cite[Ch.~3]{Lee1997}.

\begin{lemma}\label{lem:lip}
If the ambient configuration space $\e{Q}$ is a complete connected Riemannian manifold, 
$\Mbar$ is a nondegenerate inertia tensor, and 
the magnitude of $\Upsilon_I - \Nbar_I$ grows at most linearly with velocity and distance from some (hence any) point in $\e{Q}$,
i.e.\ if there exists $C\in\Real$, $\bfq_0\in\e{Q}$ such that,
\begin{align}
\forall&(\bfq,\dot{\bfq})\in T\e{Q}:
\label{eq:effbddJ}
\absMinv{\Upsilon_I(\bfq,\dot{\bfq}) - \Nbar_I(\bfq,\dot{\bfq})} \le C(1 + \absM{\dot{\bfq}} + d_{\Mbar}(\bfq_0,\bfq)),
\end{align}
then the flow associated with the vector field~\eqref{eq:dyn} is forward complete, i.e.\ the maximal integral curve through any initial condition is defined for all positive time, and hence
Assumption~\ref{ass:complete} is satisfied.
\end{lemma}

\begin{proof}
This is simply an application of \cite[Thm.~10]{Ballard2000} in the absence of unilateral constraints.
\end{proof}

We expect this condition to be met by any model based on a physical system, and is trivially
met if there is a global bound on the magnitude of the applied,  $\Upsilon_I$, and potential, $\Nbar_I$,  forces
(whereas, notice, the necessarily unbounded Coriolis and centripetal forces are accounted for by the Lemma 
and require no further consideration).

Unfortunately this condition assumes that the inertia tensor $\Mbar$ is nondegenerate, precluding the presence of massless limbs (Assumption~\ref{ass:cml}).
Allowing instead for a degenerate inertia tensor but enforcing the unconstrained massless limb assumption (\ref{ass:uml}), we now describe a set of sufficient conditions that ensure Assumption~\ref{ass:complete} holds.

\begin{lemma}\label{lem:lip:ml}
Suppose that in each contact mode $I$ the active constraints are either \emph{holonomic} or \emph{integrable} \cite[Sec.~6.1.1]{book:mls-1994}, meaning that there exists a reduced configuration manifold $\e{Y}_I$ (i.e., generalized coordinates)
such that every point in $\e{Q}_I$ lies in the image of an embedding $h : \e{Y}_I\into\e{Q}_I$ that is invariant under~\eqref{eq:dyn} \cite[Sec.~G]{johnson_selfmanip_2013} and restricted to which the reduced inertia tensor \cite[Eqn.~36]{johnson_selfmanip_2013} is nondegenerate.

If the hypotheses in Lemma~\ref{lem:lip} are satisfied for $\e{Y}_I$, its reduced inertia tensor, and its reduced dynamics \cite[Eqn.~34]{johnson_selfmanip_2013},
and furthermore the vector field $\td{F}_I$ governing unconstrained massless limbs is forward complete and uncoupled 
from 
the massive or constrained coordinates, i.e.,
\eqnn{\label{eq:FI}
\Tq_I = Dh(\Ty_I), \
\dot{\td{\bfq}}_I = \td{F}_I(\td{\bfq}_I),
}
then Assumption~\ref{ass:complete} is satisfied.
\end{lemma}

\begin{proof}
We seek to define a 
forward-complete flow $\phi_I:[0,\infty)\times T\e{Q}_I\into T\e{Q}_I$ consistent with the vector field in~\eqref{eq:dyn}.
Let $h:\e{Y}_I\into\e{Q}_I$ denote the embedding associated with the reduced coordinates \cite[Sec.~G]{johnson_selfmanip_2013}.
Apply Lemma~\ref{lem:lip} to the reduced system to obtain a forward-complete flow $\td{\phi}_I:[0,\infty)\times T\e{Y}_I\into T\e{Y}_I$.
Then since \cite[Prop.~3]{Ballard2000} implies $Dh$ maps integral curves from the reduced state space to the original,
for all $t\in[0,\infty)$ and $\Ty\in T\e{Y}_I$, defining $\phi(t,Dh(\Ty)) = Dh\circ\td{\phi}(t,\Ty)$ yields the desired forward-complete flow on $\e{Q}_I$.
\end{proof}

Lemmas~\ref{lem:lip} \&~\ref{lem:lip:ml} provide sufficient conditions guaranteeing that certain systems
with either full rank inertia tensors or only holonomic constraints satisfy Assumption~\ref{ass:complete} -- in 
the most general case, however, this remains an assumption. 
We speculate that it is possible to derive a condition analogous to~\eqref{eq:effbddJ} using concepts 
from \emph{singular Riemannian geometry} \cite{Hermann1973} that ensure the existence of a forward-complete 
flow in the presence of nonintegrable constraints and a singular inertia tensor.

\subsection{Impulsive Dynamics}
\label{sec:impactmap}

Define the \emph{touchdown predicate}, $\TD:\KN \times T\calQ \rightarrow \TF$, where $\TF:=\{True,False\}$, as,
\begin{align}
\TD(k,\biq):= \bfa_k(\bfq) =0 \wedge \bfA_k(\bfq) \dot{\bfq} < 0, \label{eq:TDsimp}
\end{align}
so that $\TD(k,\biq)$ is true only at those points $\bfq$ where contact $k$ should be considered for
addition (in a manner to be qualified in Theorem~\ref{thm:ivc} by the 
impulse--velocity complementarity condition,~\eqref{eq:CPIV}, defined below).
Furthermore, define the \emph{new touchdown predicate},
\begin{align}
\NTD(\biq):= \bigvee_{k \in \KN } \TD(k,\biq), \label{eq:NTD}
\end{align}
such that $\NTD(\biq)$ is true only at those states where some new constraint is 
impacting. 

At impact into contact mode $J$, any incoming constraint velocity $\bfA_J\dot{\bfq}$ must be eliminated.
Here, we assume a Newtonian impact law, e.g.\ \cite[Eqn.~3]{chatterjee1998new} or \cite[Eqn.~11.65]{featherstone2008rigid}, that is,
\begin{assumption}[Plastic Impact]
\label{ass:plastic}
Impacts \changed{are} plastic (inelastic), occur instantaneously, and their effect described by an algebraic equation~\eqref{eq:dqp}, defined below. 
\end{assumption}

In general, $\Delta\dot{\bfq}_J:=\dot{\bfq}^+_J-\dot{\bfq}^-$, the instantaneous change in velocity 
from $\dot{\bfq}^-$ in contact mode $I$ before impact to $\dot{\bfq}^+_J$ in contact mode $J$ after impact,
is modeled as,
$\Delta\dot{\bfq}_J= -(1+e) \bfAdT_J\bfA_J \dot{\bfq}^-$
(recall that $\bfAdT:T\calC\rightarrow T\calQ$ maps velocities in the contact frames to velocities of the system state).
The coefficient of restitution, $e$, may be defined in any of the usual
ways, however throughout this paper plastic impact ($e=0$) is assumed.
We restrict to plastic impacts as we believe it to be more relevant to most robotics applications, and
since ambiguities arise when an elastic impact occurs in a system with multiple active constraints: different choices of impact model can yield distinct post-impact velocities (see Section~\ref{sec:lit:impact}).
For plastic impacts, the \emph{post-impact velocity} in mode $J$ is,
\begin{align}
\dot{\bfq}^+_J = (\Id - \bfAdT_J\bfA_J)\dot{\bfq}^- = \Md_J \Mbar \dot{\bfq}^-, \label{eq:dqp}
\end{align}
where the final simplification follows from~\eqref{eq:dagprop2} and matches \cite[Eqn.~9]{Westervelt_Grizzle_Koditschek_2003}.
The \emph{body impulse} in configuration coordinates is, 
\eqnn{\label{eq:P_J}
\qimp_J := -\Mbar (\dot{\bfq}^+_J - \dot{\bfq}^-). 
}
The \emph{contact impulse} (i.e., the impulse at the contact points that induces the desired change in velocity to agree
with the new contact mode $J$) is,
\begin{align}
\limp_{J} := \bfAd_J \qimp_J = \bfAd_J\Mbar\bfAdT_J\bfA_J \dot{\bfq}^- = - \Lambda_J \bfA_J \dot{\bfq}^- = \bfAd_J \Mbar \dot{\bfq}^-,  \label{eq:impulse}
\end{align}
where recall that $\bfA_J$, $\bfAd_J$, $\Mbar$, and $\Lambda_J$ are functions of the state $\bfq$ (which does
not change during impact, i.e.\ $\bfq^+=\bfq^-$), and the impulses, $\qimp_J$ and $\limp_{J}$, are 
also functions of the incoming velocity, $\dot{\bfq}^-$.
The final simplification arises from~\eqref{eq:dagprop2}, matches \cite[Eqn.~10]{Westervelt_Grizzle_Koditschek_2003}, and is used
in Section~\ref{sec:pseudoimpulse}.

\begin{lemma}
\label{thm:impulse}
When $\Mbar$ is invertible, contact impulse~\eqref{eq:impulse} into contact mode $J$ is
equivalent to the non-degenerate plastic impact law,
\begin{align}
&\limp_J = (\bfA_J \Mbar^{-1} \bfA_J^T)^{-1} \bfA_J \dot{\bfq}^-, \label{eq:plammass}
\end{align}
as listed e.g.\ in \cite[Eqn.~3]{chatterjee1998new}.
\end{lemma}
As with the proof of Lemma~\ref{thm:dynamics}, the result may be seen by substituting~\eqref{eq:bfAdT} or~\eqref{eq:Lambda}, the
explicit solution to~\eqref{eq:astardef} when $\Mbar$ is invertible, into~\eqref{eq:impulse}, as worked
out in Appendix~\ref{app:pfimp}.

\subsection{Complementarity}
\label{sec:complementarity}

We now introduce the classical complementarity problems for forces and impulses at the contact points, and provide a reformulation that allows massless limbs.
We begin with a general statement of the complementarity property \citep[as in e.g.][]{ingleton1966problem,cottle1968problem,lotstedt1982mechanical,van1998complementarity}, then subsequently specialize in Sections~\ref{sec:fac} and~\ref{sec:ivc} to formulations of force--acceleration and impulse--velocity complementarity.
Both versions have the general form of seeking real vectors $\bfy$ and $\bfz$ such that,
\begin{align}
\bfy \geq 0, \qquad \bfz \geq 0, \qquad \bfy^T \bfz = 0, \label{eq:compl}
\end{align}
(where for a vector $\bfy$, $\bfy \geq 0 \Leftrightarrow \bfy_i \geq 0 \forall\; i$) 
subject to some problem-specific constraints.
While the most general problem is \emph{uncoupled}, that is $\bfy$ and $\bfz$ may be
chosen arbitrarily so long as they satisfy \eqref{eq:compl}, the cases we consider here
are \emph{coupled} by these problem-specific constraints \cite[Sec.~3]{pang1996complementarity}.
In the linear complementarity problem (LCP), for instance, the coupling constraint has the form $\bfz := \bfA \bfy + \bfc$ 
\citep[e.g.][Eqn.~8]{brogliato2002numerical}.
The functional relationships between $\bfy$ and $\bfz$ for the complementarity problems in this paper 
is in general nonlinear (as discussed in the rest of this section). 
Since the relation of interest 
is generally problem-specific and index dependent in an essential way, we introduce temporarily 
an abstract scalar relation, $\unrhd$ instead of $\succeq$ or $\geq$ and similarly $\rhd$ 
instead of $\succ$ or $>$, whose different instantiations \changed{are prescribed}
in the force--acceleration and impulse--velocity versions of the problem, respectively. 
\changed{For simplicity of notation we use $=$ as the corresponding equality relation.}

Solutions to this problem produce a natural bipartition $(J,J^C)$ on some 
index set, $\calI$, the \emph{scope} (some subset of the universal scope $\calK$, to be discussed below),
where $J= \{j \in \calI: z_j = 0\}$ and $J^C = \calI \backslash J = \{j\in \calI:z_j \rhd 0\}$. 
Here, the role of $\bfy$ and $\bfz$ \changed{are played} by physically determined functions 
of a specified (``incoming'') state,  $\biq^-=(\bfq^-,\dot \bfq^-)$, to yield an ``outgoing'' 
bipartition $(J, J^C)$ of the indexing scope, $\calI$.
The indexing scope \changed{is a} function only of the incoming continuous state, $\calI:T\calQ \rightarrow 2^\calK$,
as defined in~\eqref{eq:IScope}.

It should now be clear that for this paper  
the complementarity problem is reduced to finding the unknown bipartition $(J, J^C)$, 
also known as the \emph{mode selection problem} \cite[]{van1998complementarity},
as opposed to finding the values of the two complementarity vectors directly, e.g.\ \cite[]{cottle1968problem}.
Namely, given an index set $\calI$, two functions $Y,Z: 2^\calI \times T\calQ \rightarrow \Real^{|\calI|}$ that map a subset $J \in 2^\calI$ 
into a Euclidian space with dimension equal to the size of the index set,
and a generic relation $\rhd$ (to be instantiated as $\succ$ or $>$ in the following sections\footnote{\changed{Recall that the
relation may, in the case of $\succ$, depend on the vector field at that point.}}), we require a solution to a set of constraint equations of the form,
\begin{align}
&Y_j(J, \biq^-) \unrhd 0, \qquad Z_j(J, \biq^-) = 0, \qquad \forall \; j \in \calI \cap J, \label{eq:YjJ}\\
&Y_k(J, \biq^-) = 0, \qquad Z_k(J, \biq^-) \rhd 0, \qquad \forall \; k \in \calI \backslash J, \label{eq:YkJ}
\end{align}
(where by definition, $\calI \cap J = J$).
For the complementarity problems of interest in this paper, the 
equality constraints in~\eqref{eq:YjJ}--\eqref{eq:YkJ} hold for all arguments 
$(J,\biq^-) \in 2^\calI \times T\calQ$ by construction (enforced, e.g., by the flow~\eqref{eq:dyn} 
in the force--acceleration version, and by the impact map~\eqref{eq:impulse} in the impulse--velocity version). 

The complementarity problem as stated thus far is not explicitly coupled \cite[Sec.~3]{pang1996complementarity}, i.e.\ it
places no requirements on the relationship between $Y_k$ and $Z_k$ other than their common dependence on $J$ and $\biq^-$,
which is why existence and uniqueness properties are challenging to define in general. 
Furthermore, this necessitates the evaluation of both $Y_k$ and $Z_k$ for constraints $k$ that are not in $J$. 
With the possibly massless limbs in our setting, the evaluation of $Z_k$ \changed{is not always possible} as
the concept of a separation velocity or acceleration is poorly defined (once
such a contact point has lifted off the ground the corresponding joints must be dropped from the state according
to the unconstrained massless limb assumption, \ref{ass:uml}).
Thus the specifics of $Z_k$ in the problems considered in this paper necessitate an alternate formulation
that takes advantage of the coupling between $Y_k$ and $Z_k$, as
the inequality constraints have the property that,
\begin{align}
Z_k(J,\biq^-) \rhd 0 \Leftrightarrow Y_k(J \cup \{k\},\biq^-) \nrhd 0, \label{eq:ZkJ}
\end{align}
(the importance of the $(J \cup \{k\})$ mode was first noted in \cite[Eqn.~1.7.3]{ingleton1966problem}).
This suggests the combined expression,
\begin{align}
(k \in J) \Leftrightarrow (Y_k(J\cup \{k\},\biq^-) \unrhd 0), \qquad \forall \; k \in \calI, \label{eq:YkJa}
\end{align}
which is equivalent to~\eqref{eq:YjJ} \&~\eqref{eq:YkJ},
\begin{lemma}
\label{thm:comp}
The separate relational statements of the complementarity problem,~\eqref{eq:YjJ}--\eqref{eq:YkJ}, are equivalent to a single 
biconditional statement,~\eqref{eq:YkJa}, provided that the complementary functions $Y$ and $Z$ satisfy~\eqref{eq:ZkJ}.
\end{lemma}
\begin{proof}
First note that for $k \in J$ it is trivially true that $J \cup \{k\} = J$ and 
so~\eqref{eq:YkJa} simplifies to the first condition of~\eqref{eq:YjJ}.
For $k \notin J$, the expression in~\eqref{eq:YkJa} along with the substitution of~\eqref{eq:ZkJ}
reduces to the second condition of~\eqref{eq:YkJ}.
\end{proof}

Expressing \eqref{eq:YkJa} as a predicate $\PRED:2^\calK \times T\calQ \rightarrow \TF$, 
\begin{align}
&\PRED(J,\biq^-) = 
\left(J \subseteq \calI \right) \bigwedge_{k \in \calI} (k \in J) \Leftrightarrow (Y_k(J\cup\{k\},\biq^-) \unrhd 0). \nonumber
\end{align}
We denote by,
\begin{align}
\CP_{\PRED}:&  T\calQ  \rightarrow 2^\calK : \biq^- \mapsto J, \label{eq:CP}
\end{align}
the implicit function that solves this set of constraints for the unknown required bipartition,
where $\PRED$ varies with the particular instances as determined by the 
appropriate form of $Y_k$. Note that while the codomain is $2^\calK$, the solution is always
a member of $2^\calI$.

The existence of this implicit function~\eqref{eq:CP} (i.e., the existence and
uniqueness of a solution, $J$, to the mode selection problem) is in the most general cases
an additional assumption\footnote{However note that the remainder of this paper only requires a unique \emph{choice}
of a solution that satisfies the predicate if multiple solutions exist.}
 (see Assumption~\ref{ass:fac} and~\ref{ass:ivc}, below), 
although the specific complementarity problems in this section (i.e., based on the relationship of
the specific functions $Y$ and $Z$ used in these cases) in the absence of friction reduce down
to the conventional LCP problem and so existence and
uniqueness has been proven in e.g.\ \cite{ingleton1966problem,cottle1968problem,lotstedt1982mechanical,van1998complementarity}. 

The motivating literature and related work discussed in Section~\ref{sec:lit}
generally imposes two complementarity conditions on rigid body dynamics models. 
The \emph{force--acceleration} (FA) variant of~\eqref{eq:YjJ}--\eqref{eq:YkJ}, presented in~\eqref{eq:Ujl}--\eqref{eq:Ukl}, 
stipulates that there cannot be both a continuous time contact force and a separation
acceleration at the same contact point, and is widely considered to 
arise from fundamental physical reasoning.
In the present setting, FA complementarity 
governs exclusively the nature of \emph{liftoff} events (and extended in Section~\ref{sec:friction} to stick/slip events) wherein the number of active 
contacts (i.e., cardinality of the mode set) is reduced for reasons discussed in Section~\ref{sec:fac}.
In contrast, during instantaneous impact events the contact forces have no time to perform work. Instead,
the \emph{impulse--velocity} (IV) variant of~\eqref{eq:YjJ}--\eqref{eq:YkJ}, 
presented in~\eqref{eq:UjP}--\eqref{eq:UkP}, precludes a simultaneous impact-induced 
contact impulse and separation velocity at the same contact point. This constraint is 
known not to follow inevitably from the rational mechanics of rigid body models \cite[]{chatterjee1999realism}, 
but represents a commonly exploited algorithmic simplification that we embrace 
in this inelastic model at the possible expense of consistency with elastic impact models in the limit.
In the present setting, IV complementarity governs exclusively the nature of \emph{touchdown} 
events wherein one or more new contacts become active (i.e., cardinality of the 
mode set is increased) for reasons discussed in Section~\ref{sec:ivc}.

\subsubsection{Force--Acceleration (FA) Complementarity}
\label{sec:fac}

For continuous time contact forces, when $\NTD(\biq)$ is false and therefore $\qimp=0$ when one or more contact 
constraints violate the unilateral constraint cone\footnote{ Recall from Section~\ref{sec:manip}
that $\bfU$ in the normal direction is $-1$ according to the frame conventions of 
\cite[Eqn.~76,~78]{johnson_selfmanip_2013}.}~$\bfU$, some constraints lift 
off and must be removed from the set of active constraints, resulting in a transition to a new mode.
Determining that next mode sets up a complementarity 
problem over the existing contact mode $I$ between the unilateral constraint cone, $\bfU_k(\lambda)$, if the contact is kept,
and the separation acceleration $\frac{d}{dt} \bfA_k \dot{\bfq}=\bfA_k \ddot{\bfq}+\dot{\bfA}_k \dot{\bfq}$ 
if it is removed 
(recall that as an active constraint the state velocity is initially $\bfA_k\dot{\bfq}=0$).
The full scope of contact constraints that should be considered is the set of all contacts which are ``touching'', 
i.e.\ those whose normal direction component have zero contact distance and a non-separating velocity\footnote{ Note that thus
far only normal direction constraints have been considered, however Section~\ref{sec:friction} extends this to include
tangential (sliding friction) constraints and this scope is defined in this general way in order to apply there as well.},
\begin{align}
\calI(\biq):=&\{i \in \calK: \bfa_{\alpha(i)}(\bfq)=0 \wedge \bfA_{\alpha(i)}(\bfq)\dot{\bfq}\leq 0 \} \label{eq:IScope}\\
 =&\{i \in \calK: \left(\bfa_{\alpha(i)}(\bfq)=0 \wedge \bfA_{\alpha(i)}(\bfq)\dot{\bfq} = 0\right)
\vee \TD(\alpha(i),\biq) \}. 
\end{align}
Recall that force--acceleration complementarity only holds when $\NTD(\biq)$ is false and so the final
condition applies here. Furthermore, while the full scope is formally required and does not depend on the active mode, numerically
it suffices to check $I\subseteq \calI$ -- this reduced scope eliminates the numerical
challenge of checking the exact equality of~\eqref{eq:IScope}. 
\changed{See Section~\ref{sec:dis:numerical} for a further discussion of this simplification
and other numerical implementation details.}

For transition into $J$, consider contact force~\eqref{eq:ldyn} both in $J$ but also in
the alternative mode $J\cup\{k\}$ where contact $k$ is maintained (the reason for this alternative
mode is clear in Theorem~\ref{thm:fac}),

\begin{align}
&\bfU_j(\lambda_J) \succeq 0, \bfA_j \ddot{\bfq}+\dot{\bfA}_j \dot{\bfq}\equiv 0, \quad \quad \forall\;j \in \calI \cap J,\label{eq:Ujl}\\
&\bfU_k(\lambda_J)\equiv0, \bfU_k(\lambda_{J\cup\{k\}}) \prec 0, \quad\quad \forall\;k \in \calI \backslash J,\label{eq:Ukl}
\end{align}
where the identically zero constraints
are guaranteed to hold in consequence of the dynamics governing mode $J$, namely, 
the invariance of the flow ($\bfA_j \ddot{\bfq}+\dot{\bfA}_j \dot{\bfq}\equiv 0 \;\forall\;j \in J$ by~\eqref{eq:dyn}) and the
Lagrange multipliers ($\bfU_k(\lambda_J)\equiv0 \;\forall\; k \notin J$ by~\eqref{eq:ldyn}).
Note the importance here of the trending positive/negative conditions (Definition~\ref{def:through}) --
in general it is not sufficient to simply check the sign of the contact force but possibly higher 
derivatives as well. For example, in Figure~\ref{fig:ptex}, cases (c) and (d), assume the particle is sliding
along the constraint from left to right. At the moment the particle reaches the origin, the contact force is zero. 
However in (c) the contact force is trending negative and the constraint should be removed, while in (d) the
contact force is trending positive and it should be maintained.

Constraints~\eqref{eq:Ujl}, \eqref{eq:Ukl} can be simplified into a form analogous to~\eqref{eq:ZkJ}, hence, 
by Lemma~\ref{thm:comp},
\begin{align}
(k \in J) \Leftrightarrow (\bfU_k(\lambda_{J\cup\{k\}})\succeq 0), \qquad \forall \; k \in \calI, \label{eq:Ukla}
\end{align}
or as the predicate $\FA:2^\calK \times T\calQ \rightarrow \TF$, 
\begin{align}
\FA(J,\biq^-) = \left(J\subseteq\calI\right) \bigwedge_{k \in \calI} (k \in J) \Leftrightarrow (\bfU_k(\lambda_{J\cup\{k\}})\succeq 0), \label{eq:CPFA}
\end{align}
for which we write the associated implicit function solution,
following~\eqref{eq:CP}, as, $\CP_{\FA}(\biq^-)=J$. 

This formulation of the force--acceleration complementarity problem is required to allow for
massless limbs, for which a separation acceleration is not well defined. However when the separating acceleration is
defined,
\begin{theorem}
\label{thm:fac}
The non-penetrating acceleration condition at a contact $k$ after liftoff into mode $J$, 
$\bfA_k \ddot{\bfq}+\dot{\bfA}_k \dot{\bfq}\succ 0$ 
(when such an acceleration is well defined), 
is equivalent to a trending negative contact force
$\bfU_{k}(\lambda_K)\prec0$ in mode $K:=J \cup \{k\}$. In other words,~\eqref{eq:Ujl}--\eqref{eq:Ukl} are 
equivalent to the usual formulation \citep[e.g.][Eqn.~10]{brogliato2002numerical},
\begin{align}
&\bfU_j(\lambda_J) \succeq 0, \bfA_j \ddot{\bfq}+\dot{\bfA}_j \dot{\bfq} \equiv 0, \quad \quad \forall\;j \in \calI \cap J,\label{eq:Ujl2}\\
&\bfU_k(\lambda_J)\equiv0, \bfA_k \ddot{\bfq}+\dot{\bfA}_k \dot{\bfq} \succ 0, \quad\quad  \forall\;k \in \calI \backslash J.\label{eq:Ukl2}
\end{align}
\end{theorem}
\begin{proof}
Recall from Section~\ref{sec:not} that $\pi_{ I}$ is the canonical projection 
onto the linear subspace spanned by the coordinate axes indexed by $I$, and 
assume without loss of generality that $k$ is the final index in $K$ such that,
\begin{align}
\bfA_K &:= \pi_K \bfA_\calK = \left[ \begin{array}{c} \bfA_J\\ \bfA_k\end{array}\right] = \left[ \begin{array}{c} \pi_J \bfA_\calK\\ \pi_k\bfA_\calK\end{array}\right]. \label{eq:bfAK}
\end{align}
The constraint cone, following \cite[Eqn.~78]{johnson_selfmanip_2013}, applied to the contact forces,~\eqref{eq:ldyn}, is (see 
Appendix~\ref{app:la} for details, in particular~\eqref{eq:finalbreakout} expanding $\bfAd_K$ and $\Lambda_K$),
\begin{align}
\bfU_{k}(\lambda_K) &=-\pi_k\left( \bfAd_K \left(\Upsilon  - \Cbar\dot{\bfq} - \Nbar\right) - \Lambda_K \dot{\bfA}_K \dot{\bfq}\right)\\
&= -\frac{\bfA_k \Md_J}{\bfA_k \Md_J \bfA_k} \left(\Upsilon  - \Cbar\dot{\bfq} - \Nbar\right)
- 
\frac{1}{\bfA_k\Md_J \bfA_k^T}\left[ \begin{array}{cc}   \bfA_k \bfAdT_J &  -1 \end{array}\right]
\left[ \begin{array}{c} \dot{\bfA}_J\\ \dot{\bfA}_k\end{array}\right]\dot{\bfq}\label{eq:fasub}\\
&= -\frac{\bfA_k \Md_J \left(\Upsilon  - \Cbar\dot{\bfq} - \Nbar\right)- \bfA_k \bfAdT_J\dot{\bfA}_J\dot{\bfq} + \dot{\bfA}_k\dot{\bfq}}{\bfA_k\Md_J \bfA_k^T}, \label{eq:uklden}
\end{align}
while the separating acceleration for constraint $k$ in mode $J$ is, using~\eqref{eq:dyn},
\begin{align}
\bfA_k\ddot{\bfq}_J +\dot{\bfA}_k \dot{\bfq}_J&= \bfA_k \left(\Md_J \left(\Upsilon  - \Cbar\dot{\bfq} - \Nbar\right) - \bfAdT_J \dot{\bfA}_J \dot{\bfq}\right) +\dot{\bfA}_k \dot{\bfq} 
= - (\bfA_k\Md_J \bfA_k^T) \bfU_{k}(\lambda_K).  \label{eq:ddaul}
\end{align}
Since the denominator in~\eqref{eq:uklden} is a positive scalar function of 
state (as shown in Appendix~\ref{app:la},~\eqref{eq:SAdef}), by Definition~\ref{def:through} 
and Lemma~\ref{lem:twofun}
a trending positive separation acceleration, $\bfA_{k} \ddot{\bfq} + \dot{\bfA}_{k} \dot{\bfq} $ implies a 
trending negative contact force, $\bfU_{k}(\lambda_K)$, and vice-versa, and 
therefore~\eqref{eq:Ujl}--\eqref{eq:Ukl} are 
equivalent conditions to~\eqref{eq:Ujl2}--\eqref{eq:Ukl2}. 
\end{proof}

Furthermore, with or without a full rank inertia tensor, we assume the existence 
of a unique solution to the force--acceleration complementarity problem,
\begin{assumption}[Force--Acceleration Complementarity]
\label{ass:fac}
The force--acceleration complementarity constraints, $\FA$, \eqref{eq:CPFA}, always 
admit an implicit function, $\CP_{\FA}$, over the entire state space $T\calQ$, 
even under the frictional properties that follow Assumption~\ref{ass:friction}.
That solution correctly captures the behavior of the physical system.
\end{assumption}
While there has been a long line of literature (e.g., \cite[Ex.~3.3]{van1998complementarity}\footnote{ In
the language of that paper, this is a \emph{Dynamic Complementarity Problem} (DCP), and note 
that \cite[Eqn.~33]{van1998complementarity} asserts complementarity with the base constraint (which they call $y = C(q)$), but
here the position and velocity are already zero, i.e.\ $\forall\; j \in I, \bfa_j(\bfq) = \bfA_j(\bfq)\dot{\bfq} =0$, and so the acceleration
is the first degree that must be checked.})
that proves that this is always true for independent, plastic, frictionless contacts, no result
has been found that covers the limited frictional conditions introduced in~\ref{sec:friction}.
We impose this condition in support of the Theorems in Section~\ref{sec:hyb}.

\subsubsection{Impulse--Velocity (IV) Complementarity}
\label{sec:ivc}

Impact at one contact location can cause another contact to break, as the contact impulse
must obey the unilateral constraint cone,~$\bfU_j(\limp_{J})\geq 0\;\forall\;j\in J$, i.e.\ both that the impulse in the normal direction
be positive (non-adhesive) and that the tangential impulse lie in the friction cone \cite[]{chatterjee1998new}.
Any contact point that would have violated that requirement must be dropped from the active contact mode.

In addition the post-impact velocity must not allow the removed
contact point to leave with a penetrating velocity (i.e., the impulse cannot result in a velocity ``into'' the surface).
However in the case of massless legs a positive separation velocity is always achievable. 
As an alternative requirement that is based only on impulses\footnote{ Note that this formulation
based only on impulses also simplifies the inclusion of the pseudo-impulse condition~\eqref{eq:PIV}.}, 
consider the contact impulse~\eqref{eq:impulse}, $\limp_{J}$ (associated with the passage from contact $I$ to contact $J$),
but also the contact impulse $\limp_{J\cup\{k\}}$ (associated with the 
passage from contact $I$ to alternative mode ${J\cup\{k\}}$ where contact $k$ is maintained).
These impulses, along with the post-impact velocity~\eqref{eq:dqp}, $\bfA_J \dot{\bfq}^+$, must satisfy,
\begin{align}
&\bfU_j(\limp_{J}) \geq 0, \bfA_{j} \dot{\bfq}^+ = 0, \qquad \qquad\; \forall\;j \in \calI \cap J ,\label{eq:UjP}\\
&\bfU_k(\limp_{J})=0, \bfU_k(\limp_{J\cup\{k\}}) < 0, \qquad \forall\;k \in \calI \backslash J ,\label{eq:UkP}
\end{align}
where the scope, $\calI$, is again formally the set of all touching 
constraints,~\eqref{eq:IScope}. 
However in numerical simulation it is sufficient to check only the active constraints as well as
those touching down\footnote{
Note that this excludes those constraints which are touching but separating, but whose post-impact
velocity is penetrating. For such cases the application
of this first impact puts the state into the guard for those constraints, and thus
even in this pathological case the execution continues from the correct mode. },~\eqref{eq:TDsimp},
\begin{align}
\calI_{\IV}:=I \cup \{i \in \calK\backslash I: \TD(\alpha(i),\biq)\} \subseteq \calI. \label{eq:IVScope2}
\end{align}
Note that the equality constraints are enforced
by the impact law~\eqref{eq:impulse}, and so,
by Lemma~\ref{thm:comp}, \changed{constraints~\eqref{eq:UjP}--\eqref{eq:UkP} reduce to,}
\begin{align}
(k \in J) \Leftrightarrow (\bfU_k(\limp_{J\cup\{k\}}) \geq 0), \qquad \forall \; k \in \calI. \label{eq:UkPa}
\end{align}
or as the predicate $\IV:2^\calK \times T\calQ \rightarrow \TF$, 
\begin{align}
\IV(J,\biq^-) = \left(J \subseteq \calI\right)\bigwedge_{k \in \calI} (k \in J) \Leftrightarrow (\bfU_k(\limp_{J\cup\{k\}})\geq 0), \label{eq:CPIV}
\end{align}
whose solution, following~\eqref{eq:CP}, is denoted, $\CP_{\IV}(\biq^-)=J$.

As with the force--acceleration complementarity problem, this formulation of the 
impulse--velocity complementarity problem is required to allow for
massless limbs, for which a separating velocity is not well defined. However when the separating velocity is
defined,
\begin{theorem}
\label{thm:ivc}
The non-penetrating velocity condition at a contact $k$ after impact into mode $J$, 
$\bfA_k\dot{\bfq}^+>0$ (where such a velocity is well defined), is equivalent to 
a negative contact impulse, $\bfU_{k}(\limp_{K})<0$,
at impact into mode $K:=J \cup \{k\}$. In other words,~\eqref{eq:UjP}--\eqref{eq:UkP} are equivalent to the
usual formulation \citep[e.g.][Eqn.~9]{brogliato2002numerical},
\begin{align}
&\bfU_j(\limp_{J}) \geq 0, \bfA_{j} \dot{\bfq}^+ = 0, \qquad \forall\;j\in \calI \cap J  ,\label{eq:UjP2}\\
&\bfU_k(\limp_{J}) = 0, \bfA_{k} \dot{\bfq}^+ > 0, \qquad \forall\;k\in \calI\backslash J .\label{eq:UkP2}
\end{align}
\end{theorem}
\begin{proof}
Recall from Section~\ref{sec:not} that $\pi_{ I}$ is the canonical projection 
onto the linear subspace spanned by the coordinate axes indexed by $I$, and 
assume without loss of generality that $k$ is the final index in $K$ such that $\bfA_K$ is defined as in~\eqref{eq:bfAK}.
The constraint cone, following \cite[Eqn.~78]{johnson_selfmanip_2013}, applied to the contact impulse,~\eqref{eq:impulse}, is (see Appendix~\ref{app:la} for details, in particular~\eqref{eq:finalbreakout} expanding $\Lambda_K$),
\begin{align}
\bfU_{k}(\limp_{K}) &=\pi_k \Lambda_K \bfA_K \dot{\bfq}^- 
= \frac{1}{\bfA_k\Md_J \bfA_k^T}\left[ \begin{array}{cc}   \bfA_k \bfAdT_J &  1 \end{array}\right]\left[ \begin{array}{c} \bfA_J\\ \bfA_k\end{array}\right]\dot{\bfq}^-
 =  \frac{\bfA_k \bfAdT_J \bfA_J\dot{\bfq}^- - \bfA_k\dot{\bfq}^-}{\bfA_k\Md_J \bfA_k^T},  \label{eq:uklimp}
\end{align}
while the separating velocity for constraint $k$ after impact into mode $J$ is, using~\eqref{eq:dqp},
\begin{align}
\bfA_{k} \dot{\bfq}^+_J &= \bfA_{k}\dot{\bfq}^- - \bfA_{k}\bfAdT_J\bfA_J\dot{\bfq}^- 
= - \left(\bfA_k\Md_J \bfA_k^T\right) \bfU_{k}(\limp_{K}).
\end{align}
Since the denominator in~\eqref{eq:uklimp} is a positive scalar function of 
state (as shown in Appendix~\ref{app:la},~\eqref{eq:SAdef}), 
a positive separation velocity, $\bfA_{k} \dot{\bfq}^+ $ implies a 
negative contact impulse, $\bfU_{k}(\limp_K)$, and vice-versa, and 
therefore~\eqref{eq:UjP}--\eqref{eq:UkP} are 
equivalent conditions to~\eqref{eq:UjP2}--\eqref{eq:UkP2}. 
\end{proof}
Furthermore, with or without a full rank inertia tensor, we assume the existence of a unique solution to the 
impulse--velocity complementarity problem,
\begin{assumption}[Impulse--Velocity Complementarity]
\label{ass:ivc}
The impulse--velocity complementarity constraint, $\IV$,~\eqref{eq:CPIV} always 
admit an implicit function, $\CP_{\IV}$, over the entire state space $T\calQ$, 
as does the modified problem including the pseudo-impulse, $\PIV$,~\eqref{eq:PIV}, introduced in Section~\ref{sec:pseudoimpulse}, 
and under the frictional properties that follow Assumption~\ref{ass:friction}.
That solution correctly captures the behavior of the physical system.
\end{assumption}
As with the FA complementarity problem, there has been a long line of literature \citep[e.g.][Eqn. 2.10b]{lotstedt1982mechanical}
that proves that this is always true for independent, plastic, frictionless contacts, however no result
has been found that covers the pseudo-impulse introduced in Section~\ref{sec:pseudoimpulse}
or the limited frictional conditions introduced in Section~\ref{sec:friction}.
We impose this condition in support of the Theorems in Section~\ref{sec:hyb}.

\subsection{Pseudo-Impulse}
\label{sec:pseudoimpulse}

Impulses arising from impacts (both plastic and elastic) generally break existing contacts.
For example an impulse imparted to the underside of a rigid block 
\changed{that is being pushed down onto a level surface}
must cause it to leave the \changed{surface} for a small time
interval no matter how weak the impulse or how 
\changed{strong the applied force}.
In truth the block is not rigid and the impulse is temporally distributed; modeling the resulting 
subtle deflections would greatly complicate the system. 
It appears expedient to impose a minimum threshold on impulse magnitude
below which the system may be considered \emph{quasi-static} and the block remains on the ground, 
but above which the system is \emph{dynamic} and the block detaches from the substrate.
This threshold could be specified directly in terms of a pre-defined limit on the system velocity or impulse magnitudes,
\changed{however such limits would not take into account the magnitude of the applied force (i.e.,
how strongly gravity or commanded torques are holding down the rigid block).}
Here we want the cutoff to scale with respect to the other
problem variables \changed{(such as incoming velocity, applied forces, and the contact configuration), 
and we show in Section~\ref{sec:pseudovel} that this induces an implicit (variable) bound on velocity}. 

In this section we define an additional impulse during impact which
qualitatively improves results and eliminates some Zeno phenomena. 
\changed{The time scaling parameter of this} impulse may be thought of as a tuning parameter
and while we give some physical motivation for its magnitude, the inclusion of this term
is motivated primarily by improving the qualitative behavior of the numerical simulation
\changed{(e.g.\ by excluding chattering and Zeno phenomena)} in a manner that retains physical 
fidelity across a broad range of application settings \changed{and preserves} mathematical consistency. 
See Section~\ref{sec:dis:pseudo} for examples of physical situations whose physical fidelity 
and mathematical properties both appear to be enhanced by the introduction of such a pseudo-impulse.

Therefore we make the following new assumption
about the physics of the system,
\begin{assumption}[Pseudo-Impulse]
\label{ass:pseudo}
The continuous time forces apply some small amount of work during the impact process.
\end{assumption}
In the usual Newtonian impact model, these forces have no effect \citep[as shown, e.g., in][Prop.~4.3]{tenDam1997},
however here we add an additional pesudo-impulse to support this assumption.

This assumption is not used in this paper to directly change the energy at any state, but 
rather is used within the discrete switching logic to improve the quality of make-break contact 
decisions (which implicitly changes the resulting kinetic energy of the system). 
Specifically, consider the \emph{pseudo-impulse},~$\dimp\in T^*\calC$, that the contacts would 
impart on the system to resist the continuous time forces for some small time duration,~$\delta_t\in\bbR^+$, during 
impact into mode $J$, 
\begin{align}
\Mbar \delta_{\dot{\bfq}} :=&  \lim_{\delta_t\rightarrow 0} \int_{\delta_t} \Mbar \ddot{\bfq} dt \approx (\Upsilon - \Cbar\dot{\bfq}^- -\Nbar)\delta_t, \label{eq:pseudoimp_}\\
\dimp :=&\bfAd_J\Mbar \delta_{\dot{\bfq}} = \bfAd_J (\Upsilon-\Cbar\dot{\bfq}^- -\Nbar)\delta_t. \label{eq:pseudoimp}
\end{align}
This small time~$\delta_t$ can be regarded as the finite duration of the (actually 
non-instantaneous) impact process \cite[]{quinn2005finite}.
An alternative interpretation
is that $\delta_t$ specifies the time duration after an impulsive separation of a contact during
which if that contact were to return to the active set, the qualitative behavior of the system would be improved
by never considering it as having left.
This interpretation, correct to first order for single contact systems,
is useful when considering what value of $\delta_t$ should be used.
In the simulations shown in this paper a magnitude
of $\delta_t=0.03s$ has been found to give the best results, although this value is surely dependent
on the material properties, the system velocities, and the desired qualitative and quantitative behavior of the model.

This pesudo-impulse is not directly applied to the system \citep[as in][]{quinn2005finite}, 
because in this model impacts occur instantaneously
and the velocity displacement $\delta_{\dot{\bfq}}$ would not be uniquely determined by~\eqref{eq:pseudoimp} when $\Mbar$ is singular.
Instead the pseudo-impulse is added to the plastic impact impulse,~\eqref{eq:impulse}, and in that way can be considered as a modified incoming
momentum,
\begin{align}
\limp_{K}+\dimp_{K} &= - \Lambda_K \bfA_K \dot{\bfq} + \bfAd_K \Mbar \delta_{\dot{\bfq}} = \bfAd_K \left( \Mbar\dot{\bfq} + \Mbar\delta_{\dot{\bfq}}\right),
\end{align}
This modified momentum is used as an extra guard condition during impact, $\bfU(\limp + \dimp)\geq0$,
in addition to the usual condition, $\bfU(\limp)\geq0$,
since the pseudo-impulse should not break contacts that would otherwise persist. 
That is, when $\delta_t>0$ the $\IV$ complementarity problem~\eqref{eq:CPIV} is replaced
by the predicate,
\begin{align}
\PIV(J, \biq)& := (\CP_\IV(\biq) \subseteq J \subseteq \calI) \label{eq:PIV} 
\bigwedge_{k \in \calI} (k \in J) \Leftrightarrow (\bfU_k(\limp_{K}) \geq 0 \vee \bfU_k(\limp_{K}+\dimp_{K}) \geq 0  ) 
\end{align}
whose solution, following~\eqref{eq:CP}, is denoted, $\CP_{\PIV}(\biq^-)=J$.
Note that by construction,
\eqnn{\label{eq:PIVIV}
\forall\; \Tq\in T\e{Q} : \CP_\IV(\Tq)\subseteq\CP_\PIV(\Tq).
}
The formulation of the complementarity condition based only on impulses in~\eqref{eq:UjP}--\eqref{eq:UkP}
is key to admitting this modification, as using instead~\eqref{eq:UjP2}--\eqref{eq:UkP2} would
require considering both the impulsive and velocity implications of this pseudo-impulse.

The existence and uniqueness of a solution to~\eqref{eq:PIV} is, for the purposes of this paper, simply
assumed under Assumption~\ref{ass:ivc}. 
In numerical studies, we did not encounter simulations that violated this property.
However, we have not derived sufficient conditions ensuring the property holds for this modified complementarity predicate.

\subsubsection{Velocity Implications of the Pseudo-Impulse}
\label{sec:pseudovel}

The $\IV$ predicate is provided as a purely logical proposition.
However, its truth value varies in physically-interpretable ways with respect to variations in the base point at which it is evaluated.
In the following, Lemma~\ref{lem:pseudoimp} shows that $\CP_\IV$ returns the same answer regardless of the 
impact speed along any particular impacting velocity direction, and Theorem~\ref{thm:pseudoimp} shows 
that $\CP_\PIV$ does not. See Section~\ref{sec:dis:pseudo} for a discussion of physical implications of these
results in simple example settings.

\begin{lemma}
\label{lem:pseudoimp}
Let $I\subseteq\calK$, $(\bfq,\dot{\bfq})\in T\calQ$ be such that there exists a unique $k\in\e{K}\sm I$ such that $\bfa_k(\bfq) = 0$ and $\bfA_k(\bfq)\dot{\bfq} < 0$, i.e.\ the system instantaneously undergoes impact with exactly one constraint $k\in\e{K}$.
Define $K := I \cup \set{k}$ and,
\eqnn{\label{eq:dqs}
\forall\; s \ge 0 : \dot{\bfq}_s := \paren{\Id - (1 - s)\bfAdT_K \bfA_K}\dot{\bfq}.
}
If $\delta_t=0$, then $J := \CP_{\IV}(\bfq,\dot{\bfq})\subseteq \calK$ satisfies,
\eqnn{
\forall\; s > 0 : J = \CP_{\IV}(\bfq,\dot{\bfq}_s),
}
that is the solution to the complementarity problem is the same at any impacting speed.
\end{lemma}

\begin{proof}
The impulse--velocity complementarity predicate~\eqref{eq:CPIV} is a conjunction of propositions involving conic inequalities;
since furthermore the contact impulse~\eqref{eq:impulse} is simply scaled (this first identity
can be seen using the expansions given in Appendix~\ref{app:la}, \eqref{eq:finalbreakout}, and see also~\eqref{eq:dagprop}),
\begin{align}
\bfA_J \bfAdT_K \bfA_K &= \bfA_J \left[ \bfAdT_J + \Md_J ... \qquad -\Md_J...\right]\bfA_K 
= (\Id + 0) \bfA_J + 0 = \bfA_J,\\
\limp_J(\bfq,\dot{\bfq}_s) &= - \Lambda_J \bfA_J (\Id - (1-s)\bfAdT_K \bfA_K) \dot{\bfq} 
= s\,(-\Lambda_J\bfA_J\dot{\bfq})= s\:\limp_J(\bfq,\dot{\bfq}),
\end{align}
we have,
\eqnn{
\forall\; j\in J : \bfU_j\paren{\limp_J(\bfq,\dot{\bfq}_s)} = \bfU_j\paren{s\: \limp_J(\bfq,\dot{\bfq})}.
}
Therefore 
\eqnn{
\forall\; j\in J : \bfU_j\paren{\limp_J(\bfq,\dot{\bfq}_s)} \ge 0 \Leftrightarrow \bfU_j\paren{\limp_J(\bfq,\dot{\bfq})} \ge 0;
}
the conclusion of the Lemma follows directly.
\end{proof}

\begin{theorem}
\label{thm:pseudoimp}
Assume the same conditions as in Lemma~\ref{lem:pseudoimp}.
If further there is a unique constraint $i\in I\sm J$ such that,
\eqnn{ 
\label{eq:pimpcond}
\bfU_i\left( \bfAd_{J \cup \{i\}}(\bfq) \left(\Upsilon(\bfq,\dot{\bfq}_0) - \Cbar(\bfq,\dot{\bfq}_0) - \Nbar(\bfq,\dot{\bfq}_0) \right)\right) > 0,
}
(i.e., a constraint $i\in I$ is impinged upon by external forces but would be removed after impact with constraint $k$ if $\delta_t = 0$)
then for all values of the pseudo-impulse parameter $\delta_t > 0$, there exists $\obar{s} > 0$ such that with $\dot{\bfq}_s$ defined as in~\eqref{eq:dqs} we have,
\eqnn{\label{eq:CPPIVs}
\forall\; s\in (0,\obar{s}) : i\in \CP_\PIV(\bfq,\dot{\bfq}_s),
}
that is, the pseudo-impulse prevents the liftoff of constraint $i$ for all sufficiently small impacting speeds.
\end{theorem}

\begin{proof}
For all $s> 0$ let $K(s) := \CP_\PIV(\bfq,\dot{\bfq}_s)$.
Recall from~\eqref{eq:PIVIV} that $\CP_\IV(\bfq,\dot{\bfq}_s)\subset \CP_\PIV(\bfq,\dot{\bfq}_s)$ for all $s \ge 0$.
Thus although $K(s)$ may not be constant, we are only concerned with the asymptotic inclusion of 
additional constraints. 

Note that,
\eqnn{
\lim_{s\goesto 0^+} \limp_{K(s)}(\bfq,\dot{\bfq}_s) = \bf0,
}
and therefore for all values of the pseudo-impulse parameter $\delta_t > 0$,
\eqnn{
\lim_{s\goesto 0^+} \bfU_k(\limp_{K(s)}(\bfq,\dot{\bfq}_s)+\dimp_{K(s)}(\bfq,\dot{\bfq}_s)) = \bfU_k(\dimp_{K(s)}(\bfq,\dot{\bfq}_s)) .
}
for all $k \in K(s)$. By assumption, constraint $i$ is the only constraint such that $\bfU_i(\dimp) >0$ and,
by~\eqref{eq:PIV}, is to be included in the solution. As no other constraints are added or removed, 
there must exist some $\obar{s} > 0$ that ensures \eqref{eq:CPPIVs} holds.
\end{proof}

While Theorem~\ref{thm:pseudoimp} is limited to only single constraints, the pseudo-impulse parameter
can similarly prevent the impulsive liftoff of at least some constraints when there are multiple that 
satisfy~\eqref{eq:pimpcond}, subject to the nature of the $\PIV$ complementarity~problem. 

\subsubsection{Pseudo-Impulse Example}\label{sec:pseudoimpex}

\begin{figure}
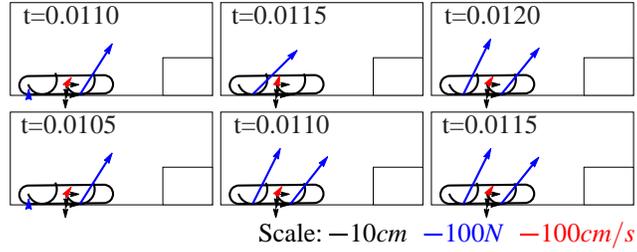

\centering
\def\svgwidth{8.3cm}
\include{pseudocomp}
  \vspace{-8pt}
\caption{Keyframes around the impact of the second leg with the ground -- note the difference in contact
forces (blue arrows) which indicate which contacts are active.
 \emph{Top Row:} Without pseudo-impulse ($\delta_t=0$). \emph{Bottom Row:} With pseudo-impulse ($\delta_t = 0.03$). 
The center top frame shows that the contact with the front leg is lost when the rear leg touches down (and
therefore no contact force is possible), but
the center bottom frame shows both contacts are maintained with the pseudo-impulse (and therefore
both contacts can apply forces to the system).
Note that there is a slight difference in touchdown time due to similar discrepancies around the time of the first leg touchdown. See also Figure~\ref{fig:lam1ncomp}.}
\label{fig:pseudoimp}
\end{figure}

\begin{figure}
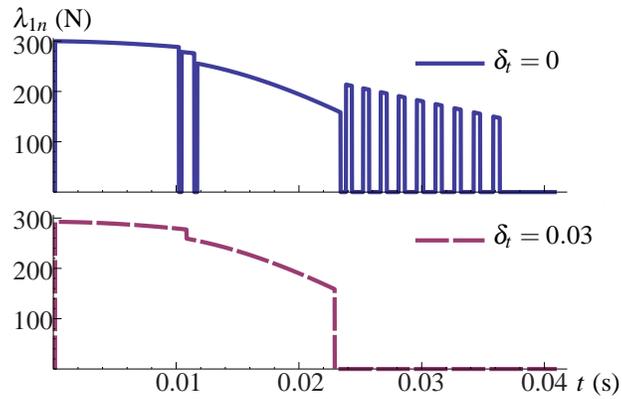

\centering
\def\svgwidth{7.9cm}
\include{lam1ncomp}
\caption{Comparison of the front leg normal direction ground reaction force for evaluations with and without the pseudo-impulse. See also Figure~\ref{fig:pseudoimp}.}
\label{fig:lam1ncomp}
\end{figure}

The pseudo-impulse can help to resolve certain Zeno executions, as shown in Section~\ref{sec:pseudozeno}, 
but more importantly it reduces ``chattering'', or executions that involve many impulsive transitions
that are qualitatively undesirable. \changed{As an example from} the RHex leaping simulation
of Figure~\ref{fig:leapsim} (see also the additional examples
explored in Section~\ref{sec:dis:pseudo}), Figure~\ref{fig:pseudoimp}
compares the state just before and after the rear leg touches down  
with and without the pseudo-impulse term.
At that instant the calculated impulse \eqref{eq:impulse} is $P_{l} = -1.47$Ns (in the normal direction on the
front leg). Even though the leg motor
is applying maximum torque trying to keep the leg on the ground the small negative impulse
causes the leg to separate, and then the motor torque quickly
accelerates the leg back to the ground (with or without massless legs, as recall that even massless legs are assumed to have finite acceleration, thus
the leg may return quickly but not instantly). With the pseudo-impulse 
this is balanced out by $P_{\delta,1} =7.91$Ns,
and the leg does not leave the ground (as would be the case on the real robot in this configuration to within modeling precision).
If the induced impulse were much larger then
the desired result may be for the front leg to lift off the ground,
while a much smaller impulse would clearly not break the front leg's contact.
The $\delta_t$ term is in essence a tuning parameter that determines the threshold 
between a \emph{quasi-static} regime (where contacts are maintained) and a \emph{dynamic} regime (where impulses may break existing contacts).

Impulsively breaking contact at the wrong time is an even bigger problem when considering
a full behavior and not just analyzing an individual impact event. As 
Figure~\ref{fig:lam1ncomp} suggests, without a pseudo-impulse this impulsive liftoff can lead to chattering.
In this case starting around $t=0.023$ the front leg lifts off but the continuous time forces
return the leg to the ground after a short time. When the front leg impacts the ground,
the rear leg then impulsively breaks contact, and a cyclic oscillation begins. 
This behavior is not quite a Zeno-execution, as the finite acceleration of the leg in the air results in only
finitely many transitions in a finite time, however these transitions are still qualitatively~undesirable.

\subsection{Friction}
\label{sec:friction}

While this paper is not focused on methods for modeling friction, including friction
in some form is unavoidable \cite[]{McGeer_Wobbling_1989}. Here, 
in order to advance our targeted conclusions guaranteeing the existence and uniqueness of a solution,
we assume~that,
\begin{assumption}[Friction]
\label{ass:friction}
All contact points with Coulomb friction are attached only to massless links. Contact
points without friction are assumed to never resist sliding motion, and all contact points
that are sliding have no kinetic coefficient of friction. No sliding-to-sticking transitions are considered.
\end{assumption}
The velocity constraint, $\bfA$, unilateral constraint cone, $\bfU$, and complementarity problems, $\CP$,
are thus taken to include any active frictional constraints -- see 
Section~\ref{sec:not} or \cite[Assumptions~C.3 and~C.4]{johnson_selfmanip_2013} for details.

This restrictive frictional assumption ensures that during impact (i.e., in evaluating $\CP_\IV$, the impulse--velocity complementarity
problem~\eqref{eq:CPIV}) any conflict involving the frictional constraints can be resolved by simply removing that contact
(the normal and tangential constraints)
from the active set
(see Section~\ref{sec:intfric} for a summary of pathologies that arise when this assumption is relaxed). 
As a massless link, it can always rotate out of the way fast enough
(as discussed above in Section~\ref{sec:complementarity}).

In continuous time (i.e., in evaluating $\CP_\FA$, the force--acceleration complementarity problem~\eqref{eq:CPFA})
this frictional assumption as applied to the RHex model 
states that the body has a low coefficient of friction and does
not resist tangential forces while the legs' rubber feet have a high coefficient of friction
and therefore \changed{typically do resist tangential forces.
However it is known that even for contacts with infinite friction, allowing a 
sliding mode is sometimes required to find a consistent solution to the frictional 
force--acceleration complementarity problem \cite[]{McGeer_Wobbling_1989}.
Furthermore, a strict}
\emph{a priori} assumption about friction is certainly not a good model for every situation -- consider
what happens when RHex's legs push against each other, as with the vertical
leap described in \cite{paper:johnson-icra-2013} (see in particular Footnote~8 and the end
of Section~V.A). In order to model such a behavior the leg contact points must be allowed
to transition to sliding contact when the contact forces reach the friction cone in the tangential direction,
$\bfU_{k}(\lambda)\geq 0$ \cite[Eqn.~4]{johnson_selfmanip_2013}, as with the liftoff condition
in~\eqref{eq:Ujl}--\eqref{eq:Ukl}. 
\changed{Allowing this transition enables for example
the simulation of the vertical leap shown in Figure~\ref{fig:vertsim} or the leap onto
a ledge shown in Figure~\ref{fig:ledgesim}, which each use a (hand tuned) value of $\mu=0.8$.
The forward leap of Figure~\ref{fig:leapsim} still requires this transition,
for much the same reason as in \cite[]{McGeer_Wobbling_1989}, but uses the 
relatively high values of $\mu=1.8$ for the front leg and $\mu=2.5$ for the rear.}
After transition to sliding the kinetic 
coefficient of friction is taken as $\mu_k =0$ (as with the frictionless body contact points) so that 
the jamming problems discussed in Section~\ref{sec:intfric} are again avoided.

\begin{figure}
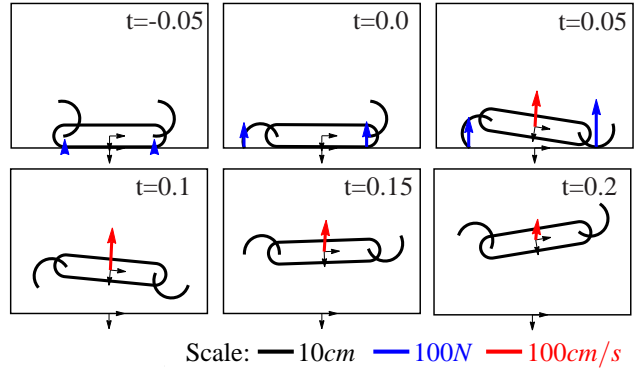

\centering
\def\svgwidth{8.3cm}
\include{vert_leap}
  \vspace{-8pt}
\caption{Keyframes from RHex simulation leaping vertically to a height of 37cm. Blue arrows show contact forces while
the red arrow shows body velocity. The coefficient of friction is $\mu=0.8$ and the relative leg
timing is $t_2=-0.06s$.
}
\label{fig:vertsim}
\end{figure}

The transition 
from sliding to sticking is much trickier. 
A sliding constraint sticks
when the tangential velocity drops to zero and the
resulting contact forces lie within the friction cone, i.e.\ contact $k \in \KT$ is to be
added if and only if its corresponding normal constraint is active ($\alpha(k) \in I$) and, 
\begin{align}
\bfA_{k} \dot{\bfq} =0 \wedge \bfU_k(\lambda_{I \cup \{k\}}) \succeq 0. \label{eq:slst}
\end{align}
However this additional condition complicates the force--acceleration complementarity problem and furthermore
is not needed to model any of the leaping behaviors in \cite{paper:johnson-icra-2013}, our motivating
scenario. Therefore for the purposes of this paper we do not consider such slip-to-stick transitions.
This limitation prevents, e.g., the modeling of a leg that slides upon contact with the ground but
gains traction later using static friction.

\section{Hybrid Dynamics}
\label{sec:hyb}

In this section we first define a general $C^r$ \emph{hybrid dynamical system} (Section~\ref{sec:hs})
that is then instantiated as the main object of study for this paper, the \emph{self-manipulation hybrid 
dynamical system} (Section~\ref{sec:smsystem}).
Section~\ref{sec:correct} establishes that this system is indeed a $C^r$ 
hybrid dynamical system and Section~\ref{sec:consistency} further shows its internal consistency.
Finally, Section~\ref{sec:zeno} shows that Zeno executions of the system accumulate
and that the pseudo-impulse truncates certain Zeno executions.

\subsection{$C^r$ Hybrid Dynamical System}
\label{sec:hs}
In the following definitions we make use of the natural (disjoint-union) topology on the hybrid state space,
consisting of a collection of manifolds with corners \cite[Def.~2.1]{Joyce2012}; 
see Appendix~\ref{app:hdg} or \cite[Sec.~II]{BurdenRevzen2013} for more details. 
The hybrid system notation introduced in this section is summarized in Table~\ref{tab:Symbols2}.

\begin{table}[t]
  \begin{center}
    \begin{tabular}{ l l}
      \toprule
      \midrule
      $\calD:=\mycoprod_{I\in\calJ} D_I$ & Hybrid system domain (Def.~\ref{def:hs})$\!\!\!\!\!$\\
      $\calF: \calD \rightarrow T\calD, F_I:= \calF|_{D_I}$ & Vector field (Def.~\ref{def:hs}) \\
      $\calG :=\mycoprod_{(I,J)} G_{I,J},\ G_{I,J} \subset D_I$ & Guard set (Def.~\ref{def:hs}) \\
      $\mathcal{H}:=(\mathcal{J},\nGamma,\mathcal{D,F,G,R})$ & Hybrid dynamical system (Def.~\ref{def:hs})$\!\!\!\!\!$\\
      $\mathcal{J} \subset \mathbb{N}$ & Discrete indexing set (Def.~\ref{def:hs})\\
      $\mathcal{R}:\calG\rightarrow \calD,\ R_{I,J}:=\calR|_{G_{I,J}}$ & Reset map (Def.~\ref{def:hs})\\
      $\calT:=\mycoprod_{i} T_i,\ T_i \subset \Real$ & Hybrid time domain (Def.~\ref{def:time})\\
      $\nGamma \subset \mathcal{J} \times \mathcal{J}$ & Set of discrete transitions (Def.~\ref{def:hs})$\!\!\!\!\!$\\
      $\sigma(\chi)\in \mathcal{J}^N$ & Contact word of length $N$ (Def.~\ref{def:ex})$\!\!\!\!\!$\\
      $\chi:\calT\into \e{D}$ & Execution of the system (Def.~\ref{def:ex})$\!\!\!\!\!$\\
      \bottomrule
    \end{tabular}
     \caption{Hybrid system and execution symbols, with definition of introduction marked.
     See also Table~\ref{tab:Symbols} for symbols introduced in Section~\ref{sec:imp}.
     }
     \label{tab:Symbols2}
  \end{center}
\end{table}

\begin{definition}
\label{def:hs}
A $C^r$ \emph{hybrid dynamical system}, $r\in\Nat\cup\set{\infty,\omega}$, is a tuple 
$\mathcal{H}:=(\mathcal{J},\nGamma,\mathcal{D,F,G,R})$, where the constituent parts are defined as:
\begin{enumerate}
\item $\mathcal{J} := \{I,J,\dots,K\} \subset \mathbb{N}$ is the finite set of \emph{discrete modes};
\item $\nGamma \subset \mathcal{J}\times\mathcal{J}$ is the set of \emph{discrete transitions}, forming a directed 
graph structure over $\mathcal{J}$;
\item ${\mathcal D} := \mycoprod_{I\in{\mathcal J}}D_I$ is the collection of \emph{domains}, where $D_I$ is a $C^r$ manifold with corners; 
\item $\calF : {\mathcal D}\into T{\mathcal D}$ is a $C^r$ hybrid map that restricts to a vector field $F_I := \calF|_{D_I}$ for each $I\in\e{J}$;
\item $\mathcal{G} :=  \mycoprod_{(I,J)\in\nGamma} G_{I,J}$ is the collection of \emph{guards}, where $G_{I,J}\subset D_{I}$ for each $(I,J)\in\nGamma$;
\item $\calR : \e{G} \into{\mathcal D}$ is a continuous map called the \emph{reset} that restricts as $R_{I,J} := \calR|_{G_{I,J}}:{G_{I,J}}\into{D_J}$ for each $(I,J)\in\nGamma$.
\end{enumerate}
\end{definition}

Before we proceed, we make a few clarifying comments about this definition.
While $\nGamma$ is a directed graph it is not generally a tree (i.e., $(I,J)$ and $(J,I)$ may both be members).
When we write $\mathcal{G} :=  \mycoprod_{(I,J)\in\nGamma} G_{I,J}$ where $G_{I,J}\subset D_{I}$ for each $(I,J)\in\nGamma$, we are simultaneously specifying that (i) each $G_{I,J}$ is an arbitrary subset of $D_{I}$ and (ii) $\e{G}$ is the finite disjoint union of these subsets.
The domain $\e{D}$ should be regarded as a $C^r$ hybrid manifold as described in Appendix~\ref{app:hdg} since each $D_I$ is a $C^r$ manifold with corners. 
In contrast the guard $\e{G}$ does not generally possess a smooth structure since each $G_{I,J}\subset D_I$ is not even required to be a topological manifold.
We say that $\e{H}$ \emph{has disjoint guards} if
$G_{I,J}\cap G_{I,L}=\varnothing$ for each pair $(I,J)$, $(I,L)\in\nGamma$ such that $J \ne L$.
An illustration of some of the elements of a $C^r$ hybrid dynamical system is shown in Figure~\ref{fig:hybrid}.

\begin{figure}
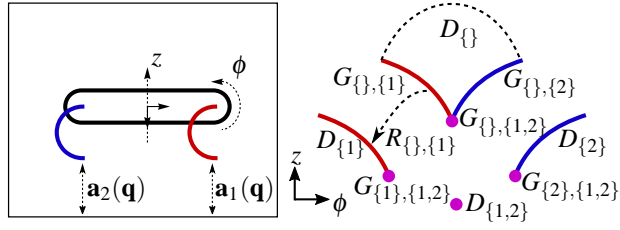

\centering
\def\svgwidth{8.3cm}
\include{hybrid_}
\caption{
Illustration of elements from the $C^r$ hybrid dynamical system (Definition~\ref{def:hs}) for the RHex model.
Note that this is a 5-dimensional model ($\dim\e{Q} = 5$), so we cannot faithfully represent the domains and guards on the printed page; instead, we illustrate a two-dimensional slice 
of height $z$ and body pitch $\phi$.
From the unconstrained mode $D_{\set{}}$ there are three possible discrete transitions corresponding to touchdown of the front leg ($G_{\set{},\set{1}}\subset\set{(\bfq,\dot{\bfq})\in T\e{Q} : \bfa_1(\bfq) = 0}$), rear leg ($\bfa_2(\bfq) = 0$), or simultaneous touchdown of both legs. 
We annotate the reset map $R_{\set{},\set{1}}$ corresponding to front leg touchdown, but emphasize that there are corresponding maps defined over all the guards $G_{I,J}$.
Each constrained mode also generally contains liftoff guards (e.g.\ $G_{\set{1},\set{}}\subset D_{\set{1}}$); these are not illustrated.
}
\label{fig:hybrid}
\end{figure}

Roughly speaking, an \emph{execution} of a hybrid dynamical system is set in motion from an initial condition in ${\mathcal D}$ by following the continuous-time dynamics determined by the vector field $\e{F}$ until the trajectory reaches the guard ${\mathcal G}$, at which point the reset map $\e{R}$ is applied to obtain a new initial condition.
We formalize this using the notion of a \emph{hybrid time domain}.

\defn{\label{def:time}
A \emph{hybrid time domain} is a disjoint union of intervals $\calT := \mycoprod_{i=1}^N T_i$ such that:
\begin{enumerate}
\item $N\in\Nat\cup\set{\infty}$;
\item $T_i\subset\Real$ is a closed interval for all $i < N$, and if $N < \infty$ then $T_N\subset\Real$ is also a closed interval; 
\item $T_i\cap T_{i+1}$ is nonempty and consists of a single element for all $i < N$.
\end{enumerate}
Note that an interval may be \emph{degenerate}, i.e.\ $T_i = \set{t_i}$.
We define $\sup\e{T} := \sup\bigcup_{i=1}^N T_i$.
}

This definition is essentially equivalent to the \emph{hybrid time trajectory} \cite[]{LygerosJohansson2003}, 
the \emph{hybrid time set} \cite[]{Collins2004}, and the \emph{hybrid time domain} \cite[]{GoebelTeel2006}, 
and enables us to formalize the conceptual description of the domain of a hybrid execution from \cite{Back_Guckenheimer_Myers_1993} 
as being divided into contiguous \emph{epochs} separated by \emph{events} where the reset map is applied at an instant referred to as an \emph{event time}.
Furthermore, this definition has two appealing features.
First, an \emph{execution} (defined below) becomes a smooth (hybrid) map defined from a hybrid time domain $\calT$ into the continuous state space $\calD$ of the hybrid system, avoiding the use of set-valued maps or cumbersome left- or right-handed limits; see Appendix~\ref{app:hdg} for the definition of smoothness for hybrid maps.
Second, it treats the model of time in the same class of mathematical objects as the model for the state space, namely, a disjoint union of smooth manifolds with corners.
Note that under this definition a transition time $t_i \in T_i \cap T_{i+1}$ appears in two consecutive components of the time domain $T_i$ and $T_{i+1}$, allowing the flow on each interval to include
both endpoints. 
Also note that this allows for two transitions (or more) to occur at the same instant in time, e.g.\ it is possible that
$T_i = \set{t_i}, T_{i-1} \cap T_{i} \cap T_{i+1} \cap ... = \set{t_i}$; the middle portion of the trajectory would have been excised from a left- or right-handed definition of execution, or potentially muddled with the surrounding trajectory portions in a set-valued definition.

\defn{\label{def:ex}
An \emph{execution} of a hybrid dynamical system $\mathcal{H}=(\mathcal{J},\nGamma,\mathcal{D,F,G,R})$ 
is a smooth map $\chi:\calT\into \e{D}$ over a hybrid time domain $\calT = \mycoprod_{i=1}^NT_i$ satisfying:
\begin{enumerate}
\item $\forall\; i\in\Nat, i\le N$: if $T_i$ is not a degenerate interval then $\dt{t} \chi|_{T_i}(t) = \e{F}(\chi(t))$ for all $t\in T_i$;
\item $\forall\; i < N$: for $\set{t_i} = T_i\cap T_{i+1}$ (the \emph{event times}), we have $\chi|_{T_i}(t_i)\in \e{G}$, $\e{R}(\chi|_{T_i}(t_i)) = \chi|_{T_{i+1}}(t_i)$, and for all $s\in T_i\sm\set{t_i}$ we have $\chi|_{T_i}(s)\not\in\e{G}$.
\end{enumerate}
The execution has an associated \emph{word} denoted by $\sigma(\chi) = \set{J_i}_{i=1}^N\in\e{J}^N$ that specifies the sequence of discrete modes encountered by the execution: $\chi|_{T_i}\subset D_{J_i}$ for all $i\in\Nat$, $i \le N$.
An execution $\chi:\e{T}\into\e{D}$ is \emph{maximal} if it cannot be extended to an execution over a longer hybrid time domain.
We say\footnote{ Following \cite{LygerosJohansson2003}.} that $\e{H}$ is: \emph{deterministic} if for every initial condition $x\in\e{D}$ there exists a unique maximal execution $\chi:\e{T}\into\e{D}$ such that $\chi|_{T_1}(0) = x$; \emph{non-blocking} if for every initial condition $x\in\e{D}$ and any maximal execution $\chi:\e{T}\into\e{D}$ with $\chi|_{T_1}(0) = x$, then with $\e{T} = \mycoprod_{i=1}^N T_i$ either $N = \infty$ or $N < \infty$ and $T_N = [t_N, \infty)$.
}
The \emph{contact word}, $\sigma(\chi)$, also called the \emph{contact motion plan}, is useful for comparing
and reasoning about different executions of the hybrid system \cite[]{xiao2001automatic}.

\subsection{The Self-Manipulation System}
\label{sec:smsystem}

While the previous hybrid system specification is very general, it is useful to instantiate it
for a model of a physical system. This section defines the self-manipulation system \cite[]{johnson_selfmanip_2013} 
(and by the analogy of that paper, equivalently a manipulation system \cite[]{book:mls-1994}, as summarized in
Section~\ref{sec:manip}), 
where the discrete mode, $I \subset \calK$, corresponds to the active contact mode.

\begin{definition}
\label{def:smhs}
A \emph{self-manipulation hybrid system} is a $C^r$ hybrid dynamical system, $\mathcal{H}_s=(\mathcal{J},\nGamma,\mathcal{D,F,G,R})$,
defined as follows, 
\end{definition}
\subsubsection{Discrete Modes} 
The set of modes, all physically permissible combinations of contact constraints, is given by,
\begin{align}
\mathcal{J} = \Big\{I\in 2^{\e{K}} : & \bfa_{I_n}^{-1}(0) \neq \varnothing \wedge \alpha(I_t) \subset I\Big\}, \label{eq:jdef}
\end{align}
that is there are two requirements: 1) there must exist some point, $\bfq \in \calQ$, such that 
these normal contact constraints are active, $\bfa_{I_n}(\bfq)=0$, and 2) no tangential constraint
is included whose corresponding normal constraint is not also included, i.e.\ $\nexists\; i \in I_t : \alpha(i) \notin I$.

\subsubsection{Edges} 
The set of edges is made up of any pair of modes whose union is also a mode -- in 
other words, sets arising from the intersection of the two base sets that satisfy 
respectively the two sets of normal constraints,

\begin{align}
\nGamma &= \left\{(I,J)\in \e{J}\times\e{J} :
I \neq J, 
I \cup J \in \calJ
\right\}. \label{eq:gamdef}
\end{align} 
This set of edges can be further restricted based on the guards, defined below,
as there are some transitions $(I,J)\in \nGamma$ where no points in $D_I$ satisfy the requirements of
the guard, i.e.\ $G_{I,J}=\varnothing$. In that case we reduce the edge set~further,
\begin{align}
\tilde{\nGamma} = \{(I,J) \in \nGamma: G_{I,J} \neq \varnothing \}. \label{eq:tgam}
\end{align}

\subsubsection{Domains} 

The domain associated with a contact mode $I\in\e{J}$ is the subset of the ambient tangent bundle $T\e{Q}$
that satisfies the normal non-penetration and tangential non-sliding constraints,
\begin{align}
D_{I} = &\left\{ (\bfq,\dot{\bfq}) \in T\mathcal{Q}:
\bfa_{I_n}(\bfq) =  0,
\bfa_{\KN}(\bfq)\geq0, 
\bfA_{I}(\bfq)\dot{\bfq} = 0
 \right\}, \label{eq:Ddef}
\end{align}
where recall that $\mathcal{Q}:=\Theta \times SE(\mathrm{d})$ is the joint space combined with the position
space of the body.

\subsubsection{Flows} 
The vector field on each domain is based on the self-manipulation dynamics for
$\ddot{\bfq}$, as in~\eqref{eq:dyn} and \cite[Eqn.~33,~72]{johnson_selfmanip_2013}, are,
\begin{align}
F_I (\bfq,\dot{\bfq}) = \left[ \dot{\bfq}, \quad \Md(\Upsilon - \Cbar \dot{\bfq} - \Nbar) - \bfAdT_I \dot{\bfA}_{I} \dot{\bfq}
\right]. \label{eq:smflow}
\end{align}
for the coordinates in $\calQ_I$, and recall from 
Assumption~\ref{ass:uml} (unconstrained massless limbs) that the coordinates associated with unconstrained massless limbs lie in the
subspace $\tilde{\calQ}_I$ and evolve according to the vector field $\tilde{F}$, such that the combined
vector field is complete over all of $T\calQ$. 
The control input $\tau \in T^*\Theta$ 
that appears in $\Upsilon$ is prescribed by a $C^r$ function of state $\tau\in C^r(T\e{Q},T^*\e{Q})$
(for example a fixed-voltage motor model $\tau_i = \kappa_P\kappa_G(1-\kappa_G \dot{\theta}_i)$ \cite[Sec.~IV.C.4]{johnson_selfmanip_2013}).

\subsubsection{Guards} 
We find it convenient to construct the guard set, for mode $I$, $G_I \subset D_I$, 
as a union of subsets indexed by its ``outgoing'' edges, $(I,J) \in \nGamma$, 
using the touchdown predicate~\eqref{eq:NTD} and 
the complementarity problem predicates~\eqref{eq:PIV} 
and~\eqref{eq:CPFA} specified as\footnote{ Note that the 
requirement that $(I,J) \in \nGamma$ and therefore $J \in \calJ$ ensures that all 
tangential constraints in the new contact condition must have a matching normal 
constraint also (or trending so). \changed{Furthermore,
as in~\eqref{eq:tgam}, note} that only some of these outgoing edges make a non-empty 
contribution, $G_{I,J}$, to this union. },
\begin{align}
G_{I,J} = \Big\{\biq \in D_{I}:\;
&\NTD(\biq) \Rightarrow \PIV( J, \biq),  \label{eq:gdpiv}\\
\lnot\; &\NTD(\biq) \Rightarrow \FA( J, \biq) \label{eq:gdffa}        
\Big\}.
\end{align}
Conceptually, the component of the guard for 
mode $I$ associated with edge $(I,J)$ consists of any base states, $\bfq$, at which 
any new touchdown event can occur from mode $I$ into mode $J$, 
according to the $\NTD$ predicate,~\eqref{eq:NTD}, subject to $\PIV$ 
complementarity,~\eqref{eq:PIV}. An additional
condition on the base and tangent states, $\biq$, is that if no new contacts are 
touching down (``liftoff''), then $\FA$ complementarity 
holds~\eqref{eq:CPFA}. 

The outlet set in domain $I$, defined as $G_I := \cup_J G_{I,J}$, is,
\begin{align}\label{eq:sm:outlet}
G_I &= \Bigg\{ \biq \in D_I: 
\Bigg(\bigvee_{k\in\KN \sm I}\bfa_k(\bfq) \preceq 0\Bigg) 
 \bigvee \Bigg(\bigvee_{i\in I} \bfU_i(\lambda_I)\prec 0\Bigg) 
  \Bigg\}.
\end{align}
The outlet set is used in the proof of Theorem~\ref{thm:nb}, but more importantly provides a
computationally expedient method of simulating an execution: first check if $\biq \in G_I$, then
determine the subsequent mode afterwards,
\begin{align}
J = \left\{ \begin{array}{lr}
\CP_{\PIV}(\biq), \qquad& \NTD(\biq)\; \\
\CP_{\FA}(\biq), & \lnot\;\NTD(\biq). 
\end{array} \right.  \label{eq:gdJcalc}
\end{align}

\subsubsection{Reset Maps} \label{sec:reset}
The reset map associated with edge $(I,J) \in \nGamma$ (taking its domain exactly on $G_{I,J}$, defined above) is,
\begin{align}
 R_{I,J}(\bfq,\dot{\bfq}) &=\left[\bfq, \quad \dot{\bfq} - \Delta \dot{\bfq}_J \right]
= \left[\bfq, \quad \dot{\bfq} - \bfAdT_{J}\bfA_{J} \dot{\bfq}  \right]. \label{eq:reset}
\end{align}
Note that for takeoff events, $J_n \subseteq I_n$, the prior velocity already agrees with the 
new contact mode and therefore the impact map has no effect.

\subsection{The Self-Manipulation System is a Hybrid System}
\label{sec:correct}

This section shows that,
\begin{theorem}
\label{thm:agree}
The self-manipulation system (Def.~\ref{def:smhs}) is a $C^\omega$ hybrid dynamical system (Def.~\ref{def:hs}).
\end{theorem}

\begin{proof}
Definition~\ref{def:hs} has a number of requirements and so this proof \changed{is broken up} into the constituent 
parts and show that each component of Definition~\ref{def:smhs} is compatible with the requirements.
 
\begin{enumerate}[nolistsep]
\item $\calJ$ in~\eqref{eq:jdef} is a finite set, the only requirement on $\calJ$.
\item $\nGamma\,$ in~\eqref{eq:gamdef} is a subset of $\calJ \times \calJ$ by construction, and $\tilde{\nGamma}$ in~\eqref{eq:tgam} 
is a subset of $\nGamma$.
\item By Assumption~\ref{ass:rank} (simple constraints), for all $I\in\calJ$ the maps $\bfa_{I_n}\in C^\omega(\e{Q},\calC_{I_n})$ and 
$\bfA_{I}\in C^\omega(T\e{Q},T\calC_{I})$ are constant rank, and therefore
each $D_I\subset T\e{Q}$, as defined by these functions in~\eqref{eq:Ddef} is a closed $C^\omega$ manifold with 
corners \cite[Thm.~5.12]{Lee2012}, \cite[Def.~2.1]{Joyce2012}.
The Nash Embedding Theorem \cite[]{Nash1966} \changed{states that} $\e{Q}$ can be embedded analytically in Euclidean space of sufficiently high dimension; this embedding therefore restricts to an embedding of the submanifold $D_I$.

\item For all $(\bfq,\dot{\bfq})\in D_I$, $F_I(\bfq,\dot{\bfq})\in T_{(\bfq,\dot{\bfq})} D_I$ 
for $F_I$ given in~\eqref{eq:smflow} (based on~\eqref{eq:dyn}, \cite[Eqn.~33]{johnson_selfmanip_2013}, which enforce
the equality constraints of the definition of the domain,~\eqref{eq:Ddef}, and 
therefore lie within $TD_I$) and hence we may write $F_I\in C^\omega(D_I, TD_I)$.
\item $G_{I,J}$ in~\eqref{eq:gdpiv} is a subset of $D_I$ by construction.
\item For the reset map in~\eqref{eq:reset}, $\im R_{I,J}(G_{I,J}) \subset D_{J}$, 
as the domain $D_{J}$ has three requirements~\eqref{eq:Ddef}:

1) $\bfa_{J_n}(\bfq) = 0$, i.e.\ any normal constraints are touching
the surface indicated. For pre-existing constraints, $\{i: i \in I_n\cap J_n \}$, this requirement is 
already guaranteed, i.e.\ $\bfq \in G_{I,J} \subset D_{I} \Rightarrow \bfa_{I_n}=0$, and the
reset map does not alter the base coordinates $\bfq$. New
normal constraints in mode $J$, $\{j: j \in J_n \backslash I_n\}$, satisfy this requirement
by the touchdown predicate in the guard~\eqref{eq:gdpiv} (and~\eqref{eq:TDsimp}),
where $\TD(j,\biq)$ is true only when $\bfa_{j}(\bfq)=0$. 

2) $\bfa_{\KN}(\bfq) \geq 0 $, i.e.\ all base constraints are non-negative. 
Again since the reset map does not alter the base coordinates $\bfq$, then 
$\bfq \in G_{I,J} \subset D_{I} \Rightarrow \bfa_{\KN}\geq 0$.

3) $\bfA_{J} (\bfq) \dot{\bfq}=0$, i.e.\ any velocity in constrained directions is zero, but this is guaranteed by
the reset map as $\bfA_J \dot{\bfq}^+ = \bfA_J \dot{\bfq}^- - \bfA_J \bfAdT_{J}\bfA_{J} \dot{\bfq}^- = 0$. 
Therefore, as claimed, the image of the reset map~\eqref{eq:reset} lies within $D_{J}$.

\end{enumerate}
\end{proof}

\subsection{Consistency Properties}
\label{sec:consistency}

This section establishes several additional properties of the self-manipulation 
system that are of practical importance, Theorems~\ref{thm:disjoint}--\ref{thm:uniqueex},
which we shall for convenience collectively call \emph{consistency} properties.

\begin{theorem}
\label{thm:disjoint}
The self-manipulation system (Def.~\ref{def:smhs}) has disjoint guards.
\end{theorem}
\begin{proof}
The disjointedness of the guards follows directly from the assumption of uniqueness of solution for the constituent
complementarity problems. Define the liftoff predicate where no contact is touching down,
\begin{align}
\LO(I,\biq):= \lnot\;\NTD(\biq) \wedge \left(\bigvee_{k \in \calI} \bfU_k(\lambda_I)\prec0\right), \label{eq:LO}
\end{align}
and so by a refinement of the complementary block of the partition defined on the right hand side of~\eqref{eq:gdJcalc},
\begin{align}
\CP_G(I,\biq) = \left\{ \begin{array}{lr}
\CP_{\PIV}(\biq), \quad& \NTD(\biq) \\
\CP_{\FA}(\biq), & \LO(I,\biq) \\
I, & \emph{otherwise},
\end{array} \right.
\end{align} 
is equal to either the unique mode $J$ for which $\biq \in G_{I,J}$, or simply $I$ if 
the state is not in a guard, $\biq \notin G_{I,J} \forall\; (I,J) \in \nGamma$.
\end{proof}

In their most general formulation, hybrid dynamical systems can accept executions that terminate before infinite time (continuous or discrete) has elapsed, or accept multiple distinct executions from the same initial condition.
This behavior is undesirable in practice, and inconsistent with our 
experience on real manipulation and self-manipulation systems.
Necessary and sufficient conditions \cite[Lems.~III.1,~III.2]{LygerosJohansson2003} have been formulated that ensure a system is \emph{deterministic} and \emph{non-blocking}.
Since these conditions are
applicable to a general class of hybrid dynamical systems, they can be difficult to verify directly for particular classes of models.
No previous authors have established that these conditions hold for any broad class of hybrid system models 
for Lagrangian dynamics subject to multiple unilateral constraints, much less with the particular structure of the self-manipulation system.
The conditions in \cite[Lem.~1~\&~2]{JohanssonEgerstedt1999} come closest, as 
they would apply to an instance of a self-manipulation system exhibiting only a single constraint.

To serve the needs of the present paper, we introduce an extension 
of the line of reasoning in \cite{LygerosJohansson2003} establishing that the
self-manipulation system, Definition~\ref{def:smhs}, is indeed deterministic and non-blocking,
Definition~\ref{def:ex}, in the presence of an arbitrary number of unilateral constraints.

\begin{theorem}
\label{thm:det}
The self-manipulation system is deterministic.
\end{theorem}

\begin{proof}
Assumption~\ref{ass:complete} (Lagrangian dynamics) imposes a partial flow on each component $D_I$, hence 
continuous trajectories are unique and nondeterminism could only be introduced 
through a reset. But the definition of execution, Def.~\ref{def:ex}, implies that a discrete transition 
occurs at $\biq\in D_I$ if and only if there exists $J\in\e{J}\sm\set{I}$ such that $\biq \in G_{I,J}$.
Since the guards are disjoint by Theorem~\ref{thm:disjoint},
there is at most one guard containing $\biq$. 
The execution continues from $R_{I,J}(\biq)\in D_J$.
\end{proof}

The non-blocking property is a bit more subtle as the self-manipulation systems escape some of the 
structure required to handle the more general class of systems addressed in \cite{LygerosJohansson2003} and
used there to establish conditions for non-blocking. For self-manipulation hybrid systems the non-blocking 
property arises from the discrete logic and continuous dynamics in an essential manner that we now rehearse 
informally in preparation for the statement and proof of Theorem~\ref{thm:nb}. The 
guard, $G_I$, intersects the corresponding domain, $D_I$, both on the boundary of the domain (to handle impact on an erstwhile 
inactive constraint) as well as in the interior of the domain (to handle a sign change on some active 
constraint's contact force). An execution might be blocked by conventional finite escape, i.e., if 
the continuous flow brings some initial state to the boundary of the domain at a point in the complement
of the guard in finite time. Alternatively, it might be blocked by hybrid ambiguity, i.e., if the continuous 
flow brings some initial state to some point that is in the complement of the guard but still lies
on the boundary of the guard, for this would violate the semantics of execution that restricts continuous 
flow to closed intervals (formally, Def.~\ref{def:ex} requires a minimum -- not merely an infinum -- time of entry 
into a guard). In the following proof we preclude both cases by showing that the guard contains 
all points reached by the continuous flow that lie in the boundary of the domain or the interior of the domain
but the boundary of the guard.

\begin{theorem}
\label{thm:nb}
The self-manipulation system is non-blocking.
\end{theorem}

\begin{proof}

Recall from part 6 of the proof of Theorem~\ref{thm:agree} that the image of the guard set under the reset
map is within the domain (and thus the discrete transition is never blocking),
and from Assumption~\ref{ass:complete} (Lagrangian dynamics) that the flow is forward complete over $T\calQ$.
Therefore we need only check that the flow only reachs the boundary of the domain or the 
boundary of the guard within the domain interior at a point which is included in the guard.

Recall from~\eqref{eq:Ddef} the definition of $D_I$, and note that it is a~subset,
\eqnn{\label{eq:DIcond}
D_I \subset \set{\biq\in\e{Q} : \bigwedge_{k\in\KN \sm I} \bfa_{k}(\bfq)\geq0},
}
where furthermore,
\eqnn{\label{eq:velcons}
\paren{\bigwedge_{i\in I_n} \bfa_{i}(\bfq) =  0}\wedge \paren{\bigwedge_{i\in I} \bfA_{I}(\bfq)\dot{\bfq} = 0}.
}
Note that the constraints in~\eqref{eq:velcons} are invariant under the flow of~\eqref{eq:dyn}, whence 
under the completeness assumption (Assumption~\ref{ass:complete}) it is only possible to flow out 
of $D_I$ in forward time by violating one of the inequality conditions of~\eqref{eq:DIcond}.

Recall from~\eqref{eq:sm:outlet} the union of all guards, $G_I$, and
then using Lemma~\ref{lem:closure} note that the
closure of the union of all guards is,
\begin{align}
\label{eq:barGI}
\bar{G}_I &= \Bigg\{ \biq \in D_I: 
\Bigg( \bigvee_{k\in\e{K}\sm I}\bfa_k(\bfq) \le 0\Bigg) \bigvee \Bigg(\bigvee_{i\in I} \bfU_i(\lambda_I) \le 0\Bigg)\Bigg\}.
\end{align}

Now consider an arbitrary point in the domain, $\biq \in D_I$. 
If $\mu(\biq) > 0$ for all $\mu \in \set{\bfa_k}_{k\in\e{K}\sm I}\cup\set{\bfU_i(\lambda_I)}_{i\in I}$ 
then $\biq$ is on the interior of $D_I\backslash G_I$ and it is possible to flow for positive time while remaining in the domain $D_I$ and not reach a guard $G_I$.
Otherwise there exists $k\in\e{K}\sm I$ such that $\bfa_k(\biq) = 0$ (i.e., the state has reached the boundary of the domain) 
or there exists $i\in I$ such that $\bfU_i(\lambda_I(\biq)) \le 0$ (i.e., a sign change on some active 
constraint's contact force). We now consider the two (mutually exclusive) possibilities concerning 
whether a contact condition or an active force is trending negative: 
\begin{enumerate}[nolistsep]
\item $\bfa_k(\biq) \succ 0$ for all $k\in\e{K}\sm I$ or $\bfU_i(\lambda_I)\succeq 0 $ for all $i\in I$;
\item there exists $k\in\e{K}\sm I$ such that $\bfa_k(\biq) \preceq 0$ or there exists $i\in I$ such that $\bfU_i(\lambda_I(\biq)) \prec 0$.
\end{enumerate}
\noindent
In case 1), when there is neither a negative trending contact nor active force, then
it is possible to flow for positive time in the domain without intersecting any guard
or leaving the domain (Lemma~\ref{lem:trend}); this provides the unique extension to the execution.
The contrary case 2) is just the situation the hybrid system's logic is designed to flag: i.e., 
$\biq$ is in a guard,~\eqref{eq:sm:outlet}, so the application of the reset map provides 
the unique extension to the execution.

Therefore every initial condition $\biq\in D$ yields a unique execution defined over a hybrid time 
domain that spans infinite time (continuous or discrete), whence the self-manipulation system is 
non-blocking.
\end{proof}

The self-manipulation hybrid system may undergo multiple hybrid transitions in succession 
at the same time $t$, as there is no ``dwell time'' requirement to continue after reset under the continuous dynamics
for any minimum amount of time. Therefore it is important to bound the number of such 
multiple transitions to ensure that the continuous execution eventually continues over an open interval of time. Here, 
Theorem~\ref{thm:dbltrans} relates the image of the reset map to the guard sets
to show that continuous execution continues after at most two successive hybrid transitions.
 
As a simple example, consider the self-manipulation system model consisting of a point mass in a gravitational field that points away from a constraint surface (i.e., a ball under a ceiling).
If the mass is initialized with a velocity that causes it to impact the constraint surface, it transitions first to the constrained
(ceiling) contact mode through an impact that ensures zero relative velocity. After spending zero time in the constrained mode,
and therefore at the same continuous time, the system transitions again back to the unconstrained mode as the mass succumbs to
gravity. The execution continues in the unconstrained mode as the mass accelerates away from the ceiling. 
In the self-manipulation system these are treated as separate transitions.
At the expense of a small amount of additional bookkeeping in the definition of execution, we dramatically 
simplify the specification of the reset map (in this example, eliminating the need for a reset map
from the unconstrained mode to the same unconstrained mode consisting of impulses from a constraint not present
in either the original or destination mode). 

\begin{theorem}
\label{thm:dbltrans}
An execution of a self-manipulation hybrid system without massless limbs may undergo no more than two hybrid transitions at
a single time $t$.
\end{theorem}
\begin{proof}
The guards are partitioned into two types by the new touchdown predicate, $\NTD(\biq)$,~\eqref{eq:NTD}, into \emph{touchdown}
and \emph{liftoff} (non-touchdown) components. It suffices to show simply the reset map 
(i) takes states that are in touchdown guards to  either non-guard states, or states in a liftoff guard, 
and that (ii) the reset map always takes states in the liftoff guard to non-guard states.

To show (i), note that for all points $\biq$ in a touchdown guard,~\eqref{eq:gdpiv},
the impulse--velocity complementarity ensures that all constraints $k$ that are touching
but are not in the outgoing contact mode $J$ have a separating velocity, $\bfA_k \dot{\bfq}^+>0$, 
after the application
of the reset map, if such a velocity is well defined~\eqref{eq:UkP2} (as in Theorem~\ref{thm:ivc}). 
Therefore all constraints $k$ not in $J$ must either not be 
touching ($\bfa_k(\biq^+) > 0$) or have a separating velocity ($\bfA_k \dot{\bfq}^+ > 0$), but therefore
cannot satisfy the touchdown predicate, $\TD(k,\biq^+)$,~\eqref{eq:TDsimp}, and so $\NTD(\biq^+)$ is false and the
state after the reset map is either not in a guard or is in a liftoff guard.

To show (ii), note that for all points $\biq$ in a liftoff guard, no contacts are touching down, 
i.e.\ $\forall\;k\in\KN, \TD(k,\biq)$
is false. The force--acceleration complementarity problem that defines these guards,
$\FA(J,\biq)$,~\eqref{eq:CPFA}, does not depend on the active mode, $I$. Furthermore, the reset map $R_{I,J}(\biq)$ is
simply identity, and so the state remains the same after this transition, $\biq=R_{I,J}(\biq)$. Therefore
$\biq \in G_{I,J}$ such that $\NTD(\biq)$ is false and $\FA(J,\biq)$ is true
implies that $\NTD(R_{I,J}(\biq))$ is also false and that $\FA(J,R_{I,J}(\biq))$ is still the correct
solution, and the state after the reset map is not in any guard.

Therefore at any given time $t$, the system can undergo at most two transitions -- first if some contact
is touching down the system undergoes an impulsive transition, and then if some contact force is trending
negative it undergoes a smooth liftoff, but no further transitions are possible at that contact
mode and state.
\end{proof}

Finally, the last consistency property considers the dynamics of the discrete modes.  
A general $C^r$ hybrid system whose domains are intersecting subsets of some ambient 
domain (as is true for the self-manipulation hybrid system) need not define a unique
execution for a given state from any appropriate mode, even if the hybrid system is 
deterministic and non-blocking. A given state $\biq \in \calQ$ in general 
is a member of more than one domain, such as the corner point in purple from
Figure~\ref{fig:hybrid} or, for an arbitrary point, any subset of the current constraint mode.
As such,
we must ensure that the evolution of the system is not biased by ascribing the 
``incorrect'' mode to that state, nor capable of supporting more than one word 
(discrete mode sequence) over a given sequence of continuous time trajectories 
associated with an execution.  

\begin{theorem}
\label{thm:uniqueex}
From an initial state $\biq$ at time $t_0$ and any contact mode $I$ consistent with that state, 
i.e.\ $I \in \calJ_\bfq:=\{I \in \calJ: \biq \in D_I\}$, the execution is uniquely defined (in both state
and contact mode) for all $t>t_0$ after undergoing up to one hybrid transition. 
\end{theorem}

\begin{proof}
The proof considers in turn the two mutually exclusive cases defined by the truth or 
falsity of the predicate $\NTD(\biq)$,~\eqref{eq:NTD}, and in each case the execution is
uniquely defined due to the uniqueness of the corresponding complementarity problem.

If $\NTD(\biq)$ is true then there is some additional constraint $j$ that is impacting, 
i.e.\ $\TD(j,\biq)$,~\eqref{eq:TDsimp}, where note that $\forall\; I \in \calJ_\bfq, j \notin I$. Therefore 
from any consistent mode $I\in \calJ_\bfq$, $\biq$ is in some guard $G_{I,J}$,~\eqref{eq:gdpiv},
 determined by $\PIV$ 
complementarity, $J=\CP_\PIV(\biq)$,~\eqref{eq:PIV}. Since the reset map depends on $J$ but not $I$, and 
$J$ is unique by the impulse--velocity complementarity assumption (\ref{ass:ivc}), the execution
continues from the unique point $(J,R_{I,J}(\biq))$.

If $\NTD(\biq)$ is false, then for any $I\in \calJ_\bfq$ the system could be in a liftoff
guard. Consider the mode $J=\CP_\FA(\biq)$, \eqref{eq:CPFA}, uniquely defined for a given
$\biq$ by the force--acceleration complementarity assumption (\ref{ass:fac}). If $I=J$ then the state is not in
a guard and therefore no reset map is applied.
Otherwise $\biq \in G_{I,J}$,~\eqref{eq:gdffa}, and the system undergoes a hybrid transition, 
though recall that liftoff reset maps are the identity map.
In either case the execution continues from the unique point $(J,\biq)$.
\end{proof} 

\subsection{Zeno}\label{sec:zeno}

An execution for a hybrid system is referred to as \emph{Zeno} if it undergoes an infinite number of discrete transitions in finite time \cite[Def.~II.3]{LygerosJohansson2003}.

\begin{definition}
\label{def:zeno}
An execution $\chi:T\into\e{D}$ for a hybrid dynamical system $\hds$ over a hybrid time trajectory $T = \mycoprod_{i=1}^N T_i$ is \emph{Zeno} if $N = \infty$ and $\sum_{i=1}^\infty \abs{T_i} < \infty$.
\end{definition}

Zeno executions need not accumulate in a general hybrid system, that is, the limit $\lim_{t\goesto\sup\e{T}}\chi(t)$ may be undefined \cite[Def.~6]{ZhangJohansson2001}.
However, Lagrangian systems subject to unilateral constraints give rise to unique trajectories defined for all time \cite[Thm.~10]{Ballard2000}.
We show in Section~\ref{sec:zenothms} that the self-manipulation hybrid dynamical system (Def.~\ref{def:smhs}) 
models this property through the mechanism of Zeno executions accumulating on a
unique limit in the ambient space $T\e{Q}$, from which the hybrid execution proceeds through the next 
smooth component (and so forward, continuously, through ambient time). 
Then in Section~\ref{sec:zenodisc} we discuss extensions and connections with results in the literature.

\subsubsection{Accumulation of Zeno executions}\label{sec:zenothms}
In the following Theorem, we rely on several results originally obtained using sophisticated measure-theoretic techniques \cite[]{Ballard2000}.
At the expense of additional notational overhead, we reproduce the necessary arguments in our hybrid system framework using elementary mathematical machinery.

\begin{theorem}\label{thm:zenolim}
Given a self-manipulation hybrid system with 
a complete connected configuration manifold $\e{Q}$,
if the inertia tensor $\Mbar$ is non-degenerate and
the forces abide by the bound in~\eqref{eq:effbddJ},
then the projection of any Zeno execution $\chi:\e{T}\into \e{D}$ into the ambient state space $T\e{Q}$, $\pi\circ\chi:\e{T}\into T\e{Q}$, accumulates on a unique limit,
\eqnn{\label{eq:zenolim}
(\bar{\bfq},\dot{\bar{\bfq}}) := \lim_{t\goesto\sup\e{T}} \pi\circ\chi(t).
}
\end{theorem}

\begin{proof}
Let $\chi:\e{T}\into \e{D}$ be a Zeno execution over the hybrid time trajectory $\e{T} = \mycoprod_{i=1}^\infty T_i$.
  With $T_i\cap T_{i+1} = \set{t_i}$ for all $i\in\Nat$, let $\obar{t} = \sup \e{T} < \infty$.  
For notational convenience in this proof we let $T_i = [t_{i-1}^+,t_{i}^-]$ for all $i\in\Nat$.  
This notation is somewhat redundant since $t_{i}^+ = t_{i}^- \in\Real$ for all $i\in\Nat$; we use it to signify that $t_i^-\in T_i$ and $t_i^+\in T_{i+1}$.
Note that $\lim_{i\goesto\infty} (t_{i}^- - t_{i-1}^+) = 0$ since $\lim_{i\goesto\infty} t_i = \obar{t}$.
When there should be no confusion as to the index of the time domain, we abuse notation by suppressing the
index and simply write $\dot{\bfq}(t)$ instead of $\dot{\bfq}(i,t)$ for $t\in [t_{i-1}^+, t_i^-]$.

Let $\pi:\e{D}\into T\e{Q}$ be the canonical projection that simply removes the label from the 
disjoint union $\e{D} = \mycoprod_{I\in\e{J}} D_I$, and let $(\bfq,\dot{\bfq}) := \pi\circ\chi$ denote the velocity and position components of the execution $\chi$.
Note that since the reset map,~\eqref{eq:reset}, does not change the position, $\bfq$, 
the position trajectory
$\bfq:\e{T}\into\e{Q}$
satisfies,
\eqnn{\label{eq:pos}
\bfq|_{T_{i}}(t_i) = \bfq(t_i^-) = \bfq(t_i^+) = \bfq|_{T_{i+1}}(t_i),
}
for all $i\in\Nat$, i.e.\ positions evolve continuously with respect to time.
Therefore $\bfq$ uniquely determines a continuous curve $\td{\bfq}:T\into \e{Q}$ over the half-open interval $T = \bigcup_{i=1}^\infty T_i = [0,\obar{t})$.
The restriction $\dot{\bfq}|_{T_i}$ is continuous for every $i\in\Nat$, therefore it uniquely determines a right-continuous curve $\paren{\td{\bfq},\dot{\td{\bfq}}^+}:T\into T\e{Q}$.
The bound in~\eqref{eq:effbddJ} ensures that the velocity is bounded on finite time horizons,
\eqnn{\label{eq:vbar}
\bar{v} :=\sup_{t\in T}\set{\absM{\dot{\td{\bfq}}^+(t)}} < \infty,
}
as shown using a sequence of standard results in Appendix~\ref{app:mech} (and adapted from the proof 
of \cite[Thm.~10]{Ballard2000}).
For any nondecreasing Cauchy sequence, $\set{s_i}_{i=1}^\infty\subset T$ such that $s_i\goesto\obar{t}$, the sequence $\set{\td{\bfq}(s_i)}_{i=1}^\infty$ is also Cauchy since,
\eqnn{ 
\forall\; n,m\in\Nat : d_{\Mbar}(\td{\bfq}(s_n),\td{\bfq}(s_m)) & \le \int_{s_m}^{s_n}\absM{\dot{\td{\bfq}}^+(s)} ds 
\le \bar{v}\abs{s_n - s_m}.
}
Therefore the position tends to a unique limit in the ambient state space, i.e.\ the following limit exists:
\eqn{\bar{\bfq} := \lim_{t\goesto\obar{t}} \td{\bfq}(t).}
Under the simple constraints assumption (\ref{ass:rank}), the Rank Theorem \cite[Thm.~4.12]{Lee2012} ensures there exists a coordinate chart $(V,\psi)$ near $\bar{\bfq}$ where $\bfA_{\e{K}} =[\Id, 0]$.
Continuity of $\td{\bfq}$ ensures there exists $\ubar{t} \in T$ for which $\td{\bfq}({[\ubar{t},\obar{t})})\subset V$. 

Let $\sigma(\chi) = \set{J_i}_{i=1}^\infty\subset\e{J}$ denote the sequence of discrete modes visited by $\chi$ and let $\ubar{m} := \min\set{i\in\Nat : t_i \ge \ubar{t}}$.
Specializing the definition of execution of a hybrid system to the self-manipulation system and 
performing integration-by-parts as in Appendix~\ref{app:parts} we conclude that in coordinates,
\eqnn{\label{eq:dq}
\forall\; i > 
\ubar{m}, t\in[t_{i-1}^+, t_{i}^-] :\, & 
\Mbar 
(\bfq(t))\dot{\bfq}(t) - \Mbar(\bfq(t_{i-1}^+))\dot{\bfq}(t_{i-1}^+) \\
& = \int_{t_{i-1}}^t\left(\vphantom{\int} \Upsilon(\bfq,\dot{\bfq}) - \Nbar(\bfq,\dot{\bfq})  - \Ctd(\bfq,\dot{\bfq}) \right. 
\left.\vphantom{\int} - \bfA_{J_i}(\bfq)^T \lambda_{J_i}(\bfq,\dot{\bfq}) \right) ds, \\
&\Mbar(\bfq(t_{i}^+))\dot{\bfq}(t_{i}^+) - \Mbar(\bfq(t_{i}^-))\dot{\bfq}(t_{i}^-) = -\bfP_{J_i},
}
where 
$\bfP_{J_i}$ is defined in~\eqref{eq:P_J}
and for each $\ell \in\set{1,\dots,\rmq}$ the $\ell$-th coordinate of the covector $\Ctd\in T^*\e{Q}$ is given by,
\eqn{
\Ctd^\ell(\bfq,\dot{\bfq}) := -\frac{1}{2}\sum_{j,k=1}^{\mathrm{q}} \vfof{\Mbar_{kj}(\bfq)}{\bfq^\ell} \dot{\bfq}^k \dot{\bfq}^j.
}
Recursively substituting using~\eqref{eq:dq} and~\eqref{eq:pos}, for any $t\in [t_{\ubar{m}},\obar{t})$ with $\obar{m} := \max\set{i\in\Nat : \ubar{t} \le t_i \le t}$, 
the velocity component of the execution 
$(\bfq,\dot{\bfq}):\calT\rightarrow T\calQ$
satisfies,
\eqnn{\label{eq:dqsum}
\Mbar&(\bfq(t))\dot{\bfq}(t) - \Mbar(\bfq(t_{\ubar{m}}))\dot{\bfq}(t_{\ubar{m}}) \\
& = \int_{t_{\obar{m}}}^t \left(\vphantom{\int} \Upsilon(\bfq,\dot{\bfq}) - \Nbar(\bfq,\dot{\bfq}) - \Ctd(\bfq,\dot{\bfq}) \right. 
\left.\vphantom{\int} - \bfA_{J_{\obar{m}}}(\bfq)^T \lambda_{J_{\obar{m}}}(\bfq,\dot{\bfq}) \right) ds \\
& \quad + \sum_{i=\ubar{m}}^{\obar{m}}\left[\int_{t_{i-1}}^{t_i}\left(\vphantom{\int} \Upsilon(\bfq,\dot{\bfq}) 
- \Nbar(\bfq,\dot{\bfq}) - \Ctd(\bfq,\dot{\bfq})  \right. \right. 
\left. \left. - \bfA_{J_i}(\bfq)^T \lambda_{J_i}(\bfq,\dot{\bfq})\vphantom{\int}\right) ds - \bfP_{J_i}\right].
}

Noting that for all time $t \in (t_{i-1}^+, t_i^-)$ on the interior of each time interval $i > \ubar{m}$ that
$\td{\bfq}(t) = \bfq(t)$ and $\dot{\td{\bfq}}^+(t) = \dot{\bfq}(t)$,
we conclude that for all $t\in [t_{\ubar{m}}, \obar{t})$ the right-continuous 
velocity $\dot{\td{\bfq}}^+:T\rightarrow T\calQ$ satisfies,
\eqnn{\label{eq:dq+}
\Mbar&(\td{\bfq}(t))\dot{\td{\bfq}}^+(t) - \Mbar(\td{\bfq}(t_{\ubar{m}}))\dot{\td{\bfq}}^+(t_{\ubar{m}}) \\
& = \int_{t_{\ubar{m}}}^t \Upsilon(\td{\bfq},\dot{\td{\bfq}}^+) - \Nbar(\td{\bfq},\dot{\td{\bfq}}^+) - \Ctd(\td{\bfq},\dot{\td{\bfq}}^+)  ds 
- \int_{t_{\obar{m}}}^{t} \bfA_{J_{\obar{m}}}(\td{\bfq})^T \lambda_{J_{\obar{m}}}(\td{\bfq},\dot{\td{\bfq}}^+) ds \\
& - \sum_{i=\ubar{m}}^{\obar{m}} \brak{\int_{t_{i-1}}^{t_i} \bfA_{J_i}(\td{\bfq})^T \lambda_{J_i}(\td{\bfq},\dot{\td{\bfq}}^+) ds + \bfP_{J_i} }.
}
This equation,~\eqref{eq:dq+}, is the transcription of \cite[Eqn.~36]{Ballard2000} into our formalism.

Recall that in coordinates $(V,\psi)$ we have $\bfA_{\e{K}} =[\Id, 0]$ and
that $\bfU(\bfP)\geq 0 \Rightarrow \bfP \leq 0$, $\bfU(\lambda)\succeq0 \Rightarrow \lambda \leq0$.
The complementarity conditions,~\eqref{eq:Ukla} and \eqref{eq:UkPa}, thus ensure that each component of $\bfA_{J}^T \lambda_{J}$ 
and $\bfP_{J}$ are non-positive for each~$J\in\e{J}$.

We conclude by rearranging~\eqref{eq:dq+} (and suppressing dependence on $\td{\bfq}$ and $\dot{\td{\bfq}}^+$)   
and invoking the bound from Appendix~\ref{app:eff} that there exists $\alpha,\beta\in\Real$ such that for each $j\in\set{1,\dots,\abs{\e{K}}}$,
\eqnn{
0 \le \sum_{i=\ubar{m}}^{\obar{m}} & \brak{\int_{t_{i-1}}^{t_i} \lambda_{J_i}^j ds + \bfP_{J_i}^j} + \int_{t_{\obar{m}}}^{t} \lambda_{J_m}^j ds 
=-\brak{\Mbar^j(t)\dot{\td{\bfq}}^+(t) - \Mbar^j(t_{\ubar{m}})\dot{\td{\bfq}}^+(t_{\ubar{m}})} 
 + \int_{t_{\ubar{m}}}^t \Upsilon^j - \Nbar^j - \Ctd^j ds \\
&
\leq \alpha + \beta(t - t_{\ubar{m}}). \label{eq:abbound}
}
Therefore the infinite sum,
\eqn{\sum_{i = \ubar{m}}^\infty \brak{\int_{t_{i-1}}^{t_i} \lambda_{J_i}^j ds + \bfP_{J_i}^j},} 
exists and is finite by the Monotone Sequence Theorem \cite[Thm.~1.16]{Folland2002}.
Thus each coordinate of each term in~\eqref{eq:dq+} tends to a unique limit as $t\goesto\obar{t}$, i.e.\ the following limit exists:
\eqn{\dot{\bar{\bfq}} := \lim_{t\goesto\obar{t}} \dot{\td{\bfq}}^+(t).}
\end{proof}

Let $\e{Z}\subset\e{J}$ denote the set of modes visited infinitely often by $\chi$.  
Since the sequence $\sigma(\chi) = \set{J_n}_{n=1}^\infty\subset\e{J}$ of discrete modes visited by $\chi$ is an infinite sequence of elements taken from a finite set, $\e{Z}\ne\emptyset$.

Although the previous result guarantees the continuous state of the system associated with a Zeno execution 
has a well defined limit, we must further guarantee that it limits on a consistent mode as well. The 
following result guarantees that this limiting state is indeed achieved in a well defined mode which 
is also physically meaningful in the sense of being composed of any and all of the constraints that had 
been active infinitely often during the Zeno~execution.

\begin{corollary}\label{cor:zenolim}
Let $\bar{Z} = \bigcup\e{Z}$ denote the set of all constraints visited during a 
Zeno exection $\chi:T\into D$. The set is a valid mode,
representing the asymptotic contact mode, and 
the zeno limit $(\bar{\bfq},\dot{\bar{\bfq}})$ from~\eqref{eq:zenolim} lies in the domain $D_{\bar{Z}}$,~i.e., 
\eqn{
\bar{Z} \in \calJ, \qquad \bar{\chi} := (\bar{\bfq},\dot{\bar{\bfq}})\in D_{\bar{Z}}.}
\end{corollary}

\begin{proof}
We continue with the notational conventions from the Proof of Theorem~\ref{thm:zenolim}.
For all $Z\in\e{Z}$, let $\bfW_Z = \bfA_Z^T(\bfA_Z \Minv \bfA_Z^T)^{-1} \bfA_Z$ and note:
$\bfW_Z = \bfW_Z^T \ge 0$; 
$\bfW_Z \Minv \bfW{_Z} = \bfW_Z$;
$\dot{\bfq}(t_i^+) = \Id - \Minv\bfW_{J_i}) \dot{\bfq}(t_i^-)$ for all $i\in\Nat$; and
$\exists\;\bfS_Z$ such that $(\bfA_Z \Minv \bfA_Z^T)^{-1} = \bfS_Z^T \bfS_Z$.
Impacts do not increase energy since for all $i\in\Nat$:
\eqn{
\frac{1}{2}\dot{\bfq}(t_i^-)^T \Mbar \dot{\bfq}(t_i^-) - \dot{\bfq}(t_i^+)^T \Mbar \dot{\bfq}(t_i^+) 
= & \frac{1}{2}\dot{\bfq}(t_i^-)^T \Mbar \dot{\bfq}(t_i^-) 
- \frac{1}{2}\dot{\bfq}(t_i^-)^T(\Id - \bfW_{J_i} \Minv) \Mbar (\Id - \Minv \bfW_{J_i})\dot{\bfq}(t_i^-) \\
= & \frac{1}{2}\dot{\bfq}(t_i^-)^T ( \Mbar - \Mbar + \bfW_{J_i} + \bfW_{J_i} - \bfW_{J_i} \Minv \bfW_{J_i}) \dot{\bfq}(t_i^-) \\
= & \frac{1}{2}\dot{\bfq}(t_i^-)^T ( 2 \bfW_{J_i} - \bfW_{J_i} ) \dot{\bfq}(t_i^-) 
= 
\frac{1}{2}\dot{\bfq}(t_i^-)^T \bfW_{J_i} \dot{\bfq}(t_i^-) \ge 0.
}
Equation~\eqref{eq:vbar} implies impacts must extract a finite amount of energy,
\eqnn{\label{eq:sumWJi}
\sum_{i=1}^\infty \dot{\bfq}(t_i^-)^T \bfW_{J_i} \dot{\bfq}(t_i^-) < \infty.
}
and hence in particular,
\eqnn{\label{eq:limWJi}
\lim_{i\goesto\infty} \dot{\bfq}(t_i^-)^T \bfW_{J_i} \dot{\bfq}(t_i^-) = 0.
}
Taking~\eqref{eq:limWJi} together with,
\eqn{\dot{\bfq}^T \bfW_Z\dot{\bfq} = \dot{\bfq}^T \bfA_Z^T(\bfA_Z \Minv \bfA_Z^T)^{-1} \bfA_Z \dot{\bfq} = \absM{\bfS_Z \bfA_Z \dot{\bfq}}^2,}
implies each Zeno constraint is asymptotically satisfied:
\eqnn{\label{eq:limAz}
\forall\; z\in\bar{Z} = \bigcup\e{Z} : \lim_{t\goesto\obar{t}} \bfA_{z} \dot{\bfq}(t) = 0.
}
For all constraints $j\in \bar{Z}, j \in Z_j$ for at least one $Z_j \in \e{Z}$ visited
infinitely often in $\chi$. By the definition of $D_{Z_j}$,
\eqnn{\label{eq:limaz}
\forall\; z\in\bar{Z}: \lim_{t\goesto\obar{t}} \bfa_{z}(\bfq(t)) = 0.
}
and therefore $\bar{\chi} \in \bfa_{\bar{Z}}^{-1} \neq \varnothing$. Furthermore
for all constraints $j\in \bar{Z}$ and all modes $Z_j \in \e{Z}$ containing $j$,
the definition of mode $Z_j$,~\eqref{eq:jdef}, requires that $\alpha(j) \in Z_j \subset \bar{Z}$.
Therefore $\bar{Z} \in \calJ$ by~\eqref{eq:jdef}.
The domain $D_{\bar{Z}}$ has three requirements,~\eqref{eq:Ddef}, two of which we have already
shown to be met by $\bar{\chi}$ in~\eqref{eq:limAz} \&~\eqref{eq:limaz}. Finally,
as in~\eqref{eq:limaz}, for all constraints $k \in \KN$, and all modes $Z \in \e{Z}$,
by definition of $D_{Z}$,
\eqnn{\label{eq:limap}
\forall\; k\in\KN: \lim_{t\goesto\obar{t}} \bfa_{k}(\bfq(t)) \geq 0.
}
Thus $\bar{\chi} = (\bar{\bfq},\dot{\bar{\bfq}})\in D_{\bar{Z}}$.
\end{proof}

\subsubsection{Effect of pseudo-impulse on Zeno executions}\label{sec:pseudozeno}
As suggested in Section~\ref{sec:pseudoimpulse}, the inclusion of the pseudo-impulse prevents an infinite number of liftoff transitions in a finite amount of time from constraints impinged upon by external forces.

\begin{theorem}\label{thm:zeno}
Let $\chi:\e{T}\into \e{D}$ be a Zeno execution of a self-manipulation hybrid dynamical system with exactly two contact constraints, so that the limiting set $\bar{Z} = \e{J}$.
Under the hypotheses and notation of Theorems~\ref{thm:zenolim} \& Corollary~\ref{cor:zenolim}, 
when the pseudo-impulse parameter is positive, $\delta_t > 0$, 
we conclude that,
\eqnn{\label{eq:zeno}
\forall\; z\in\bar{Z} : \bfU_z\paren{\bfAd_{\bar{Z}}(\bar{\bfq}) \left(\vphantom{\int}\Upsilon(\bar{\bfq},\dot{\bar{\bfq}}) 
- \Cbar(\bar{\bfq},\dot{\bar{\bfq}}) - \Nbar(\bar{\bfq},\dot{\bar{\bfq}}) \right)} \leq 0,
}
that is, 
the constraint forces cannot be positive for either constraint at the Zeno limit point.
\end{theorem}

\begin{proof}
We know $\lim_{t\goesto\obar{t}}\bfA_{z} \dot{\bfq}(t) = 0$ for all $z\in\bar{Z}$ by~\eqref{eq:limAz}. 
When the liftoff velocity drops below the threshold given implicitly by Theorem~\ref{thm:pseudoimp}, the 
pseudo-impulse prevents liftoff from constraint $i\in\bar{Z}$ if it violates~\eqref{eq:zeno} (i.e., meets
the condition~\eqref{eq:pimpcond}).
Therefore, the contact force must be negative for both constraints $z\in\bar{Z}$ that undergoes an infinite number of liftoff transitions.
\end{proof}

\section{Discussion}
\label{sec:discussion}

In this section we \changed{discuss} the limitations in physical scope incurred by the twelve assumptions of 
Section~\ref{sec:imp}, \changed{implementation details for numerical simulation,}
and some consequences of the results of this paper through additional examples.
Specifically, Section~\ref{sec:dis:ass} reviews on a conceptual level the meaning and implications of the original 
assumptions. Then Sections~\ref{sec:dis:ml} \&~\ref{sec:dis:pseudo} use a number of specific physical 
examples to elucidate the nature and origin of the more conservative restrictions that are helpful but, 
we speculate, not necessary to achieve the results of interest in Section~\ref{sec:hyb}. 
\changed{Section~\ref{sec:dis:numerical} briefly discusses the issues involved with numerically simulating
executions of a $C^r$ hybrid dynamical system.}
Finally, Section~\ref{sec:zenodisc} explores the relationship of these assumptions to the Zeno results
of Section~\ref{sec:zeno}.

\subsection{The Base Assumptions} 
\label{sec:dis:ass}

Most of the assumptions listed in Section~\ref{sec:imp} are quite common in the modeling of physical systems,
especially for models focused on robotics. These limit the scope of the physical settings of interest,
and while there are certainly examples of robots that would be poorly modeled by each (many of which are
explored in Section~\ref{sec:lit}), we believe that
there remains a large class of systems that are covered by most if not all of these. Specifically, this class at
its core consists of rigid bodies (Assumption~\ref{ass:rigid}) under Lagrangian dynamics (\ref{ass:complete})
with analytic (\ref{ass:analytic}), independent (\ref{ass:rank}) constraints.
That these constraints persist (\ref{ass:contact}) and are added through plastic impact (\ref{ass:plastic})
are certainly domain specific assumptions, but many robotic tasks involve touching an object or the environment
with the goal of continuing that contact in order to do some work. 

With these assumptions in place, solving a complementarity problem (\ref{ass:fac} \&~\ref{ass:ivc})
is the most direct and mathematically tractable way to formulate the change in contact conditions and
is in step with a large literature (reviewed in Section \ref{sec:lit:comp}).
Our insistence on unique solutions to the these problems, key to the consistency conclusions of
Section~\ref{sec:consistency}, has poorly understood consequences except for the case of independent plastic 
frictionless contacts for which these assumptions are known to hold. The unique structure of the complementarity
problems used here allows for the inclusion of other assumptions (in particular massless limbs and the pseudo-impulse, \ref{ass:uml} and \ref{ass:pseudo}),
and Theorems~\ref{thm:fac} \&~\ref{thm:ivc} (along with Lemma~\ref{thm:comp}) ensures that this form
of the complementarity assumptions agrees with the more common versions. 

Assumption~\ref{ass:complete} is imposed both in the interest of the physical scope (Lagrangian dynamics) 
and mathematical consistency. It arises from the same motivation as the familiar conditions that 
preclude finite escape in classical dynamical systems but must nevertheless be couched in more 
technically involved language because of the hybrid setting. 
Thus we have found it expedient to provide further analysis of what is left behind: the results of 
Lemmas~\ref{lem:lip} and~\ref{lem:lip:ml} guarantee the admissibility of most physically interesting problem instances, 
but bar (for reasons reflecting the need for a more technical framework, we suspect, rather than mathematical necessity) 
only the case of nonholonomically constrained massless links.

The remaining assumptions are not imposed to facilitate the definitions and consistency proofs underlying 
the formal hybrid system (Section~\ref{sec:hyb}), but, rather, relate to the practicality of the physical
models they can support. We have found the following simplifying (and, strictly speaking, physical fidelity 
diminishing) assumptions critical to not merely the mathematical tractability but also the qualitative accuracy 
of the models we use in the robotic settings of interest (as exemplified by the illustrative cases explored below).  
Thus, the formal results of this paper have been adapted wherever possible to allow for their consistent inclusion.
In particular the massless leg assumptions, \ref{ass:cml} and~\ref{ass:uml}, are sometimes made for
mathematical tractability, but often are not analyzed carefully. 
Briefly, \ref{ass:cml} is tantamount to the assertion that the Lagrange D'Alembert formulation of 
constrained mechanics should admit smooth generalized coordinates relative to which the kinetic energy 
is nonsingular, while Assumption \ref{ass:uml} similarly requires that any massless degrees of freedom not in contact 
be assigned some reasonable dynamics. Lemmas~\ref{thm:mlequiv}, \ref{thm:dynamics}, \&~\ref{thm:impulse} and
Theorems~\ref{thm:fac} \&~\ref{thm:ivc} all concern the inclusion of massless limbs with the other assumptions,
and the theorems of Section~\ref{sec:consistency} ensure that the resulting system is consistent.

Friction in various forms is a common modeling assumption, however the specific setup in Assumption~\ref{ass:friction}
(which divides contacts into either completely slipping or completely sticking but precludes 
sliding-to-sticking transitions) consists of common components but in a very restricted manner. 
This particular combination is evidently not the best model of friction for many systems.
The existence and uniqueness of solutions to the corresponding complementarity problems for this setup
has never been demonstrated and therefore is simply assumed in this paper. 

Finally, the practical and theoretical implications of the new pseudo-impulse model assumed 
in~\ref{ass:pseudo} have a complicated interplay to whose exploration we devote Section~\ref{sec:dis:pseudo}.

\subsection{Massless Limbs}
\label{sec:dis:ml}

One common set of circumstances that satisfy the requirements of the massless limbs assumption (\ref{ass:cml}) arises when only the robot's most distal link (the finger,
lower leg, foot, or in the case of RHex, the entire leg) is massless and the motion of its most distal point is 
completely constrained when it is on the ground. Although the rank requirement is not limited to
this setting, it represents the immediate motivation for our inquiry.

Though there are no truly massless limbs, computing the dynamics using \eqref{eq:dyn}--\eqref{eq:ldyn} 
is numerically more stable than inverting $\Mbar_\epsilon$ in the presence of large disparities in limb segment 
masses \cite[Sec.~4.3]{Holmes_Full_Koditschek_Guckenheimer_2006}.
This is evidenced by an order of magnitude improvement in 
the condition number (ratio of largest to smallest singular values) for the RHex model used 
here \cite[Sec.~5.1.1]{thesis:johnson-2014}.

\subsection{Pseudo-Impulse}
\label{sec:dis:pseudo}

\begin{figure}
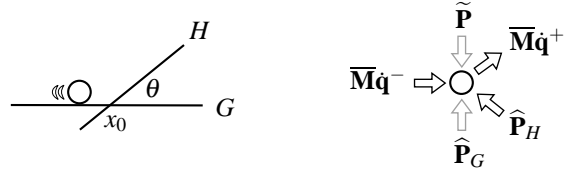

\centering
\def\svgwidth{8.5cm}
\include{slidingpt}
\caption{\emph{Left:} A point sliding along ground $G$ approaches hill $H$. \emph{Right:} Free body diagram showing impulses 
at point of contact. Without $\dimp$ no positive impulse from the ground $\limp_G$ is possible for 
any initial momentum $\Mbar\dot{\bfq}^-$ and any hill slope $\theta<90^\circ$.}
\label{fig:slidingpt}
\end{figure}

As a simple example that motivates the need for the pseudo-impulse (Assumption~\ref{ass:pseudo}), consider a point sliding on the ground as in Figure~\ref{fig:slidingpt}, which
hits a hill at some slope $\theta$. The contact impulse from the hill $\limp_H$ causes the 
particle to break contact with the ground and leave with some velocity sliding up the hill. This is true
for any initial velocity, no matter how small (Lemma~\ref{lem:pseudoimp}), and any $\theta <90^\circ$. With a pseudo-impulse
$\dimp$ acting in the direction of gravity,~\eqref{eq:pseudoimp}, Theorem~\ref{thm:pseudoimp} states that there are initial conditions that result in
the point coming to rest with impulses from both the ground and the hill (i.e., all impulses are positive
and sum to zero in Figure~\ref{fig:slidingpt}). Note that in this case all quantities scale linearly
with mass and as such the solution is the same for any mass.

\begin{figure}
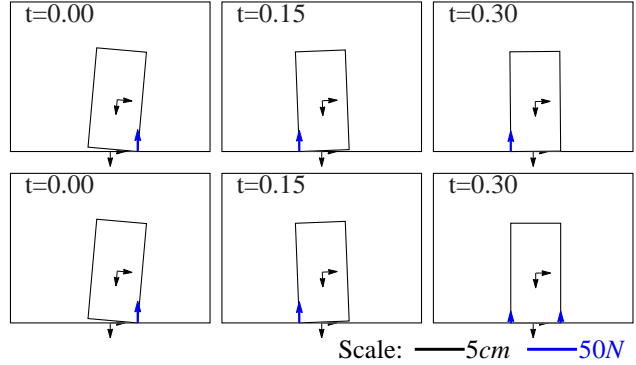

\centering
\def\svgwidth{8.3cm}
\include{block_ex}
  \vspace{8pt}
\caption{A rocking block (height $h=10$cm, width $w=5$cm, mass $m=5$kg) settling on the ground.
 \emph{Top Row:} Without pseudo-impulse ($\delta_t=0$). \emph{Bottom Row:} With pseudo-impulse ($\delta_t = 0.03$).  
The execution is identical until the last frame. }
\label{fig:blockex}
\end{figure}

To see how the pseudo-impulse resolves Zeno executions, consider the ``rocking block'' example of a rectangular 
rigid body in Figure~\ref{fig:blockex} of width $w$, height $h$, mass $m_b$, and inertia $I_b$
(where if a uniform distribution is assumed $I_b=m_b(w^2+h^2)/12$), as studied in e.g.\ \cite{housner1963,McGeer_Wobbling_1989,LygerosJohansson2003,yilmaz2009solving}. 
As it is falling onto the ground if a corner (labeled ``l'') is touching down\footnote{ In
this example the contact points are assumed to resist sliding friction, although when they are both in
contact with the ground one of the redundant tangential constraints is dropped. The phenomenon of interest occurs
equally well with frictionless contact however the analysis is simpler in the frictional case as presented here.} then
the normal direction impulse at that corner when the other corner (labeled ``r'') hits the ground is,
\begin{align}
\bfU_{l}(\limp) &= \frac{\dot{z}(2I_b+m_b(w^2 - h^2)/2)}{w^2},
\end{align}
(note that by convention a positive velocity $\dot{z}$ is one that is towards the ground)
and the required impulse is negative~if,
\begin{align}
h^2> w^2 + \frac{4I_b}{m_b} \Rightarrow \bfU_{l}(\limp) <0,
\end{align}
in which case the contact is broken no matter how slow the block is moving -- this is exactly what
Lemma~\ref{lem:pseudoimp} predicts.
The system exhibits Zeno behavior requiring infinite transitions in finite time 
as each impact removes some energy but does not immobilize the block, as \changed{shown} in the 
upper row of Figure~\ref{fig:blockdbz} which plots the vertical velocity as the system undergoes a Zeno execution.

Instead if the pseudo-impulse is considered,
\begin{align}
\bfU_{l}(\limp+\dimp) &= \frac{\dot{z}(2I_b+m_b(w^2 - h^2)/2)}{w^2} + \frac{ \delta_t m_b g}{2}
\end{align}
the contact is broken if,
\begin{align}
  h^2 &>   w^2 + \frac{4I_b}{m_b} +  \frac{\delta_t g w^2}{\dot{z}} \Rightarrow \bfU_{l}(\limp+\dimp)<0,
\end{align}
where as the speed goes to zero ($\dot{z}\rightarrow 0$) the threshold on height that  
allows the contact to persist grows and eventually is met -- this is exactly the case considered in 
Theorem~\ref{thm:pseudoimp}.  This truncation of 
the Zeno execution \changed{shown} in the lower row of Figure~\ref{fig:blockdbz}, where for
the dimensions used the block comes to rest if the vertical speed at impact is
less than $6.3cm/s$.

\begin{figure}
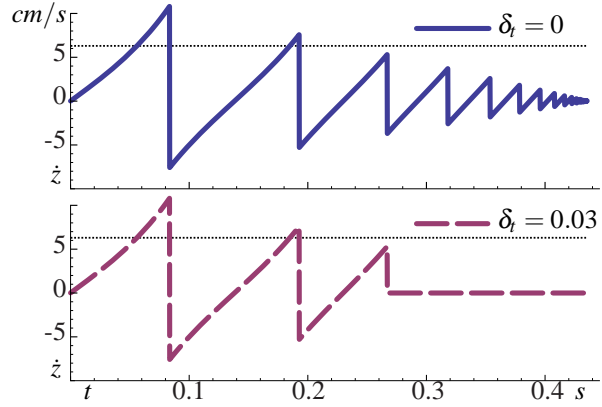

\centering
\def\svgwidth{7.5cm}
\include{block_dbz}
\caption{Comparison of the vertical velocity of a settling block for evaluations
with and without the pseudo-impulse. 
The execution is identical until the impact at $t=0.27$s. The pseudo-impulse implicitly
bounds the vertical velocity such that an impact at speeds lower than $6.3cm/s$ 
causes the block to come to rest, as indicated by the dotted line.}
\label{fig:blockdbz}
\end{figure}

Finally, note that the (somewhat restrictive) result in Theorem~\ref{thm:zeno} applies exactly to this rocking block example.
The word associated with the Zeno execution of interest alternates between the left and right constraints being active,~i.e.,
\eqnn{
\sigma(\chi) = \set{\set{l},\set{r},\set{l},\set{r},\dots}.
}
Since the (gravitational) force violates~\eqref{eq:zeno}, any value of the pseudo-impulse parameter 
$\delta_t > 0$ prevents an infinite number of liftoff (and hence touchdown) transitions for either constraint.
We conclude in this case that inclusion of the pseudo-impulse has the effect of \emph{truncating} the 
Zeno execution, i.e.\ preventing an infinite number of discrete transitions in finite time, as explored
further in \changed{Section~\ref{sec:zenodisc}}.

\changed{
\subsection{Numerical Simulation}
\label{sec:dis:numerical}

As noted in the Introduction, this paper is focused 
on constructing from the rigid body dynamics assumptions of Section~\ref{sec:imp} a hybrid
dynamical systems model in Section~\ref{sec:hyb} that has formally established, useful mathematical properties. 
However, as is generally true of classical ODE models in engineering applications, it often 
lends additional insight to approximate an execution of this hybrid system through numerical simulation.
For example, simulations of this model are used to generate Figures~\ref{fig:leapsim}, 
\ref{fig:pseudoimp}, \ref{fig:lam1ncomp}, \ref{fig:vertsim}, \ref{fig:blockex}, 
\ref{fig:blockdbz}, and \ref{fig:ledgesim}. A full exploration of the issues involved
in numerical simulation lies far outside the scope of the present paper.
In this section we briefly describe
some of the most critical details needed to simulate executions of our self-manipulation hybrid dynamical system.

The simulations presented in this paper were implemented in 
Mathematica\footnote{\url{http://www.wolfram.com/mathematica/}} using a conventional event-driven scheme%
\footnote{\changed{
Proposed originally in~\cite{Witsenhausen1966} and subsequently popularized by~\cite{ShampineGladwell1991,Back_Guckenheimer_Myers_1993},
this simulation algorithm for hybrid systems was proven to converge to \emph{orbitally stable} trajectories that encounter
(i) isolated transitions by~\cite{Tavernini1987}
and
(ii) simultaneous transitions by~\cite{burden2013metrization}. 
We refer the reader to~\cite{burden2013metrization} as a comprehensive reference for the definition of orbitally stable trajectories, implementation details for the simulation algorithm, and proofs of convergence of simulations to executions.
}}.
An execution of the self-manipulation hybrid dynamical system,
Definition~\ref{def:ex}, is calculated for each interval of the hybrid time domain, 
Definition~\ref{def:time}, sequentially. The flow in that particular contact mode, $F_I$, \eqref{eq:smflow}, is integrated
from the state at the initial time as an ordinary or differential-algebraic equation~\cite{AscherPetzold1998}, as appropriate, using the {\tt NDSolve} command with default parameters. 
The integration is stopped when the state reaches any guard, $G_{I}$, \eqref{eq:sm:outlet}, 
as detected by the {\tt WhenEvent} command. Separate events are used to test the new touchdown predicate,
$\NTD$, \eqref{eq:NTD}, and liftoff predicate, $\LO$, \eqref{eq:LO}, which make up the guard
(described in further detail below).
When an event is detected the integration stops and the subsequent contact mode is determined
by evaluating the appropriate complementarity problem according to~\eqref{eq:gdJcalc}. Using
this combined guard set requires fewer conditions than testing for each guard separately, however
the resulting executions are identical. Finally the reset map, $R_{I,J}$,~\eqref{eq:reset}, is applied
and the numerical integration continues again in the next time interval in contact mode $J$. 

The numerical implementation of the two event predicate detections can be simplified from
the full definition used in the above proofs. Specifically, the touchdown predicate
may be checked with an event $\bfa_k(\bfq) < 0$, which given the domain constraints,
\eqref{eq:Ddef}, becomes true at the moment the conditions of \eqref{eq:TDsimp} hold.
Similarly, by Lemma~\ref{lem:trend}, the liftoff predicate (which includes a trending condition, $\prec$)
may be checked with an event
$\bfU_k(\lambda_I)<0$, which becomes true at the moment the conditions of \eqref{eq:LO} hold.
Numerically determining these event times and states involves integrating beyond the event
(since the relevant quantities can be formally extended outside the domain $\calD$) 
and then stepping back to approximate the zero crossing to some desired precision using a numerical root-finding algorithm%
\footnote{\changed{
Although there exist pathological cases wherein this scheme determines event times inaccurately (or fail to detect events entirely), so long as the desired execution satisfies a mild \emph{orbital stability} property this scheme is guaranteed to succeed~\cite[Theorem~27]{burden2013metrization}.
}}. 

Once an event is detected, the complementarity problem, either $\CP_\PIV$ or $\CP_\FA$, 
can be solved in many ways, see e.g.\ \cite{cottle_pivot_1968}, \cite{pang1996complementarity}. For
small systems of contacts, a simple but inefficient method is to simply check the truth value
of the predicate, either $\PIV$, \eqref{eq:PIV}, or $\FA$, \eqref{eq:CPFA}, for all possible subsequent 
modes, $J \subseteq \calI$, and choose the unique mode where the predicate is true. 

The scope, $\calI$, \eqref{eq:IScope}, of these complementarity problems can also be simplified, as noted in 
Section~\ref{sec:complementarity}. For liftoff guards, in general it suffices to check simply
the active constraints, $I\subseteq\calI$. Any constraints that are not in $I$ and therefore
not algebraically guaranteed to satisfy the equality condition in \eqref{eq:IScope} will, 
due to numerical error, only be close. 
Using $I$ instead of $\calI$ will miss cases such as Figure~\ref{fig:ptex} (d), wherein
a constraint $j$ that is not in the current active mode ($j \notin I$) satisfies the 
scope in \eqref{eq:IScope} ($j \in \calI$). However these examples are not generic as any perturbation
in the state or constraint resolves this problem -- they make up a set of
measure zero which we do not expect can be reached numerically. Similarly the scope for touchdown
guards may be taken as the active set plus any impacting constraints, as stated in~\eqref{eq:IVScope2}.

}

\subsection{Zeno convergence results}\label{sec:zenodisc}
It is instructive to contrast Theorems~\ref{thm:zenolim} and~\ref{thm:zeno} and Corollary~\ref{cor:zenolim} with the \emph{completion} results in \cite{OrAmes2011}.
For Lagrangian systems subject to plastic impact with a single unilateral constraint,
the \emph{completion} of the \emph{simple Lagrangian hybrid system} in \cite{OrAmes2011} coincides with
the \emph{self-manipulation hybrid system} we develop in Section~\ref{sec:smsystem}.
Specifically, the completion is a hybrid dynamical system (in the sense of Definition~\ref{def:hs} with one \emph{constrained} and one \emph{unconstrained} mode.
Transition from the unconstrained to the constrained mode occurs at impact; outward-trending forces trigger the transition back to the unconstrained mode.
Our self-manipulation system can therefore be viewed as a generalization of the completion to Lagrangian systems undergoing plastic impact with an arbitrary number of unilateral constraints, a situation not considered in \cite{OrAmes2011}.
In connection with the numerical simulation literature discussed in Section~\ref{sec:sim},
we note that~\cite[Lem.~12]{stewart1998convergence} provides an accumulation result for time-stepping algorithms that is analogous to our Corollary~\ref{cor:zenolim}.
 
We further clarify the relationship between our contributions and the results in \cite{OrAmes2011} in the case of purely inelastic (i.e., plastic) impact.
Although the (un-completed) simple Lagrangian hybrid system allows plastic impacts (i.e., a \emph{coefficient of restitution} $e = 0$ in \cite[Eqn.~4]{OrAmes2011}), the definition of the guard in \cite[Sec.~II-A-3]{OrAmes2011} implies that every plastic impact is a Zeno event -- every $(\bfq,\dot{\bfq})\in T\e{Q}$ for which $\bfa(\bfq) = 0$ and $\bfA(\bfq)\dot{\bfq} = 0$ is a fixed point of the reset map in \cite[Eqn.~4]{OrAmes2011}.
This stands in contrast to our guard definition~\eqref{eq:gdpiv}--\eqref{eq:gdffa}, where we have excised such points from the domains of the reset maps.
In plain language, we ensure that constraints may persist after an impact without instantaneously triggering Zeno events.
As an illustration, consider just a single impact event in the \emph{rocking block} of Figure~\ref{fig:blockex}.
In the simple Lagrangian hybrid system of \cite{OrAmes2011}, plastic impact at time $t\in\Real$ results in a 
Zeno execution over a hybrid time trajectory $\e{T} = \mycoprod_{i=1}^\infty \set{t}$ that spans zero 
(continuous) time, thus conflicting with Theorem~\ref{thm:dbltrans} (before, possibly, completion and
continuing exection to the Zeno execution considered in Section~\ref{sec:dis:pseudo}).
In our self-manipulation hybrid system, the execution continues past this impact as illustrated in Figure~\ref{fig:blockdbz} 
(\emph{top}) by transitioning to a constrained mode. 
We note that the behavior of our system (and, equivalently, the completion from \cite{OrAmes2011}) is consistent with the analysis of the rocking block in \cite[Sec.~2]{housner1963}.

We also comment on the relationship between the truncation effect introduced by our \emph{pseudo-impulse} and the \emph{reliable truncation} proposed in \cite[Def.~6]{OrAmes2011}.
The pseudo-impulse we proposed in Section~\ref{sec:pseudoimpulse} prevents an infinite number of isolated liftoffs in finite time from pairs of constraints impinged upon by the external forces; this is the content of Theorem~\ref{thm:zeno}. 
In \cite[Def.~6]{OrAmes2011}, \emph{reliable truncation conditions} were shown to yield simulated executions that approximate a Zeno execution to specified precision; this is the content of \cite[Thm.~3]{OrAmes2011}.
Thus our contribution is a phenomenological heuristic that augments the hybrid system to prevent some Zeno executions from arising (specifically, those Zeno executions that only involve two constraints impinged upon by external forces).
The contribution in \cite[Sec.~V]{OrAmes2011} is a formal guarantee of simulation accuracy for Zeno executions in the original hybrid system. 
We have yet to determine the ``reliability'' of our truncation in this sense (though, as noted in the discussion following Theorem~\ref{thm:zeno}, the psuedo-impulse truncation is reliable in this sense for the rocking block).

It is possible to relax the hypotheses in Theorem~\ref{thm:zenolim} in several ways that ensure the results in Section~\ref{sec:zenothms} still hold.
It is straightforward to allow time-dependent forcing \citep[as in][Thm.~10]{Ballard2000} so long as the applied and potential forces obey the estimate,
\eqn{\label{eq:effbddt}
\forall\; (\bfq,\dot{\bfq})\in T\e{Q}:&\absMinv{\Upsilon(t,\bfq,\dot{\bfq}) - \Nbar(t,\bfq,\dot{\bfq})} 
\le \ell(t)(1 + \absM{\dot{\bfq}} + d_{\Mbar}(\bfq_0,\bfq)),
}
where $\ell:\Real\into\Real$ is nonnegative and locally integrable.
Fully-actuated massless limbs can be included by constraining their motion with respect to the body 
degrees-of-freedom, e.g.\ through the use of ``mirror laws'' \cite[]{BuehlerKoditschek1994}, so long 
as the forces required to enforce the desired motion obeys the estimate in~\eqref{eq:effbddt}.
Care must be taken to allow the forcing to depend on the contact mode, since it is possible to introduce ``sliding modes'' wherein limbs cycle infinitely often between constrained and unconstrained modes at a single time instant; we discuss this issue further in Section~\ref{sec:forcingdisc}.

Finally, as we have not yet been able to construct an example wherein a constraint that meets~\eqref{eq:zeno}
is involved in a Zeno execution, we speculate that the pseudo-impulse truncates a larger class of 
Zeno executions than handled by Theorem~\ref{thm:zeno}. Such an extension would require a careful treatment
of the interaction between the complementarity conditions and the Zeno execution, resulting in either a proof
that~\eqref{eq:zeno} can not hold for a constraint undergoing Zeno or conditions for Zeno executions involving
more than two contact constraints.

\subsection{Contact-dependent forcing}\label{sec:forcingdisc}
The developments in Sections~\ref{sec:imp} and~\ref{sec:hyb} allow, in principle, for the applied forces $\Upsilon_I$ to depend on the set of active constraints $I\in\e{J}$.
This is a desirable feature of our formalism since many extant robots sense their contact state with 
the world and accordingly alter their actuator commands. This also enables the separate handling of massless
limbs that make or break contact with the ground, Assumptions~\ref{ass:cml} and \ref{ass:uml}.
The hybrid system formalism provides a direct route to incorporate this sort of feedback.
Indeed, so long as one can ensure that the complementarity assumptions hold,~\ref{ass:fac} and~\ref{ass:ivc}, then the self-manipulation system (Def.~\ref{def:smhs}) has disjoint guards (Thm.~\ref{thm:disjoint}) and hence is deterministic (Thm.~\ref{thm:det}) and non-blocking (Thm.~\ref{thm:nb}).
Note, however, that the complementarity problems do not depend on contact mode, and so care must be used to enable
contact-dependent forcing that does not break these assumptions, otherwise the execution may alternate
between two adjacent modes.

\section{Conclusion}
\label{sec:Conclusion}

\begin{figure}
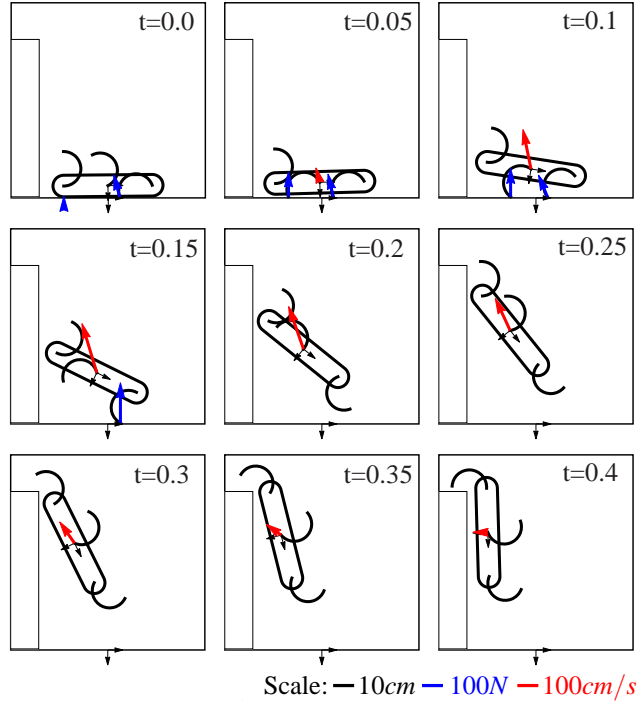

\centering
\def\svgwidth{8.3cm}
\include{ledge}
  \vspace{-8pt}
\caption{Keyframes from RHex simulation leaping onto a 73cm ledge. Blue arrows show contact forces while
the red arrow shows body velocity. The coefficient of friction is $\mu=0.8$ and the relative leg
timing is $t_2=0.06s$.
}
\label{fig:ledgesim}
\end{figure}

The hybrid system model presented here provides for the consistent inclusion of many
common simplifying physical assumptions, including rigid bodies and plastic impacts, as well as some less
common assumptions, such as the pseudo-impulse. These assumptions are well understood to be 
only approximations to the real physics: our central contribution is to develop sufficiently compatible 
refinements of previously investigated versions as to obviate their erstwhile conflicts.
Nevertheless, this refined model is still able to capture qualitatively many behaviors of interest in robotics -- not 
merely the familiar steady state tasks \citep[e.g.][]{BuehlerKoditschek1994,Holmes_Full_Koditschek_Guckenheimer_2006} but also transitional maneuvers
such as, archetypally, the leap onto a ledge shown in Figure~\ref{fig:ledgesim} (a behavior first demonstrated
in \cite{paper:johnson-icra-2013}). Simulation results such as these suggest the descriptive power of
our refined collection of physical assumptions, while the consistency properties of 
Section~\ref{sec:consistency} ensure that they avoid these potential conflicts. However, 
as noted in the text, there remain a few cases
where the formal proofs included here are limited to a still further constrained subset of mechanical 
settings than admitted by these assumptions, most notably 
Theorems~\ref{thm:pseudoimp}, \ref{thm:dbltrans}, \ref{thm:zenolim}, \&~\ref{thm:zeno}.
As noted at several points throughout the text, we believe that the conclusions remain true under 
the broader conditions (i.e., those listed in the assumptions themselves), but more general 
proofs of these properties remain an open research~question. 

Including such explicit assumptions makes it clear that certain extensions of this model satisfying these
assumptions are trivially admissible, such as more complicated rolling contact 
conditions, \cite[Sec. 5.2.1]{book:mls-1994}, over rough (though still semi-analytic) terrain. At the same 
time extensions that violate these assumptions require that some of the formal proofs be reconsidered, 
for example elastic impacts. One assumption that is often relaxed is the persistence of contact,
Assumption~\ref{ass:contact}, which precludes the use of time-stepping formulations 
\citep[e.g., presented in][]{stewart1996implicit,anitescu1997formulating} that have gained in popularity
as a modeling and simulation framework. The remaining assumptions do not explicitly depend on 
Assumption~\ref{ass:contact}, and so it may be possible to extend some of the results from 
Section~\ref{sec:imp} to these settings (although the massless leg conditions, Assumptions~\ref{ass:cml}
\&~\ref{ass:uml}, and the pseudo-impulse assumption, Assumption~\ref{ass:pseudo}, may prove challenging
to~maintain).

Indexing the contact mode as a subset of the possible contact constraints suggests a natural simplicial 
topology \cite[]{Hatcher2002} over these contact modes, as first suggested in \cite{paper:johnson-icra-2013}.
This organization of the hybrid system should enable the inspection of structural properties of the system
as a whole. Furthermore the various guard sets imply a refinement of the domains into disjoint sets that
reach a unique next guard (or remain in that mode forever). More broadly, we believe that this physically 
motivated hybrid system definition, formal consistency now established, invites study as a mathematical 
object whose properties may likely yield formal insights into the nature of these mechanical systems and 
promote the design of more complex robot behaviors that can exploit~them.

\section*{Acknowledgments}
\pdfbookmark{Acknowledgments}{Acknowledgments}

This was supported in part by the ARL/GDRS RCTA project under Cooperative Agreement Number W911NF-10-2-0016. 
\changed{The authors would like to thank the reviewers and editors for their substantial 
time and expert consideration of this work, as well as their numerous helpful suggestions 
for sharpening the claims and improving the readability of the final paper. We also thank 
Avik De, Andy Ruina, Russ Tedrake, and Michael Posa for helpful related discussions over the past few years.}
An earlier version of parts of this work previously appeared in \cite{thesis:johnson-2014}.

\appendix
\section*{Appendix}
\pdfbookmark{Appendix}{Appendix}
\renewcommand{\thesubsection}{\Alph{subsection}}

\subsection{Proof of Lemma~\ref{lem:twofun}}
\label{app:twofun}

\begin{proof}
First consider $(h(x)\prec_F 0) \Rightarrow (g(x)h(x)\prec_F 0)$.
Let $m_h$ be the index of the first nonzero derivative in the definition of $h(x)\prec_F 0$,~\eqref{eq:thru},
and so for all $\ell<m_h$ the $\ell^{th}$ Lie derivative is zero, $(\Lie^\ell_F h)(x) =0$. Therefore
we also have,
\begin{align}
\big(\Lie^\ell_F (g\cdot h)\big)(x) &= \sum_{k=0}^\ell {\ell \choose k} \big(\Lie^{\ell-k}_F g(x)\big) \cdot \big(\Lie^k_F h(x)\big) 
=  \sum_{k=0}^\ell {\ell \choose k} \big(\Lie^{\ell-k}_F g(x) \big) \cdot \big(0\big) = 0,
\end{align}
and similarly,
\begin{align}
(\Lie^{m_h}_F g\cdot h)(x) &= \sum_{k=0}^{m_h} {{m_h} \choose k} \big(\Lie^{{m_h}-k}_F g(x)\big) \cdot \big(\Lie^k_F h(x)\big) 
= g(x) \cdot \big(\Lie^{m_h}_F h(x)\big),
\end{align}
where since $g(x)>0$, $g(x) \cdot (\Lie^{m_h}_F h(x)) <0 \Leftrightarrow (\Lie^{m_h}_F h(x)) <0$.
Therefore $(h(x)\prec_F 0) \Rightarrow (g(x)h(x)\prec_F 0)$.

Now consider $(g(x)h(x)\prec_F 0) \Rightarrow (h(x)\prec_F 0)$. 
Let $m_{gh}$ be the index of the first nonzero derivative in the definition 
of $g(x)\cdot h(x) \prec_F 0$,~\eqref{eq:thru}. 
The proof proceeds by strong induction on $\ell$, where $0 \leq \ell \leq m_{gh}$, relative to the proposition,
\begin{align}
\big(\Lie^\ell_F g \cdot h\big) (x) = g(x)\Lie^\ell_F h(x), \label{eq:indprop}
\end{align}
i.e.\ that the $\ell^{th}$ 
Lie derivative of the product is equal to the product of positive function, $g(x)$, with
the $\ell^{th}$ Lie derivative of the other factor, $h(x)$. 
The base case is trivial, as for the $0^{th}$ derivative,
\begin{align}
\big(\Lie^0_F g \cdot h\big) (x) = g(x)\Lie^0_F h(x) =  g(x) h(x).
\end{align}
If $m_{gh} = 0$, then $g(x)\cdot h(x) <0$, but
since $g(x)>0$ we get that $h(x)<0$ and therefore $h(x) \prec 0$. If instead 
$m_{gh} >0$, then $g(x)\cdot h(x) = 0$ and therefore $h(x) =0$.
For the inductive step, suppose that the statement is true for all $k < \ell$, implying~that,
\begin{align}
\big(\Lie^k_F g \cdot h\big) (x) = g(x)\Lie^k_F h(x) = 0,
\end{align}
(as recall that $\ell \leq m_{gh}$), and therefore, $\Lie^{k}_F  h (x) = 0 $. Thus,
\begin{align}
\big(\Lie^\ell_F (g\cdot h)\big)(x) &= \sum_{k=0}^\ell {\ell \choose k} \big(\Lie^{\ell-k}_F g(x)\big) \cdot \big(\Lie^k_F h(x)\big) 
=  g(x) \cdot \big(\Lie^\ell_F h(x)\big),
\end{align}
and the induction complete, we now conclude that the
proposition,~\eqref{eq:indprop}, holds for all $\ell\leq m_{gh}$.  

If $g(x)h(x)\prec_F 0$, then by~\eqref{eq:thru} for all $\ell<m_{gh}$,
$(\Lie^{\ell}_F (g\cdot h)(x)) =0$, and so using~\eqref{eq:indprop} and $g(x)>0$,
we conclude that $\Lie^{\ell}_F h(x) =0 $. 
Similarly for $\ell=m_{gh}$, 
$ (\Lie^{m_{gh}}_F (g \cdot h) (x)) <0\Leftrightarrow \Lie^{m_{gh}}_F h(x) <0 $.
Taken together, these are exactly the conditions for $h(x)\prec_F 0$,~\eqref{eq:thru}, and so
$(g(x)h(x)\prec_F 0) \Rightarrow (h(x)\prec_F 0)$.

\end{proof}

\subsection{Linear Algebra}\label{app:la}

For additional notes on the Schur complement and block matrix inverse, see e.g.\ \cite{cottle1974manifestations}, \cite{Lu2002119}, or \cite{Jo2004}.
Consider a block matrix $M$ defined as,
\begin{align}
M:=\left[\begin{array}{cc} E & F \\ G & H \end{array}\right].
\end{align}
If $E$ is nonsingular, then the Schur complement of $E$ in $M$ is,
\begin{align}
S_E := H-GE^{-1}F, \label{eq:schur}
\end{align}
which is sometimes written as $(M|E)$.

If $M$ is also nonsingular, the inverse of $M$ is,
\begin{align}
\left[\begin{array}{cc} E & F \\ G & H \end{array}\right]^{-1} &= 
\left[\begin{array}{cc} E^{-1}+E^{-1}FS_E^{-1}GE^{-1} & -E^{-1}FS_E^{-1} \\
-S_E^{-1}GE^{-1} & S_E^{-1} \end{array}\right].
\end{align}

In particular when $\Mbar$ is invertible the block matrix inverse of~\eqref{eq:astardef} can be written as,
\begin{align}
&  \left[\begin{array}{cc} \Md_J & \bfAdT_J \\ \bfAd_J & \Lambda_{J} \end{array} \right] 
:= \left[\begin{array}{cc}\Mbar & \bfA_J^T \\ \bfA_J & \mathbf{0}_{J\times J} \end{array} \right]^{-1} 
\label{eq:mdexp} 
=\left[\begin{smallmatrix}\Mbar^{-1} - \Mbar^{-1} \bfA^T (\bfA \Mbar^{-1} \bfA^T)^{-1} \bfA \Mbar^{-1} & 
      \quad \Mbar^{-1}\bfA^T(\bfA \Mbar^{-1} \bfA^T)^{-1} \\
      (\bfA \Mbar^{-1} \bfA^T)^{-1}\bfA\Mbar^{-1} &-(\bfA \Mbar^{-1} \bfA^T)^{-1} \end{smallmatrix} \right].
\end{align}
  
Where when $\Mbar$ is positive definite, so is $(\bfA \Mbar^{-1} \bfA^T)^{-1}$, and therefore $\Lambda$ is negative definite. 
Similarly, when $\Mbar$ is only positive semi-definite, $\Lambda$ is negative semi-definite.

A common refinement to this inverse that comes up when considering some constraint 
sets $J$ and $K$ such that $K = J \cup \{k\}$ is\footnote{ Note that the Schur complement, $S_E$, 
used here is with respect to the blocks used in~\eqref{eq:blocksetup} as defined 
explicitly in~\eqref{eq:SAdef}.},
\begin{align}
&\bfA_K = \left[ \begin{array}{c} \bfA_J\\ \bfA_k\end{array}\right],\\
&\left[ \begin{array}{cc} \Mbar & \bfA_K^T \\ \bfA_K & \mathbf{0}_{K\times K} \end{array} \right]^{-1}=
\left[ \begin{array}{cc} 
  \left[\begin{array}{cc}\Mbar & \bfA_J^T \\ \bfA_J & \mathbf{0}_{J\times J} \end{array} \right]
&
  \left[\begin{array}{c}\bfA_k^T \\ \mathbf{0}_{J\times 1} \end{array} \right]
\\
  \left[\begin{array}{cc}\bfA_k & \mathbf{0}_{1\times J} \end{array} \right]
& 
  0
\end{array}\right]^{-1}=\label{eq:blocksetup}\\
&\left[\!\!\!\!\!\! \begin{array}{cc} 
  \left[\!\begin{smallmatrix} \Md_J & \bfAdT_J \\ \bfAd_J & \Lambda_{J} \end{smallmatrix}\! \right]\!\Big(\Id \!+\!\! 
  \left[\!\begin{smallmatrix} \bfA_k^T \\ \mathbf{0} \end{smallmatrix} \!\right] 
  \!\!\begin{smallmatrix}S_E^{-1}\end{smallmatrix}\!\!
  \left[\!\begin{smallmatrix}\bfA_k & \mathbf{0} \end{smallmatrix} \!\right]\!\!
  \left[\!\begin{smallmatrix}\Md_J & \bfAdT_J \\ \bfAd_J & \Lambda_{J} \end{smallmatrix} \!\right]\Big)
&
  \!\minus\!\left[\!\begin{smallmatrix} \Md_J & \bfAdT_J \\ \bfAd_J & \Lambda_{J} \end{smallmatrix}\! \right] \!\!\!\!
  \left[\!\begin{smallmatrix}\bfA_k^T \\ \mathbf{0} \end{smallmatrix} \!\right]
  \!\!\begin{smallmatrix}S_E^{-1}\end{smallmatrix}\!\!
\\
  \!\minus\begin{smallmatrix}S_E^{-1}\end{smallmatrix} \!\!
  \left[\!\begin{smallmatrix}\bfA_k & \mathbf{0} \end{smallmatrix} \!\right] \!\!
  \left[\!\begin{smallmatrix} \Md_J & \bfAdT_J \\ \bfAd_J & \Lambda_{J} \end{smallmatrix}\! \right] \!\!
&
  \begin{smallmatrix}S_E^{-1}\end{smallmatrix} 
\end{array}\!\!\!\!\!\!\right] 
=\left[\!\!\!\!\!\! \begin{array}{cc} 
  \left[\!\begin{smallmatrix} \Md_J & \bfAdT_J \\ \bfAd_J & \Lambda_{J} \end{smallmatrix}\! \right] \!\!+\!\! 
  \left[\!\begin{smallmatrix} \Md_J \bfA_k^T\\ \bfAd_J\bfA_k^T \end{smallmatrix} \!\right] 
  \!\!\begin{smallmatrix}S_E^{-1}\end{smallmatrix}\!\!
  \left[\!\begin{smallmatrix}\bfA_k\Md_J & \bfA_k\bfAdT_J  \end{smallmatrix} \!\right]
&
  \!\minus\!\left[\!\begin{smallmatrix} \Md_J\bfA_k^T \\ \bfAd_J\bfA_k^T \end{smallmatrix}\! \right] 
  \!\!\begin{smallmatrix}S_E^{-1}\end{smallmatrix}\!\!
\\
  \!\minus\begin{smallmatrix}S_E^{-1}\end{smallmatrix} \!\!
  \left[\!\begin{smallmatrix} \bfA_k\Md_J & \bfA_k \bfAdT_J  \end{smallmatrix}\! \right] \!\!
&
  \begin{smallmatrix}S_E^{-1}\end{smallmatrix} 
\end{array}\!\!\!\!\!\!\right]\nonumber\\
&=\left[ \!\!\!\begin{array}{cc}
 \left[ \begin{smallmatrix}  \Md_J +\Md_J \bfA_k^T S_E^{-1}\bfA_k\Md_J \end{smallmatrix} \right]
& 
  \left[\begin{smallmatrix}\bfAdT_J+\Md_J \bfA_k^TS_E^{-1} \bfA_k\bfAdT_J &   -\Md_J\bfA_k^T S_E^{-1} \end{smallmatrix} \right]
\\[.3em]
  \left[\begin{smallmatrix}\bfAd_J + \bfAd_J\bfA_k^T S_E^{-1}\bfA_k\Md_J \\  -S_E^{-1} \bfA_k\Md_J \end{smallmatrix}\right]
& 
  \left[\begin{smallmatrix} \Lambda_{J} +\bfAd_J\bfA_k^T S_E^{-1} \bfA_k\bfAdT_J 
&
  -\bfAd_J\bfA_k^T S_E^{-1}
\\
  -S_E^{-1} \bfA_k \bfAdT_J 
&
  S_E^{-1} \end{smallmatrix} \right]  
\end{array}\!\!\!\right ]\label{eq:finalbreakout} 
=:\left[ \begin{array}{cc}
  \Md_K
& 
  \bfAdT_K
\\ 
  \bfAd_K 
& 
  \Lambda_{K}
\end{array}\right],
\\
&S_E := 0 -   \left[\!\begin{smallmatrix} \bfA_k& \mathbf{0} \end{smallmatrix} \!\right] 
  \left[\!\begin{smallmatrix} \Md_J & \bfAdT_J \\ \bfAd_J & \Lambda_{J} \end{smallmatrix}\! \right]
  \left[\!\begin{smallmatrix}\bfA_k^T \\ \mathbf{0} \end{smallmatrix} \!\right]\!\! = -\bfA_k\Md_J \bfA_k^T. \label{eq:SAdef}
\end{align}
Note that when both the matrix and the first block in~\eqref{eq:blocksetup} are invertible, $S_E$ must be non-zero
as $S_E^{-1}$ is an element of the inverse in~\eqref{eq:finalbreakout}. 
Since $\Lambda_K$ is negative semi-definite, so are its principle minors, in particular $S_E^{-1}$.
Therefore $\bfA_k\Md_J \bfA_k^T$, as a positive semi-definite and non-zero scalar, is a positive number. 
This final expansion,~\eqref{eq:finalbreakout}, expresses the components of $\bfAd_K, \Md_K,$ 
and $\Lambda_K$ in terms of $\bfAd_J, \Md_J,$ and $\Lambda_J$ together with the added constraint $\bfA_k$.

\subsubsection{Proof of Lemma~\ref{thm:mlequiv}}
\label{app:pfmle}

\begin{proof}

Define some set of generalized coordinates (as in \cite[Sec.~II.G]{johnson_selfmanip_2013}), $\bfy$, such
that $\dot{\bfy}=\bfY\dot{\bfq}$ and that the Jacobian of the corresponding implicit function is defined so that $\dot{\bfq}=\bfH\dot{\bfy}$,
\begin{align}
\bfH = \left[ \begin{array}{c} \bfA \\ \bfY \end{array} \right]^{-1}\left[ \begin{array}{c} 0 \\ \Id \end{array} \right],\qquad  
\end{align}

For this proof we need to show that $\tilde{\bfM} = \bfH^T \Mbar \bfH$ is invertible if and only if $\bigl[\begin{smallmatrix}\Mbar & \bfA^T \\ \bfA& 0\end{smallmatrix} \bigr]$ is. 
%
%
The Rank Theorem \cite[Thm.~4.12]{Lee2012} implies there exists a parameterization such that the constraint can be decoupled into
a full rank c$\times$c subblock, 
$\bfA = \left[\bfB \quad 0_{c\times e} \right],$
and therefore we choose a parameterization such that,
$\bfY = \left[0_{e\times c} \quad \Id_e \right]$.
Thus,
\begin{align}
\bfH &= \left[\begin{array}{cc} \bfB & 0_{c\times e} \\ 0_{e\times c} & \Id_e \end{array}\right]^{-1} \left[\begin{array}{c} 0 \\ \Id_e \end{array}\right] 
= \left[\begin{array}{cc} \bfB^{-1} & 0_{c\times e} \\ 0_{e\times c} & \Id_e \end{array}\right] \left[\begin{array}{c} 0 \\ \Id_e \end{array}\right] = \left[\begin{array}{c} 0 \\ \Id_e \end{array}\right],
\end{align}
\begin{align}
\tilde{\bfM} &= \bfH^T \Mbar \bfH = \left[\begin{array}{cc} 0 & \Id_e \end{array}\right]
\left[\begin{array}{cc} \Mbar_{11} &\Mbar_{12} \\ \Mbar_{21} &\Mbar_{22}\end{array}\right] \left[\begin{array}{c} 0 \\ \Id_e \end{array}\right] 
= \Mbar_{22},
\end{align}
and so the requirement is that $\tilde{\bfM}$ is invertible reduces down to simply requiring that $\Mbar_{22}$ is invertible. 

On the other hand we have,
\begin{align}
\left[\begin{array}{cc} \Mbar & \bfA^T \\ \bfA& 0\end{array} \right] &= 
\left[\begin{array}{ccc} \Mbar_{11} &\Mbar_{12} & \bfB^T \\ \Mbar_{21} &\Mbar_{22} & 0\\ \bfB & 0 & 0 \end{array}\right].
\end{align}
Since $\bfB$ is full rank then the invertability of this matrix again reduces to simply invertibility of $\Mbar_{22}$
\citep[e.g.][Corollary~3.3]{Lu2002119}, and thus
the conditions are equivalent. 
\end{proof}

\subsubsection{Proof of Lemma~\ref{thm:dynamics}}
\label{app:pfdyn}
\begin{proof}
Recall that $\lim_{\epsilon\rightarrow0}\Mbar_\epsilon = \Mbar$ and that $\Mbar_\epsilon$ 
is invertible for all $\epsilon\in(0,\bar{\epsilon})$, for some $\bar{\epsilon} > 0$.
For all $\epsilon \ge 0$, define $\Md_\epsilon$, $\bfAd_\epsilon$, and $\Lambda_\epsilon$ by 
replacing $\Mbar$ with $\Mbar_\epsilon$ in~\eqref{eq:astardef}.
Using~\eqref{eq:Md}--\eqref{eq:Lambda} we rewrite the dynamics,~\eqref{eq:dyn}--\eqref{eq:ldyn},
\begin{align}
\lambda &= \bfAd \left(\Upsilon  - \Cbar\dot{\bfq} - \Nbar\right) - \Lambda \dot{\bfA} \dot{\bfq} 
= \lim_{\epsilon\rightarrow 0} \bfAd_\epsilon \left(\Upsilon  - \Cbar\dot{\bfq} - \Nbar\right) - \Lambda_\epsilon \dot{\bfA} \dot{\bfq} \\
&= \lim_{\epsilon\rightarrow 0} \left((\bfA \Mbar_\epsilon^{-1} \bfA^T)^{-1}\bfA\Mbar_\epsilon^{-1}\right) \left(\Upsilon  - \Cbar\dot{\bfq} - \Nbar\right) 
+ (\bfA \Mbar_\epsilon^{-1} \bfA^T)^{-1}\dot{\bfA} \dot{\bfq}\\
&= \lim_{\epsilon\rightarrow 0} (\bfA \Mbar_\epsilon^{-1} \bfA^T)^{-1}\left(\bfA\Mbar_\epsilon^{-1}\left(\Upsilon  - \Cbar\dot{\bfq} - \Nbar\right) + \dot{\bfA} \dot{\bfq} \right), \label{eq:lameps}\\
\ddot{\bfq} &= \Md \left(\Upsilon  - \Cbar\dot{\bfq} - \Nbar\right) - \bfAdT \dot{\bfA} \dot{\bfq} 
= \lim_{\epsilon\rightarrow 0} \Md_\epsilon \left(\Upsilon  - \Cbar\dot{\bfq} - \Nbar\right) - \bfAdT_\epsilon \dot{\bfA} \dot{\bfq}\\
&= \lim_{\epsilon\rightarrow 0} \big(\Mbar_\epsilon^{-1}\!- \Mbar_\epsilon^{-1} \bfA^T (\bfA \Mbar_\epsilon^{-1} \bfA^T)^{-1} \bfA \Mbar_\epsilon^{-1}\big) 
(\Upsilon  - \Cbar\dot{\bfq} - \Nbar)-  \Mbar_\epsilon^{-1}\bfA^T(\bfA \Mbar_\epsilon^{-1} \bfA^T)^{-1}\dot{\bfA} \dot{\bfq} \\
&= \lim_{\epsilon\rightarrow 0} \Mbar_\epsilon^{-1} \Big( \Upsilon  - \Cbar\dot{\bfq} - \Nbar  
- \bfA^T\big((\bfA \Mbar_\epsilon^{-1} \bfA^T)^{-1}(\bfA\Mbar_\epsilon^{-1}(\Upsilon  - \Cbar\dot{\bfq} - \Nbar) + \dot{\bfA} \dot{\bfq} ) \big)\! \Big) 
\\
&= \lim_{\epsilon\rightarrow 0} \Mbar_\epsilon^{-1} \left( \Upsilon  - \Cbar\dot{\bfq} - \Nbar -\bfA^T \lambda \right), \label{eq:ddqeps}
\end{align}
where~\eqref{eq:lameps} and~\eqref{eq:ddqeps} are identically equal to the desired formulation 
of~\eqref{eq:ddqmass} and~\eqref{eq:lammass} when $\Mbar_0$ is non-singular. 
\end{proof} 

\subsubsection{Proof of Lemma~\ref{thm:impulse}}
\label{app:pfimp}
\begin{proof}
Recall that $\lim_{\epsilon\rightarrow0}\Mbar_\epsilon = \Mbar$ and that $\Mbar_\epsilon$ 
is invertible for all $\epsilon\in(0,\bar{\epsilon})$, for some $\bar{\epsilon} > 0$.
For all $\epsilon \ge 0$, define $\Md_\epsilon$, $\bfAd_\epsilon$, and $\Lambda_\epsilon$ by 
replacing $\Mbar$ with $\Mbar_\epsilon$ in~\eqref{eq:astardef}.
Then using equation~\eqref{eq:bfAdT} we rewrite the impulse (where all constraints $\bfA$ are taken to be
for the target contact mode $J$),
\begin{align}
&\bfP_\lambda = - \Lambda \bfA \dot{\bfq}^- =  \lim_{\epsilon\rightarrow0}-\Lambda_\epsilon \bfA \dot{\bfq}^- 
=  \lim_{\epsilon\rightarrow0} (\bfA \Mbar_\epsilon^{-1} \bfA^T)^{-1} \bfA \dot{\bfq}^-,
\end{align}
which is identically equal to~\eqref{eq:plammass} when $\Mbar_0$ is non-singular.
\end{proof}

\subsection{Differential Topology}\label{app:dg}

Let $M$ be a $C^r$ manifold with boundary where $r\in\Nat\cup\set{\infty,\omega}$.
There are several natural constructions associated with $M$ we invoke repeatedly, so we briefly introduce them here and refer the reader to~\cite{Lee2012} for formal definitions.
At every point $x\in M$ there is an associated \emph{tangent space} $T_x M$, which is a vector space with the same dimension as $M$, i.e.\ $\dim T_x M = \dim M$; if $M$ is a submanifold in a Euclidean space of suitable dimension, $T_x M$ may be regarded as a hypersurface in the ambient Euclidean space.
Collating these tangent spaces yields the \emph{tangent bundle} $TM = \mycoprod_{x\in M} T_x M$, which is naturally a $C^r$ manifold with boundary whose dimension is twice that of $M$, i.e.\ $\dim TM = 2 \dim M$.
There is a canonical projection $\pi:TM\into M$ that simply ``forgets'' the tangent vector portion of a point $(x,v)\in TM$, i.e.\ $\pi(x,v) = x$.
At every point $x\in M$ there is an associated \emph{cotangent space} $T_x^* M$, which is the dual of the vector space $T_x M$ (i.e.\ every $\nu\in T_x^* M$ is a linear operator $v:T_x M\into\Real$).
Collating these cotangent spaces yields the \emph{cotangent bundle} $T^*M = \mycoprod_{x\in M} T_x^* M$, which is naturally a $C^r$ manifold with boundary whose dimension is twice that of $M$, i.e.\ $\dim T^*M = 2 \dim M$.
A $C^r$ map $F:M\into TM$ is called a \emph{vector field} if $\pi\circ F = \Id_M$ where $\pi:TM\into M$ is the canonical projection and $\Id:M\into M$ is the canonical identity function.
Given a $C^1$ map $f:M\into N$ between $C^r$ manifolds, there is an associated \emph{pushforward} map $Df:TM\into TN$ between their tangent spaces that evaluates to a linear map $Df(x):T_x M\into T_{f(x)} N$ at every $x\in M$; in coordinates, $Df$ is the familiar Jacobian derivative of $f$.
The \emph{rank} of $f$ at $x\in M$ is defined to be the rank of the linear operator $Df(x)$; if the rank of $f$ does not vary over $M$, it is called \emph{constant rank}.

By the Whitney Embedding Theorem \cite[Thm.~6.15]{Lee2012} (if $r\in\Nat\cup\set{\infty}$) or the Nash Embedding Theorem \cite[]{Nash1966} (if $r=\omega$), $M$ admits a $C^r$ embedding $\iota:M\hookrightarrow\Real^{2n+1}$; 
thus any $C^r$ manifold may be regarded as a submanifold of a Euclidean space of suitably high dimension.
Since $F$ is $C^r$, the pushforward $D\iota\circ F$ admits a $C^r$ extension $\td{F}:\td{M}\into T\td{M}$ over an open neighborhood $\td{M}\subset\Real^{2n+1}$ of the embedded image of $M$.
The Fundamental Theorem on Flows \cite[Thm.~9.12]{Lee2012} implies there exists a \emph{maximal flow} $\td{\Phi}\in C^r(\td{\e{O}},\td{M})$ for $\td{F}$ where 
$\td{\e{O}}\subset\Real\times \td{M}$ is the \emph{maximal flow domain}.
We may restrict $\td{\Phi}$ to obtain a flow over $M$ as follows.
For each $x\in M$, let,
\eqn{
a_x & = \inf \big\{t \le 0 \mid 
(t,\iota(x))\in\td{\e{O}} \wedge \forall\; s\in(t,0] : \td{\Phi}(s,\iota(x))\in\iota(M)\big\}, \\
b_x & = \sup\big\{t \ge 0 \mid 
(t,\iota(x))\in\td{\e{O}} \wedge \forall\; s\in[0,t) : \td{\Phi}(s,\iota(x))\in\iota(M)\big\}.
}
Let $T_x\subset\Real$ be the interval between $a_x$ and $b_x$, including the endpoint if the corresponding infimum or supremum is achieved.
Then let $\e{O} = \bigcup_{x\in M}T_x\times\set{x}\subset\Real\times M$ and, noting that $\td{\Phi}(t,\iota(x))\in\iota(M)$ if $(t,x)\in\e{O}$, define the flow $\Phi:\e{O}\into M$ by $\Phi(t,x) = \iota^{-1}(\td{\Phi}(t,\iota(x)))$.
Note that $\Phi$ is $C^r$ in the sense that $\iota\circ\Phi$ admits a $C^r$ extension, $\td{\Phi}$.

For any $G\subset\bd M$, let,
\eqn{
 H  = \set{x\in M \mid \exists\; t \ge 0 : (t,x)\in\e{O}\wedge\Phi(t,x)\in G}.
}
Define $\eta: H \into\Real$ by,
\eqn{
\forall\; x\in H  : \eta(x) = \inf\set{t \ge 0 \mid (t,x)\in\e{O}\wedge\Phi(t,x)\in G},
}
and $\psi: H \into G$ by $\psi(x) = \Phi(\eta(x),x)$ for all $x\in H$.
Letting $\td{H} = \set{x\in H : F(\psi(x))\not\in T_{\psi(x)} \bd M}$, it is clear that $\eta|_{\td{H}}\in C^r( \td{H} ,\Real)$.
Note that $\eta$ is not differentiable at any point $x\in H$ for which $F(\psi(x))\in T_{\psi(x)} \bd M$; changing coordinates to a flowbox makes this obvious.
Intuitively, the impact time has unbounded sensitivity to initial conditions near such a point of tangency.

\subsection{Hybrid Differential Topology}
\label{app:hdg}
In mechanical systems undergoing intermittent contact with the environment (i.e., terrain or objects), the dynamics are ``piecewise-defined'' (or \emph{hybrid}): whenever a limb attaches or detaches from the substrate there is an instantaneous change in the set of active constraints, leading in general to a discontinuous change in velocity and (constraint) force.
Though it is possible to analyze these discontinuous dynamics in the ambient tangent bundle as in \cite{Ballard2000}, introducing a distinct portion of state space associated with every contact mode renders both the continuous-time dynamics (given by the flow of a vector field) and discrete-time dynamics (specified by a reset map) smooth.
Thus although additional notational overhead is required to index the constituent dynamical elements, the extra effort is partially compensated by enabling the use of elementary constructions from differential topology (rather than sophisticated measure-theoretic techniques used in \cite{Ballard2000}).

Motivated largely by these observations, \cite{BurdenRevzen2013} proposed to 
define hybrid dynamical systems over a finite disjoint union,
\eqn{ 
M = \mycoprod_{J\in {\cal J}} M_J = \bigcup_{J\in\e{J}} \set{J}\times M_J = \set{(J,x) : J\in\e{J}, x\in M_J},
} 
where $M_J$ is a finite dimensional $C^r$ manifold (possibly with corners) for each $J\in {\cal J}$.
By endowing $M$ with the unique largest topology with respect to which the (canonical) inclusions $M_J \hookrightarrow M$ are continuous \cite[Prop.~A.25]{Lee2012},
the set $M$ becomes a \emph{second-countable}, \emph{Hausdorff} topological space which is \emph{locally Euclidean} in the sense that each point $x\in M$ has a neighborhood that is homeomorphic to an open subset of $\Real^{n_x}$, some $n_x\in\Nat$.
Since the dimension is no longer required to be fixed, $M$ is technically not a topological manifold \cite[Chapter~1]{Lee2012}.
However, it is a mild generalization\footnote{
Since, crucially, each of the distinct finite components $M_J$ is  a conventional smooth $C^r$ manifold (of necessarily fixed dimension).}, 
hence we refer to it as a \emph{hybrid topological manifold}.

Motivated by the self-manipulation system from Section~\ref{sec:hyb}, we extend the definition in \cite{BurdenRevzen2013} to allow the component manifolds $M_J$ to possess \emph{corners}.
Unfortunately there is not presently a consensus on what ought to be the definition of a manifold with corners \cite[Remark~2.11]{Joyce2012}.
Fortunately, for our purposes the most straightforward definition in \cite[Ch.~16]{Lee2012} suffices.
This variant, for instance, ensures smooth extensibility of maps between manifolds with corners; see the bottom paragraph of \cite[p.~27]{Lee2012}.
(Note that the discussion of manifolds with boundary in \cite[Sec.~1.4]{Hirsch1976} (termed $\bd$-manifolds) does not address this, though \cite[Lem.~3.1 in Sec.~2.3]{Hirsch1976} should make it unsurprising.)
This coincides with \cite[Def.~1,~2]{Joyce2012}.

For each $J\in\e{J}$, $M_J$ has an associated maximal $C^r$ atlas $\e{A}_J$.
We construct a maximal $C^r$ hybrid atlas for $M$ by collecting charts from the atlases on the components of $M$:
\eqn{
\e{A} = \set{ (\set{J}\times U, \vphi\circ\pi_J) : J\in\e{J}, (U,\vphi)\in\e{A}_J},
}
where $\pi_J :\set{J}\times M_J\into M_J$ is the canonical projection.
We refer to the pair $(M,\e{A})$ as a \emph{$C^r$ hybrid manifold}, but may suppress the atlas when it is clear from context.
We define the \emph{hybrid tangent bundle} as the disjoint union of the component tangent~bundles, 
\eqn{
TM = \mycoprod_{J\in\e{J}} TM_J,
}
and the \emph{hybrid boundary} as the disjoint union of the boundaries,
\eqn{
\bd M = \mycoprod_{J\in\e{J}} \bd M_J.
}

Let $M = \mycoprod_{J\in {\cal J}} M_J$ and $N = \mycoprod_{L\in {\cal L}} N_L$ be two hybrid manifolds.
Note that if a map $R:M\into N$ is continuous 
as a map between topological spaces, then for each $J\in {\cal J}$ there exists $L\in {\cal L}$ 
such that $R(M_J)\subset N_L$ and hence $R|_{M_J}:M_J\into N_L$. 
Using this observation, we define differentiability for continuous maps between hybrid manifolds.
Namely, a map $R:M\into N$ is called \emph{$C^r$} if $R$ is continuous and $R|_{M_J}:M_J\into N$ is $C^r$ for each $J\in {\cal J}$.
In this case the \emph{hybrid pushforward} $DR:TM\into TN$ is the $C^r$ map defined piecewise as $DR|_{TM_J} = D(R|_{M_J})$ for each $J\in {\cal J}$.
A $C^r$ map $F:M\into TM$ is called a \emph{hybrid vector field} if $\pi\circ F = \Id_M$ where $\pi:TM\into M$ is the canonical projection and $\Id:M\into M$ is the canonical identity function.

\subsection{Proofs supporting Theorem~\ref{thm:zenolim}}
\subsubsection{Proof that velocity is bounded}\label{app:mech}
The following are standard results used in the proof of Theorem~\ref{thm:zenolim} to prove that velocity is bounded.
For convenience, we transcribe and apply to our setting these statements from \cite{Ballard2000}, which applies them to achieve 
a similar aim, however this is not to imply that \cite{Ballard2000} is necessarily the original source of these results.

\begin{lemma}[\citealt{Ballard2000}, Prop.~7]\label{lem:bd1} 
If $(\td{\bfq},\dot{\td{\bfq}}^+):[\ubar{t},\obar{t})\into T\e{Q}$ is a right-continuous trajectory of a Lagrangian system subject to perfect unilateral constraints with efforts map $\bff := \Upsilon - \Nbar$,
then for all $t\in [\ubar{t},\obar{t})$:
\eqn{
\frac{1}{2}\absM{\dot{\td{\bfq}}^+(t)}^2 - \frac{1}{2}\absM{\dot{\td{\bfq}}^+(\ubar{t})}^2 \le \int_{\ubar{t}}^t \bff(\td{\bfq}(s),\dot{\td{\bfq}}^+(s)) \dot{\td{\bfq}}^+(s) ds.
}
\end{lemma}

\begin{lemma}[\citealt{Ballard2000}, Lem.~17]\label{lem:bd2}
Let $a:[{\ubar{t}},\obar{t}]\into\Real$ be integrable and nonnegative for almost all $t\in({\ubar{t}},\obar{t})$.
If $\phi:[{\ubar{t}},\obar{t}]\into\Real$ has bounded variation and,
\eqn{
\forall\; t\in[{\ubar{t}},\obar{t}] : \frac{1}{2}\phi^2(t) \le \frac{1}{2} \alpha^2 + \int_{\ubar{t}}^t a(s)\phi(s) ds,
}
for some $\alpha \ge 0$ then,
\eqn{
\forall\; t\in[{\ubar{t}}, \obar{t}] : \abs{\phi(t)} \le \alpha + \int_{\ubar{t}}^t a(s) ds.
}
\end{lemma}

\begin{lemma}[\citealt{Ballard2000}, Lem.~15; \citealt{Sastry1999}, Prop.~3.21]\label{lem:bd3}
Let $a_1:[\ubar{t},\obar{t}]\into\Real$ have bounded variation and $a_2:[\ubar{t},\obar{t}]\into\Real$ be integrable and nonnegative for almost all $t\in (\ubar{t},\obar{t})$.
If $\phi:[\ubar{t},\obar{t}]\into\Real$ has bounded variation and,
\eqn{
\forall\; t\in[\ubar{t}, \obar{t}] : \phi(t) \le a_1(t) + \int_{\ubar{t}}^t a_2(s) \phi(s) ds,
}
then,
\eqn{
&\forall\; t\in[\ubar{t}, \obar{t}] : \phi(t) \le a_1(t) 
+ \int_{\ubar{t}}^t a_1(s)a_2(s)\exp\paren{\int_s^t a_2(\sigma) d\sigma}ds.
}
\end{lemma}

We apply the Lemmas above as in the proof of \cite[Thm.~10]{Ballard2000} to establish that the velocity is bounded on finite time horizons.
Let $(\td{\bfq}^+,\dot{\td{\bfq}}):[\ubar{t},\obar{t})\into T\e{Q}$ be a right-continuous trajectory of a Lagrangian system subject to perfect unilateral constraints and with forces that satisfies the bound in~\eqref{eq:effbddJ}.
Lemma~\ref{lem:bd1} yields for all $t\in T := [\ubar{t},\obar{t})$:
\eqnn{
\frac{1}{2}\absM{\dot{\td{\bfq}}^+(t)}^2 - \frac{1}{2}\absM{\dot{\td{\bfq}}^+(\ubar{t})}^2 \le \int_{\ubar{t}}^t \bff(\td{\bfq}(s),\dot{\td{\bfq}}^+(s)) \dot{\td{\bfq}}^+(s) ds.
}
Applying Lemma~\ref{lem:bd2} with $\phi(t) = \absM{\dot{\td{\bfq}}^+(t)}$, $\alpha = \absM{\dot{\td{\bfq}}^+(\ubar{t})}$, and $a(s) = \absMinv{\bff(\td{\bfq}(s),\dot{\td{\bfq}}^+(s))}$
combined with~\eqref{eq:effbddJ} 
implies for $t\in T$:
\eqnn{
\absM{\dot{\td{\bfq}}^+(t)} \le & \absM{\dot{\td{\bfq}}^+(\ubar{t})} + \int_{\ubar{t}}^t \absMinv{\bff(\td{\bfq}(s),\dot{\td{\bfq}}^+(s))} ds \\
\le & \absM{\dot{\td{\bfq}}^+(\ubar{t})} 
+ \int_{\ubar{t}}^t C\left[1 + \absM{\dot{\td{\bfq}}^+(s)} + d_{\Mbar}(\td{\bfq}(\ubar{t}),\td{\bfq}(s))\right] ds.
}
Recalling that $d_{\Mbar}(\td{\bfq}(\ubar{t}),\td{\bfq}(t)) \le \int_{\ubar{t}}^t \absM{\dot{\td{\bfq}}^+(s)} ds$ we find,
\eqnn{
d_{\Mbar}&(\td{\bfq}(\ubar{t}),\td{\bfq}(t)) + \absM{\dot{\td{\bfq}}^+(t)} 
\le \absM{\dot{\td{\bfq}}^+(\ubar{t})} + C(t - \ubar{t}) 
+ \int_{\ubar{t}}^t (1 + C) \brak{\absM{\dot{\td{\bfq}}^+(s)} + d_{\Mbar}(\td{\bfq}(\ubar{t}),\td{\bfq}(s))} ds.
}
Applying Lemma~\ref{lem:bd3}
with 
$\phi(t) = d_{\Mbar}(\td{\bfq}(\ubar{t}),\td{\bfq}(t)) + \absM{\dot{\td{\bfq}}^+(t)}$,
$a_1(s) = \absM{\dot{\td{\bfq}}^+(\ubar{t})} + C(t - \ubar{t})$,
$a_2(s) = 1 + C$
yields,
\eqnn{
d_{\Mbar}&(\td{\bfq}(\ubar{t}),\td{\bfq}(t)) + \absM{\dot{\td{\bfq}}^+(t)} 
 \le a_1(t) + \int_{\ubar{t}}^t a_1(s) a_2(s)\exp\brak{\int_s^t a_2(\sigma)d\sigma}ds.\\
}
In particular, since the right-hand-side of the inequality is bounded on finite time horizons, velocity is also bounded on finite time horizons,
\eqnn{\label{eq:vbar_}
\bar{v} := \sup_{t\in T}\set{\absM{\dot{\td{\bfq}}^+(t)}} < \infty.
}

\subsubsection{Integration-by-Parts}\label{app:parts}
Suppose $(\bfq,\dot{\bfq}):[\ubar{t},t]\into T\e{Q}$ satisfies~\eqref{eq:ddqmass}.
Then left-multiplying by $\Mbar$ and rearranging,
\eqnn{\label{eq:Mddq}
\Mbar(\bfq) \ddot{\bfq} + \Cbar(\bfq,\dot{\bfq}) \dot{\bfq}= \Upsilon(\bfq,\dot{\bfq}) - \Nbar(\bfq,\dot{\bfq}) - \bfA^T(\bfq) \lambda(\bfq,\dot{\bfq}),
}
where for all $i,j\in\set{1,\dots,\rmq}$,
\eqn{
\Cbar_{ij}(\bfq,\dot{\bfq}) := \sum_{k=1}^{\rmq} 
\frac{1}{2}\paren{\!
\vfof{\Mbar_{ij}(\bfq)}{\bfq^k} + 
\vfof{\Mbar_{ik}(\bfq)}{\bfq^j} - 
\vfof{\Mbar_{kj}(\bfq)}{\bfq^i}\!}
\dot{\bfq}^k,
}
see \cite[Eqn.~4.23]{book:mls-1994} or \cite[Eqn.~30]{johnson_selfmanip_2013} for details.
Note that for all $i\in\set{1,\dots,\rmq}$,
\eqn{
& \big[ \Mbar \ddot{\bfq} + \Cbar\dot{\bfq} \big]^i
 = 
\sum_{j=1}^{\rmq}\brak{\Mbar_{ij} \ddot{\bfq}^j} 
+ \sum_{j,k=1}^{\rmq}\brak{\frac{1}{2}\paren{
\vfof{\Mbar_{ij}(\bfq)}{\bfq^k} + 
\vfof{\Mbar_{ik}(\bfq)}{\bfq^j} - 
\vfof{\Mbar_{kj}(\bfq)}{\bfq^i}}\dot{\bfq}^k\dot{\bfq}^j} \\
& = 
\sum_{j=1}^{\rmq}\brak{\Mbar_{ij} \ddot{\bfq}^j} 
+ \sum_{j,k=1}^{\rmq}\brak{\paren{
\vfof{\Mbar_{ij}(\bfq)}{\bfq^k} 
- \frac{1}{2}\vfof{\Mbar_{kj}(\bfq)}{\bfq^i}}\dot{\bfq}^k\dot{\bfq}^j} \\
& = 
\dt{t} \sum_{j=1}^{\rmq}\brak{\Mbar_{ij} \dot{\bfq}^j} 
- \sum_{j,k=1}^{\rmq}\brak{ \frac{1}{2}\vfof{\Mbar_{kj}(\bfq)}{\bfq^i}\dot{\bfq}^k\dot{\bfq}^j}
= \dt{t} \sum_{j=1}^{\rmq}\brak{\Mbar_{ij} \dot{\bfq}^j} 
+ \Ctd^i(\bfq,\dot{\bfq}),
}
where,
\eqn{
\Ctd^i(\bfq,\dot{\bfq}) := -\frac{1}{2}\sum_{j,k=1}^{\mathrm{q}} \vfof{\Mbar_{kj}(\bfq)}{\bfq^i} \dot{\bfq}^k \dot{\bfq}^j.
}
Therefore rearranging~\eqref{eq:Mddq} we have for each $i\in\set{1,\dots,\rmq}$,
\eqnn{\label{eq:Mddq2}
\dt{t} \sum_{j=1}^{\rmq}\brak{\Mbar_{ij} \dot{\bfq}^j} = \Upsilon^i - \Nbar^i - \Ctd^i - (\bfA^T \lambda)^i.
}

Integrating both sides of~\eqref{eq:Mddq2} over the time interval $[\ubar{t},t]$, reintroducing the dependence on $(\bfq,\dot{\bfq})$ and time, and vectorizing over the index $i$,
\eqn{
\Mbar&(\bfq(t))\dot{\bfq}(t) - \Mbar(\bfq(\ubar{t}))\dot{\bfq}(\ubar{t})
= \int_{\ubar{t}}^t\left(\vphantom{\int} \Upsilon(\bfq(s),\dot{\bfq}(s)) - \Nbar(\bfq(s),\dot{\bfq}(s))  - \Ctd(\bfq(s),\dot{\bfq}(s)) \right. 
\left.\vphantom{\int} - \bfA(\bfq(s))^T \lambda(\bfq(s),\dot{\bfq}(s)) \right) ds,
}
as used in~\eqref{eq:dq}.

\subsubsection{Proof that constraint forces and impulses are bounded}\label{app:eff}

The following is a transcription of the argument used in the proof of \cite[Prop.~18]{Ballard2000} to show that 
constraint forces and impulses are bounded on bounded time intervals.

Given a right-continuous trajectory $(\td{\bfq},\dot{\td{\bfq}}^+):[\ubar{t},\obar{t})\into T\e{Q}$ of a Lagrangian system subject to perfect unilateral constraints, we assume that: 
the inertia tensor $\Mbar$ is nondegenerate; 
the position tends to a limit $\bar{q} := \lim_{t\goesto \obar{t}} \td{\bfq}$; 
and the velocity is bounded by 
$\bar{v} := \sup_{t\in T}\set{\absM{\dot{\td{\bfq}}^+(t)}} < \infty$ where $T = [\ubar{t},\obar{t})$.
This ensures there exists a compact neighborhood $K\subset V$ such that $\td{\bfq}({[\ubar{t},\obar{t})})\subset K$ and hence with $\obar{B}(0,\obar{v})\subset\Real^{\rmq}$ denoting the closed ball of radius $\obar{v}$ centered at the origin, the compact subset $K' := \obar{B}(0,\obar{v})\times K\subset T V$ contains $(\td{\bfq},\dot{\td{\bfq}}^+)({[\ubar{t},\obar{t})})$.
This implies the following constants are~finite:
\eqn{
F &:= \max_{j\in\set{1,\dots,\rmq}} \max_{(\bfq,\dot{\bfq})\in K'} \abs{\Upsilon^j(\bfq,\dot{\bfq}) - \Nbar^j(\bfq,\dot{\bfq})}, \\
G &:= \max_{j,k,\ell\in\set{1,\dots,\rmq}} \max_{\bfq\in K} \abs{\vfof{\Mbar_{k\ell}(\bfq)}{\bfq^j}}.
}
Letting $\sigma_{\max}$ and $\sigma_{\min}$ denote the maximum and minimum singular values of $\Mbar$ over $K$, we obtain the following~bounds:
\eqn{
\max_{j\in\set{1,\dots,\rmq}} \max_{(\bfq,\dot{\bfq})\in K'} \abs{\sum_{k=1}^{\rmq} \Mbar_{jk}(\bfq)\dot{\bfq}^k} &\le \sqrt{\sigma_{\max}} \obar{v}, \\
\max_{j\in\set{1,\dots,\rmq}} \max_{(\bfq,\dot{\bfq})\in K'} \abs{\dot{\bfq}^j} &\le \frac{\obar{v}}{\sqrt{\sigma_{\min}}}.
}
Suppressing dependence on $\td{\bfq}$ and $\dot{\td{\bfq}}^+$, we arrive at the bound
 that for each $j\in\set{1,\dots,\abs{\e{K}}}$,
\eqn{
&-\brak{\Mbar^j(t)\dot{\td{\bfq}}^+(t) - \Mbar^j(t_{\ubar{m}})\dot{\td{\bfq}}^+(t_{\ubar{m}})} 
 + \int_{t_{\ubar{m}}}^t \Upsilon^j - \Nbar^j - \Ctd^j ds
\leq 2\sqrt{\sigma_{\max}} \mathbf{\bar{v}}  + \paren{F + \frac{\rmq^2 G \obar{v}^2}{2\sigma_{\min}}}(t - t_{\ubar{m}})  < \infty,
}
thus satisfying the condition on~\eqref{eq:abbound}.

\bibliographystyle{dcu}
\pdfbookmark{\refname}{\refname}
\bibliography{impulses}

\end{document}